%% file: arxiv.tex
\documentclass[10pt]{article}
\usepackage{float}
\usepackage{subfigure}
\usepackage{algorithm}
\usepackage{algorithmic}

\usepackage{amsthm,amsmath}
\usepackage{microtype}
\usepackage{graphicx}
\usepackage{subfigure}
\usepackage{booktabs} 
\usepackage{enumitem}
\usepackage{bm}
\usepackage{multicol,lipsum}
\usepackage[left=1.0in,right=1.0in,top=1in,bottom=1in,%
            footskip=.25in]{geometry}

\input{defs}

\title{Estimating Learnability in the Sublinear Data Regime}
\author{Weihao Kong\\ Stanford University\\ whkong@stanford.edu \and Gregory Valiant\\ Stanford University\\ gvaliant@cs.stanford.edu}
\begin{document}
\maketitle

\begin{abstract}
We consider the problem of estimating how well a model class is capable of fitting a distribution of labeled data.  We show that it is often possible to accurately estimate this ``learnability'' even when given an amount of data that is too small to reliably learn any accurate model.   Our first result applies to the setting where the data is drawn from a $d$-dimensional distribution with isotropic covariance (or known covariance), and the label of each datapoint is an arbitrary noisy function of the datapoint.  In this setting, we show that with $O(\sqrt{d})$ samples, one can accurately estimate the fraction of the variance of the label that can be explained via the best linear function of the data.  In contrast to this sublinear sample size, finding an approximation of the best-fit linear function requires on the order of $d$ samples. Our sublinear sample results and approach also extend to the non-isotropic setting, where the data distribution has an (unknown) arbitrary covariance matrix: we show that, if the label $y$ of point $x$ is a linear function with independent noise, $y = \langle x , \beta \rangle + noise$ with $\|\beta \|$ bounded, the variance of the noise can be estimated to error $\eps$ with $n=O(d^{1-1/\log{1/\eps}})$ samples if the covariance matrix has bounded condition number, or $n=O(d^{1-\sqrt{\eps}})$ if there are no bounds on the condition number.  We also establish that these sample complexities are optimal, to constant factors.  Finally, we extend these techniques to the setting of binary classification, where we obtain analogous sample complexities for the problem of estimating the prediction error of the best linear classifier, in a natural model of binary labeled data.  We demonstrate the practical viability of our approaches on several real and synthetic datasets.  
\end{abstract}

\input{main.tex}
\end{document}

%% file: defs.tex
\newcommand{\Var}{\mathbf{Var}}

\def\E{{\bf E}}
\def\e{{\bf e}}

\def\R{{\bf R}}
\def\X{{\bf X}}

\def\u{{\bf u}}
\def\v{{\bf v}}

\def\x{{\bf x}}
\def\y{{\bf y}}
\def\z{{\bf z}}

\def\m{{\bf m}}

\def\u{{\bf u}}

\def\v{{\bf v}}

\def\0{{\bf 0}}
\def\1{{\bf 1}}

\def\eps{\epsilon}

\def\argmin{\mathop{\rm argmin}}

\newtheorem{lemma}{Lemma}
\newtheorem{definition}{Definition}
\newtheorem{theorem}{Theorem}

\newtheorem{proposition}{Proposition}
\newtheorem{cor}{Corollary}
\newtheorem{fact}{Fact}

%% file: main.tex

\section{Introduction}
Given too little labeled data to learn a model or classifier, is it possible to determine whether an accurate classifier or predictor exists?   
For example, consider a setting where you are given $n$ datapoints with real-valued labels drawn from some distribution of interest, $D$.  Suppose you are in the regime in which $n$ is too small to learn an accurate prediction model; might it still be possible to estimate the performance that would likely be obtained if, hypothetically, you were to gather more data, say a dataset of size $n' \gg n$ and train a model on that data?  We answer this question affirmatively, and show that in the settings of linear regression and binary classification via linear (or logistic) classifiers, it \emph{is} possible to estimate the likely performance of a (hypothetical) predictor trained on a larger hypothetical dataset, even given an amount of data that is sublinear in the amount that would be required to \emph{learn} such a predictor.  

For concreteness, we begin by describing the flavor of our results in a very basic setting: learning a noisy linear function of high-dimensional data.   Suppose we are given access to independent samples from a $d$-dimensional isotropic Gaussian, and each sample, $\x \in \R^d$ is labeled according to a noisy linear function $y = \langle \x , \beta\rangle + \eta,$ where $\beta$ is the true model and the noise $\eta$ is drawn (independently) from a distribution E of (unknown) variance $\delta^2.$  One natural goal is to estimate the signal to noise ratio, $1-\frac{\delta^2}{\Var[Y]}$, namely estimating how much of the variation in the label we could hope to explain.  Even in the noiseless setting ($\delta = 0$), it is information theoretically \emph{impossible} to learn any function that has even a small constant correlation with the labels unless we are given an amount of data that is linear in the dimension, $d$.  Nevertheless, as was recently shown by Dicker~\cite{dicker2014variance} in this Gaussian setting with independent noise, it is possible to estimate the magnitude of the noise, $\delta$, and variance of the label, given only $O(\sqrt{d})$ samples. 

Our results (summarized in Section~\ref{sec:res}), explore this striking ability to estimate the ``learnability'' of a distribution over labeled data based on relatively little data.  
Our results significantly extend previous results and the results of Dicker in the following senses: 1) We present a unified approach that yields accurate estimation of this learnability when $n=o(d)$ which applies even when the $\x$ portion of the datapoints are drawn from a distribution with arbitrary (unknown) covariance.  This is surprising---and was conjectured to be impossible~\cite{verzelen2018adaptive}---because the best linear model can not be approximated with $o(d)$ data, nor can the covariance be consistently estimated with $o(d)$ datapoints.   2) Agnostic setting: Our techniques do not require any distributional assumptions on the label, $y$, in contrast to most previous work that assumed $y$ is a linear function plus independent noise (which is not a realistic assumption for many of the practical settings of interest).  Instead, our approach directly estimates the fraction of the variance in the label that can be explained via a linear function of $\x$. 
3) Binary classification setting: Our techniques naturally extend to  the setting of binary classification, provided a strong distributional assumption is made---namely that the data is drawn according to the logistic model (see Section~\ref{sec:res} for a formal description of this model).


Throughout, we focus on \emph{linear} models and classifiers, and our assumptions on the data generating distribution are very specific for our binary classification results.  Because some of our results apply when the covariance matrix of the distribution is non-isotropic (and non-Gaussian), the results extend to the many non-linear models that can be represented as a linear function applied to a non-linear embedding of the data, for example settings where the label is a noisy polynomial function of the features.



Still, our estimation algorithms do not apply to all relevant settings; for example, they do not encompass binary classification settings where the two classes do not occur with equal probabilities.  We are optimistic that our techniques may be extended to address that setting, and other practically relevant settings that are not encompassed by the models we consider.  We discuss some of these possibilities, and several other shortcomings of this work and potential directions for future work, in Section~\ref{future}.




\subsection{Motivating Application: Estimating the value of data and dataset selection}
In some data-analysis settings, the ultimate goal is to quantify the signal and noise---namely understand how much of the variation in the quantity of interest can be explained via some set of explanatory variables.  For example, in some medical settings, the goal is to understand how much disease risk is associated with genomic factors (versus random luck, or environmental factors, etc.).  In other settings, the goal is to accurately predict a quantity of interest.  The key question then becomes ``what data should we collect---what features or variables should we try to measure?"   The traditional pipeline is to collect a lot of data, train a model, and then evaluate the value of the data based on the performance (or improvement in performance) of the model.  

Our results demonstrate the possibility of evaluating the explanatory utility of additional features, even in the regime in which too few data points have been collected to leverage these data points to learn a model.  For example, suppose we wish to build a predictor for whether or not someone will get a certain disease.  We could begin by collecting a modest amount of genetic data (e.g. for a few hundred patients, record the presence of genetic abnormalities for each of the ~20k genes), and a modest amount of epigenetic data.  Even if we have data for too few patients to learn a good predictor, we can at least evaluate how much the model would improve if we were to collect more genetic data, versus collecting more epigenetic data.   

This ability to \emph{explore} the potential of different features with less data than would be required to \emph{exploit} those features seems extremely relevant to the many industry and research settings where it is expensive or difficult to gather data.  

Alternately, these techniques could be leveraged by data providers  in the context of a ``verify then buy'' model:  Suppose I have a large dataset of customer behaviors that I think will be useful for your goal of predicting customer clicks/purchases.  Before you purchase access to my dataset, I could give you a tiny sample of the data---too little to be useful to you, but sufficient for you to verify the utility of the dataset.

\subsection{Summary of Results}\label{sec:res}
Our first result applies to the setting where the data is drawn according to a $d$ dimensional distribution with identity covariance (or, equivalently, a known covariance matrix), and the labels are noisy linear functions.  This result generalizes the results of Dicker~\cite{dicker2014variance} and Verzelen and Gassiat~\cite{verzelen2018adaptive} beyond the Gaussian setting.  Provided there are more than $O(\sqrt{d})$ datapoints, the magnitude of the noise can be accurately determined:

\begin{proposition}\label{prop:1main} [Slight generalization of Lemma 2 in~\cite{dicker2014variance} and Corollary 2.2 in~\cite{verzelen2018adaptive}]
Suppose we are given $n$ labeled examples, $(\x_1,y_1),\ldots,(\x_n,y_n)$, with $x_i$ drawn independently from a $d$-dimension distribution of mean zero,  identity covariance, and fourth moments bounded by $C$.  Assuming that each label $y_i = \x_i \beta + \eta$, where the noise $\eta$ is drawn independently from an (unknown) distribution $E$ with mean $0$ variance $\delta^2$, and the labels have been normalized to have unit variance.  There is an estimator $\hat{\delta^2}$, that with probability $1-\tau$, approximates $\delta^2$ with additive error $O(C\frac{\sqrt{d+n}}{\tau n})$.
\end{proposition}

The fourth moment condition of the above proposition is formally defined as follows: for all vectors $\u,\v \in \R^d$, $\E[(\x^T\u)^2(\x^T\v)^2]\le C\E[(\x^T\u)^2]\E[(\x^T\v)^2]$. In the case that the data distribution is an isotropic Gaussian, this fourth moment bound is satisfied with $C = 3$.

We stress that in the above setting, it is information theoretically impossible to approximate $\beta$, or accurately predict the $y_i$'s without a sample size that is \emph{linear} in the dimension, $d$.  The above result is also optimal, to constant factors, in the constant-error regime.  No algorithm can distinguish the case that the label is pure noise, from the case that the label has a significant signal, using $o(\sqrt{d})$ datapoints (see e.g. Proposition 4.2 in~\cite{verzelen2010goodness}):
\begin{proposition}\label{prop:lrlbid}[Corollary of Proposition 4.2 in~\cite{verzelen2010goodness}]
In the setting of Proposition~\ref{prop:1main}, there is a constant $c$ such that no algorithm can distinguish the case that the signal is pure noise (i.e. $\|\beta\|=0$ and $\delta = 1$) versus almost no noise (i.e. $\delta = 0.01$ and $\beta$ is chosen to be a random vector s.t. $\|\beta\|=\sqrt{0.99}$), using fewer than $c \sqrt{d}$ datapoints with probability of success greater than $2/3$.
\end{proposition}

Our estimation machinery extends beyond the isotropic setting, and we provide an analog of Proposition~\ref{prop:1main} to the setting where the datapoints, $\x_i$ are drawn from a $d$ dimensional distribution with (unknown) non-isotropic covariance.  This setting is considerably more challenging than the isotropic setting, since a significant portion of the signal could be accounted for by directions in which the distribution has extremely small variance. 
Though our results are weaker than in the isotropic setting, we still establish accurate estimation of the unexplained variance in the sublinear regime, though require a sample size $O_{\eps}(d^{1-\sqrt{\eps}})$ to obtain an estimate within error $O(\epsilon)$. In the case where the covariance matrix is well conditioned, the sample size can be reduced to $n=O_{\epsilon}(d^{1-\frac{1}{\log{1/\eps}}})$.  We show that both of these sample complexities are optimal.

Our  results in the non-isotropic setting apply to the following standard model of non-isotropic distributions: the distribution is specified by an arbitrary $d \times d$ real-valued matrix, $S$, and a univariate random variable $Z$ with mean 0, variance 1, and bounded fourth moment.  Each sample $\x \in \R^d$ is then obtained by computing $\x = S \z$ where $\z \in \R^d$ has entries drawn independently according to $Z$.  In this model, the covariance of $\x$ will be $S S^T$.  This model is fairly general (by taking $Z$ to be a standard Gaussian this model can represent any $d$-dimensional Gaussian distribution, and it can also represent any rotated and scaled hypercube, etc), and is widely considered in the statistics literature (see e.g.~\cite{yin1983limit,bai1988limiting}).  While our theoretical results rely on this modeling assumption, our algorithm is not tailored to this specific model, and likely performs well in more general settings.

\begin{theorem}\label{thm:gen}
Suppose we are given $n<d$ labeled examples, $(\x_1,y_1),\ldots,(\x_n,y_n)$, with $\x_i=S\z_i$ where $S$ is an unknown arbitrary $d \times d$ real matrix and each entry of $\z_i$ is drawn independently from a one dimensional distribution with mean zero, variance $1$, and constant fourth moment.  Assuming that each label $y_i = \x_i \beta + \eta$, where the noise $\eta$ is drawn independently from an unknown distribution $E$ with mean 0 and variance $\delta^2$, and the labels have been normalized to have unit variance. There is an algorithm that takes $n$ labeled samples, parameter $k$, ${\sigma_{max}}$,${\sigma_{min}}$ which satisfies $\sigma_{max}I\succeq S^TS\succeq {\sigma_{min}I}$, and with probability $1-\tau$, outputs an estimate $\hat{\delta}^2$ with additive error $|\hat{\delta}^2-\delta^2|\le \min(\frac{2}{k^2},2e^{-(k-1)\sqrt{\frac{{\sigma_{min}}}{{\sigma_{max}}}}})){\sigma_{max}}\|\beta\|^2+\frac{f(k)}{\tau}\sum_{i=2}^k \frac{d^{i/2-1/2}}{n^{i/2}},$ where $f(k)=k^{O(k)}$.
\end{theorem} 

Setting $k=1/\sqrt{\eps}$ in the $\sigma_{min}=0$ case and $k = \log(1/\eps)$ in the $\sigma_{min}>0$ case yields the following corollary:

\begin{cor}\label{cor:gen}
In the setting of Theorem~\ref{thm:gen}, with constant $\|\beta\|$ and ${\sigma_{max}}$, the noise can be approximated to error $O(\eps)$ with $n=O(poly(1/\eps)d^{1-\sqrt{\eps}})$. With the additional assumption that $\sigma_{min}$ is a constant greater than $0$, the noise can be approximated to error $O(\eps)$ with $n=O(poly(\log(1/\eps))d^{1-\frac{1}{\log{1/\eps}}})$.
\end{cor}

The estimation accuracy of $\delta^2$ in Corollary~\ref{cor:gen} is optimal up to a constant factor in both singular and non-singular $\Sigma$ cases, even in the setting where $\x$ is drawn from a multivariate Gaussian distribution, as formalized in the following lower bound.
\begin{theorem}\label{thm:lb-all-cond}
The sample complexities of Corollary~\ref{cor:gen} are optimal.  Specifically, given samples $(\x , y)$ with $\x$ drawn from a $d$-dimensional Gaussian with mean 0 and covariance $\Sigma$, and $y = \x \beta+\eta$, for a vector $\beta$ satisfying $\| \beta \| \le 1,$ and independent noise $\eta$ with mean 0 and variance $\delta^2,$ then the following lower bounds apply, in the respective settings where $\Sigma$ is well-conditioned, and where $\Sigma$ is not well-conditioned.  In both settings, we assume that $\E[y^2] = \beta^T \Sigma \beta + \delta^2 = 1$, and the goal is to estimate the unexplained variance, $\delta^2$. 
\begin{itemize} 
    \item If $I \succeq \Sigma\succeq \frac{1}{2}I$, there exist a function $f$ of $1/\eps$ only, such that given $f(1/\eps)d^{1-\frac{1}{\log(1/\eps)}}$ samples $(\x_1,y_1),\ldots,(\x_n,y_n)$, no algorithm can estimate $\delta^2$ with error less than $\eps$ with probability better than $3/5$.
    \item If $\|\Sigma\|\le 1$, there exist a function $f$ of $1/\eps$ only and a constant $c_2$ such that given $f(1/\eps)d^{1-\sqrt{\eps}}$ samples no algorithm can estimate $\delta^2$ with error less than $c_2 \eps$ with probability better than $3/5$.
\end{itemize} 
\end{theorem}

Finally, we establish the following lower bound, demonstrating that, without any assumptions on $\|\Sigma\|$ or $\|\beta\|,$ no sublinear sample estimation is possible. 
\begin{theorem}\label{thm:genLB}
Without any assumptions on the covariance of the data distribution, or bound on $\|\Sigma\|\cdot\|\beta\|$, it is impossible to distinguish the case that the labels are linear functions of the data (zero noise) from the case that the labels are pure noise with probability better than $2/3$ using $c\cdot d$ samples, for some constant $c$.
\end{theorem}



\subsubsection{Estimating Unexplained Variance in the Agnostic Setting}

Our algorithms and techniques do not rely on the assumption that the labels consist of a linear function plus independent noise, and our results partially extend to the agnostic setting.  Formally, assuming that the label, $y$, can have \emph{any} joint distribution with $x$, we show that our algorithms will accurately estimate the fraction of the variance in $y$ that can be explained via (the best) linear function of $x$, namely the quantity $\min_{\beta}\E[({\beta}^T\x-y)^2]$. The analog of Proposition~\ref{prop:1main} in the agnostic setting is the following:
\begin{theorem}\label{thm:agnostic}
Suppose we are given $n$ labeled examples, $(\x_1,y_1),\ldots,(\x_n,y_n)$, with $(\x_i,y_i)$ drawn independently from a $d+1$-dimensional distribution where $\x_i$ has mean zero and identity covariance, and $y_i$ has mean zero and variance $1$, and the fourth moments of the joint distribution $(x,y)$ is bounded by $C$. There is an estimator $\hat{\delta^2}$, that with probability $1-\tau$, approximates $\min_{\beta}\E[({\beta}^T\x-y)^2]$ with additive error $O(C\frac{\sqrt{d+n}}{\tau n})$.
\end{theorem}

The fourth moment condition of the above theorem is analogous to that of Proposition~\ref{prop:1main}: namely, the fourth moments of the joint distribution are bounded by a constant $C$ if, for all vectors $\u,\v \in \R^d$, $\E[(\x^T\u)^2(\x^T\v)^2]\le C\E[(\x^T\u)^2]\E[(\x^T\v)^2]$ and $\E[(\x^T\u)^2y^2]\le C\E[(\x^T\u)^2]\E[y^2]$. As in Proposition~\ref{prop:1main}, in the case that the data distribution is an isotropic Gaussian, and the label is a linear function of the data plus independent noise, this fourth moment bound is satisfied with $C = 3$. 

If the covariance of $\x$ is close to isotropic, i.e. $(1-\epsilon)I \preceq \E[\x\x^T]\preceq (1+\epsilon)I$, the algorithm still applies, and performs analogously to the isotropic case described by Theorem~\ref{thm:agnostic}, except with an additional $O(\epsilon)$ error term, as outlined in Corollary~\ref{cor:approxid}, below.   Hence, if one has a sufficient amount of \emph{unlabeled} data to generate an estimate $\hat{\Sigma}$ for the covariance of $\x$, which has spectral error $\epsilon$, namely $(1-\eps)\Sigma \preceq \hat{\Sigma}\preceq (1+\eps)\Sigma,$ then by scaling $\x$ by $\hat{\Sigma}^{-1/2}$ the resulting distribution will be at most $\epsilon$ far from isotropic and our result on estimating learnability will apply.   If $\x$ is drawn from a sub-gaussian distribution, then $O(d/\eps^2)$ unlabeled examples suffice for the empirical covariance $\hat{\Sigma}$ to be an $\epsilon$ accurate spectral approximation of $\Sigma$.  If $\x$ satisfies the weaker condition of bounded fourth moments, then $O(d\log d/\eps^2)$ unlabeled examples suffice (see Corollary 5.50, Corollary 5.52 of~\cite{vershynin2010introduction}).  

\begin{cor}\label{cor:approxid}
Suppose we are given $n$ labeled examples, $(\x_1,y_1),\ldots,(\x_n,y_n)$, with $(\x_i,y_i)$ drawn independently from a $d+1$-dimensional distribution where $\x_i$ has mean zero and covariance $\Sigma$ which satisfies $(1-\epsilon)I \preceq \Sigma\preceq (1+\epsilon)I$, and $y_i$ has mean zero and variance $1$, and the fourth moments of the joint distribution $(x,y)$ is bounded by $C$. There is an estimator $\hat{\delta^2}$, that with probability $1-\tau$, approximates $\min_{\beta}\E[({\beta}^T\x-y)^2]$ with additive error $O(C\frac{\sqrt{d+n}}{\tau n}+\epsilon)$.
\end{cor}

In the setting where the distribution of $x$ is non-isotropic (and the covariance is unknown), the algorithm to which Theorem~\ref{thm:gen} applies still extends to this agnostic setting.  While the estimate of the unexplained variance is still accurate in expectation, some additional assumptions on the (joint) distribution of $(x,y)$ would be required to bound the variance of the estimator in the agnostic and non-isotropic setting.  Such conditions are likely to be satisfied in many practical settings, though a fully general agnostic and non-isotropic analog of Theorem~\ref{thm:gen} likely does not hold.

\subsubsection{The Binary Classification Setting}

Our approaches and techniques for the linear regression setting also can be applied to the important setting of binary classification---namely estimating the performance of the best linear classifier, in the regime in which there is insufficient data to learn any accurate classifier.  As an initial step along these lines, we obtain strong results in a restricted model of Gaussian data with labels corresponding to the latent variable interpretation of logistic regression.  Specifically, we consider labeled data pairs $(\x,y)$ where $\x \in \R^d$ is drawn from a Gaussian distribution, with arbitrary unknown covariance, and $y \in \{-1,1\}$ is a label that takes value $1$ with probability $g(\beta^T\x_i)$ and $-1$ with probability $1-g(\beta^T\x_i)$ where $g(x)=\frac{1}{1+e^{-x}}$ is the sigmoid function, and $\beta \in \R^d$ is the unknown model parameter.

\begin{theorem}\label{thm:bin-gen}
Suppose we are given $n<d$ labeled examples, $(\x_1,y_1),\ldots,(\x_n,y_n)$, with $\x_i$ drawn independently from a Gaussian distribution with mean $0$ and covariance $\Sigma$ where $\Sigma$ is an unknown arbitrary $d$ by $d$ real matrix.  Assuming that each label $y_i$ takes value $1$ with probability $g(\beta^T\x_i)$ and $-1$ with probability $1-g(\beta^T\x_i)$, where $g(x)=\frac{1}{1+e^{-x}}$ is the sigmoid function. There is an algorithm that takes $n$ labeled samples, parameter $k$, ${\sigma_{max}}$ and ${\sigma_{min}}$ which satisfies $\sigma_{max}I\succeq S^TS\succeq {\sigma_{min}I}$, and with probability $1-\tau$, outputs an estimate $\widehat{err_{opt}}$ with additive error $|\widehat{err_{opt}}-err_{opt}|\le c\Big(\sqrt{\min(\frac{1}{k^2},e^{-(k-1)\sqrt{\frac{\sigma_{min}}{\sigma_{max}}}})\sigma_{max}\|\beta\|^2+\frac{f(k)}{\tau}\sum_{i=2}^k \frac{d^{i/2-1/2}}{n^{i/2}}}\Big),$ where $err_{opt}$ is the classification error of the best linear classifier, $f(k)=k^{O(k)}$ and $c$ is an absolute constant.
\end{theorem}
As in the regression setting, when $\|\beta\|$ is constant, the above theorem shows that the error of the estimate satisfies $|\widehat{err_{opt}}-err_{opt}|\le O(\eps)$ with $n=poly(1/\eps) d^{1-\eps}$ samples when $\sigma_{min}=0$, and  $|\widehat{err_{opt}}-err_{opt}|\le O(\eps)$ with $n=poly(\log(1/\eps)) d^{1-\frac{1}{\log(1/\eps)}}$ samples when $\sigma_{min}>0$.

In the setting where the distribution of $\x$ is an isotropic Gaussian, we obtain the simpler result that the classification error of the best linear classifier can be accurately estimated with $O(\sqrt{d})$ samples.  This is information theoretically optimal, as we show in Section~\ref{sec:bin-lowerbound} in the appendix.

\begin{cor}\label{cor:bin-iso}
Suppose we are given $n$ labeled examples, $(\x_1,y_1),\ldots,(\x_n,y_n)$, with $x_i$ drawn independently from a $d$-dimension isotropic Gaussian distribution $N(0,I)$. Assuming that each label $y_i$ takes value $1$ with probability $g(\beta^T\x_i)$ and $-1$ with probability $1-g(\beta^T\x_i)$, where $g(x)=\frac{1}{1+e^{-x}}$ is the sigmoid function. There is an algorithm that takes $n$ labeled samples, and with probability $1-\tau$, outputs an estimate $\widehat{err_{opt}}$ with additive error $|\widehat{err_{opt}}-err_{opt}|\le c(\frac{\sqrt{d}}{n})^{1/2},$ where $err_{opt}$ is the classification error of the best linear classifier and $c$ is an absolute constant.
\end{cor}

Despite the strong assumptions on the data-generating distribution in the above theorem and corollary, the algorithm to which they apply seems to perform quite well on real-world data, and is capable of accurately estimating the classification error of the best linear predictor, even in the data regime where it is impossible to learn any good predictor.  One partial explanation is that our approach can be easily adapted to a wide class of ``link functions,'' beyond just the sigmoid function addressed by the above results.  Additionally, for many smooth, monotonic functions, the resulting algorithm is almost identical to the algorithm corresponding to the sigmoid link function.


\subsection{Related Work}\label{sec:relatedWork} 

There is a huge body of work, spanning information theory, statistics, and computer science, devoted to understanding what can be accurately inferred about a distribution, given access to relatively few samples---too few samples to learn the distribution in question.  This area was launched with the early work of R.A. Fisher~\cite{Fisher} and Alan Turing and I.J. Good~\cite{Turing} to estimate properties of the \emph{unobserved} portion of a distribution (e.g. estimating the ``missing mass'', namely the probability that a new sample will be a previously unobserved domain element).  More recently, there has been a surge of results establishing that many distribution properties, including support size, entropy, distances between distributions, and general classes of functionals of distributions, can be estimated in the sublinear sample regime in which most of the support of the distribution is unobserved (see e.g.~\cite{independence,Ronitt-2000,AK01,valiant2011estimating,valiant2011power,ADJOPS12,AJOS13a,CDVV14,wu2016minimax,jiao2015minimax}).  The majority of  work in this vein has focused on properties of distributions  that are supported on some discrete (and unstructured) alphabet, or structured (e.g. unimodal) distributions over $\R$ (e.g.~\cite{birge1,birge2,mon1,monotonicity,daskalakis2013testing}).

There is also a line of relevant work, mainly from the statistics community, investigating properties of high-dimensional distributions (over $\R^d$).  One of the fundamental questions in this domain is to estimate properties of the spectrum (i.e. singular values) of the covariance matrix of a distribution, in the regime in which the covariance cannot be accurately estimated~\cite{karoui2008,Bai10,LW12,donoho2013optimal,ledoit2013spectrum,li14}.  This line of work includes the very recent work~\cite{kong2017spectrum} demonstrating that the full spectrum can be estimated given a sample size that is sublinear in the dimensionality of the data---given too little data to accurately estimate any principal components, you can accurately estimate how many directions have large variance, small variance, etc.   We leverage some techniques from this work in our analysis of our estimator for the non-isotropic setting.

For the specific question of estimating the signal to variance ratio (or signal to noise), also referred to as the ``unexplained variance'', there are many classic and more recent estimators that perform well in the linear and super-linear data regime.  These estimators apply to the most restrictive setting we consider, where each label $y = \beta^T\x  + \eta$ is given as a linear function of $\x$ plus independent noise $\eta$ of variance $\delta^2$.  Two common estimators for $\delta^2$ involve first computing the parameter vector $\hat{\beta}$ that minimizes the squared error on the $n$ datapoints.  These estimators are 1) the ``naive estimator'' or the ``maximum likelihood'' estimator: $(\y- \X\hat{\beta})^T(\y- \X\hat{\beta})/n$, and 2) the ``unbiased'' estimator $(\y- \X\hat{\beta})^T(\y- \X\hat{\beta})/(n-d)$, where $\y$ refers to the vector of $n$ labels, and $X$ is the $n \times d$ matrix whose rows represent the $n$ datapoints.  Verifying that the latter estimator is unbiased is a straightforward exercise.  Of course, both of these estimators are zero (or undefined) in the regime where $n \le d$, as the prediction error $(\y- \X\hat{\beta})$ is identically zero in this regime. Additionally, the variance of the unbiased estimator increases as $n$ approaches $d$, as is evident in our empirical experiments where we compare our estimators with this unbiased estimator.

In the regime where $n<d$, variants of these estimators might still be applied but where $\hat{\beta}$ is computed as the solution to a regularized regression (see, e.g.~\cite{wencheko2000estimation}); however, such approaches seem unlikely to apply in the sublinear regime where $n=o(d)$, as the recovered parameter vector $\hat{\beta}$ is not significantly correlated with the true $\beta$ in this regime, unless strong assumptions are made on $\beta$.


Indeed, there has been a line of work on estimating the noise level $\delta^2$ assuming that $\beta$ is sparse~\cite{guo2017optimal,fan2012variance,sun2012scaled,stadler2010,bayati2013estimating}.  These works give consistent estimates of $\delta^2$ even in the regime where $n=o(d)$.   More generally, there is an enormous body of work on the related problem of \emph{feature selection}.
The basis dependent nature of this question (i.e. identifying which features are relevant) and the setting of sparse $\beta$, are quite different from the setting we consider where the signal may be a dense vector.


There have been recent results on estimating the variance of the noise, without assumptions on $\beta$, in the $n<d$ regime.  In the case where $n<d$ but $n/d$ approaches a constant $c \le 1$, Janson et al. proposed the EigenPrism~\cite{janson2017eigenprism} to estimate the noise level.  Their results rely on the  assumptions that the data $\x$ is drawn from an isotropic Gaussian distribution, and that the label is a linear function plus independent noise, and the performance bounds become trivial if $n/d\rightarrow 0.$

Perhaps the most similar work to our paper is the work of Dicker~\cite{dicker2014variance}, which proposed an estimator of $\delta^2$ with error rate $O(\frac{\sqrt{d}}{n})$ in the setting where the data $\x$ is drawn from an isotropic Gaussian distribution, and the label is a linear function plus independent Gaussian.  Their estimator is fairly similar to ours in the identity covariances setting and gives the same error rate.  However, our result is more general in the following senses:  1) Our estimator and analysis do not rely on Gaussianity assumptions; 2) Our results apply beyond the setting where label $y$ is a linear function of $\x$ plus independent noise, and estimates the fraction of the variance that can be explained via a linear function (the ``agnostic'' setting); and 3) our approach extends to the unknown non-isotropic covariance setting.

In case the sparsity of $\beta$ is unknown, Verzelen and Gassiat~\cite{verzelen2018adaptive} introduced a hybrid approach which combines Dicker's result in the dense regime and Lasso in the sparse regime to achieve consistent estimation of $\delta^2$ using $\min(k\log(d),\sqrt{d\log(d)})$ samples in the isotropic covariance setting where $k$ is the unknown sparsity of $\beta$, and they showed the optimality of the algorithm. In the unknown covariance, dense $\beta$ setting, they conjectured consistent estimation of $\delta^2$ is not possible with $o(d)$ samples; our Theorem~\ref{thm:gen} shows that this conjecture is false.

This problem of estimating the signal-noise ratio has been considered in the \textit{random effect model} with the assumption that $y=\beta^T\x+\eta$ where $\beta$ is drawn from $N(0,\sigma^2I_d)$ with some unknown $\sigma$. The classical maximum likelihood estimator of $\delta^2$ in this setting is widely used in many practical applications, particularly genomics. 
For example, genome-based restricted maximum likelihood (GREML) in genome-wide complex trait analysis (GCTA)\cite{yang2011gcta} utilizes this estimator to quantify the total additive contribution of a particular subset of genetic variants to a trait's heritability. On the method of moments side, the classical Haseman--Elston regression\cite{haseman1972investigation} is a method for estimating the signal-to-noise ratio in the random effect model, which is generalized in \cite{golan2014measuring} for the heritability estimation problem. Our algorithm in the isotropic covariance setting is essentially  equivalent to Haseman--Elston regression, after an appropriate scaling. 
Though practically popular, very little is known about the theoretical guarantee of these estimators when applied to the \textit{fixed-effect model} that we consider, where no distributional assumptions are made on $\beta$. Dicker and Erdogdu~\cite{dicker2016maximum} showed that the maximum likelihood estimator for the random effect model is consistent and asymptotically normal in the fixed effect model assuming that $\x\sim N(0,I), \eta\sim N(0,\delta^2)$ and each label $y$ is a linear function of $\x$ plus independent Gaussian noise: $y = \beta^T\x+\eta$.

\medskip

Finally, there is a body of work from the theoretical computer science community on ``testing'' whether a function belongs to a certain class, including work on testing \emph{linearity}~\cite{bellare1996linearity,ben2003randomness} generally over finite fields rather than $\R^d$, and testing \emph{monotonicity} of functions over the Boolean hypercube~\cite{goldreich1998testing,chen2014new}.  Most of this work is in the ``query model'' where the algorithm can (adaptively) choose a point, $x$, and obtain its label $\ell(x)$.  The goal is determine whether the labeling function belongs to the desired class using as few queries as possible.  This ability to query points seems to significantly alter the problem, although it corresponds to the setting of ``active learning'' in the setting where there is an exponential amount of unlabeled data.  More recent works on ``active testing''~\cite{balcan2012active,blum2017active} considers this problem of distinguishing whether a labeling function belongs in a specified class, versus is ``far'' from the class, given a modest amount of unlabeled data, and the ability to query labels for a subset of that data.  These works consider several function classes, including unions of intervals (over the domain $[0,1]$), and linear threshold functions.  In~\cite{balcan2012active}, they show that for isotropic $d$-dimensional Gaussian data, given $O(\sqrt{d \log d})$ samples, one can distinguish the case that the data is linearly separable, versus the case that all linear classifiers have constant error (bounded away from 0).  We note that for this restrictive setting, a trivial modification of our approach yields the tighter bound of $O(\sqrt{d})$ for this problem.

From the standpoint of computational hardness, in the binary classification setting, without any assumption on the distribution of $X$, the best linear classifier (also known as halfspace) is PAC-learnable in the presence of random classification noise~\cite{blum1998polynomial}. However, with adversarial classification noise, or equivalently, when the binary label $y$ is an arbitrary noisy function of the datapoint $X$, it is NP-hard even to distinguish between the case where the best linear classifier has accuracy 99\% versus no linear classifier is able to achieve accuracy 51\%~\cite{feldman2006new,guruswami2009hardness}. Regarding the problem of estimating learnability, this rules out the possibility of obtaining results in the classification setting of similar strength to the regression setting. Given the hardness result, there has been a line of work~\cite{kalai2008agnostically, awasthi2014power} studying the agnostic linear classification problem in the setting where $X$ is drawn from some ``nice'' distribution (e.g., Gaussian, log-concave etc.), and where the goal is to learn a prediction model that achieves close to optimal classification accuracy.

\subsection{Future Directions and Shortcomings of Present Work}\label{future}
This work demonstrates---both theoretically and empirically---a surprising ability to  estimate the performance of the best model in basic model classes (linear functions, and linear classifiers) in the regime in which there is too little data to learn any such model.  That said, there are several significant caveats to the applicability of these results, which we now discuss.  Some of these shortcomings seem intrinsically necessary, while others can likely be tackled via extensions of our approaches.

\noindent \textbf{More General Model Classes, and Loss Functions:} Perhaps the most obvious direction for future work is to tackle more general model classes, under more general classes of loss function, in more general settings.  While our results on linear regression extend to function classes (such as polynomials) that can be obtained via a linear function applied to a nonlinear embedding of the data, the results are all in terms of estimating unexplained variance, namely estimating ($\ell_2$ error).  Our techniques do leverage the geometry of the $\ell_2$ loss, and it is not immediately clear how they could be extended to more general loss functions.  

Our results for binary classification are restricted to the  specific model of Gaussian data (with arbitrary covariance) and with label assigned to be $\pm 1$ with probabilities according to the latent variable interpretation of logistic regression, namely 1 with probability $g(\beta^T x)$ and $-1$ with probability $1-g(\beta^T x)$, where $\beta$ is the vector of hidden parameters, and the function $g$ is the sigmoid function.  Our techniques are not specific to the sigmoid function, and can yield analogous results for other monotonic ``link'' functions. Similarly, the Gaussian assumption can likely be relaxed.  Still, it seems that any strong theoretical results for the binary classification setting would need to rely on fairly stringent assumptions on the structure of the data and labels in question. 


\medskip

\noindent \textbf{Heavy-tailed covariance spectra:} 
One of the practical limitations of our techniques is that they are unable to accurately capture portions of the signal that depend on directions in which the underlying distribution has extremely small variance.  As our lowerbounds show, this is unavoidable.  That said, many real-world distributions exhibit a power-law like spectrum, with a large number of directions having variance that is orders of magnitude smaller than the directions of larger variance, and a significant amount of signal is often contained in these directions.  

From a practical perspective, this issue can be addressed by partially ``whitening'' the data so as to make the covariance more isotropic.  Such a re-projection requires an estimate of the covariance of the distribution, which would require either specialized domain knowledge, or a (unlabeled) dataset of size at least linear in the dimension. In some settings it might be possible to easily collect a surrogate (unlabeled) dataset from which the re-projection matrix could be computed.  For example, for NLP settings, a generic language dataset such as the Wikipedia corpus could be used to compute the reprojection.

\medskip

\noindent \textbf{Data aggregation, federated learning, and secure ``proofs of value'':}  There are many tantalizing directions (both theoretical and empirical) for future work on downstream applications of the approaches explored in this work.  The approaches of this work could be re-purposed to explore the extent to which two or more labeled datasets have the same (or similar) labeling function, even in the regime in which there is too little data to learn such a function---for example, by applying these techniques to the aggregate of the datasets versus individually and seeing whether the signal to noise ratio degrades upon aggregation.  Such a primitive might have fruitful applications in realm of ``federated learning'', and other settings where there are a large number of heterogeneous entities, each supplying a modest amount of data that might be too small to train an accurate model in isolation.  One of the key questions in such settings is how to decide which entities have similar models, and hence which subsets of entities might benefit from training a model on their combined data.

Finally, a more speculative line of future work might explore the possibility of creating \emph{secure} or \emph{privacy preserving} ``proofs of value'' of a dataset.  The idea would be to publicly release either a portion of a dataset, or some object derived from the dataset, that would ``prove'' the value of the dataset while preventing others from exploiting the dataset, or while preserving various notions of security or privacy of the database).  The approaches of this work might be a first step towards those directions, though such directions would need to begin with a formal specification of the desired security/privacy notions, etc.

\section{The Estimators, Regression Setting}\label{sec:estimators}

Before describing our estimators, we first provide an intuition for why it is possible to estimate the ``learnability'' in the sublinear data regime.

\subsection{Intuition for Sublinear Estimation}\label{sec:intuition}
We begin by describing one intuition for why it is possible to estimate the magnitude of the noise using only $O(\sqrt{d})$ samples, in the isotropic setting.  Suppose we are given data $\x_1,\ldots,\x_n$ drawn i.i.d. from $N(0,I_d),$ and let $y_1,\ldots,y_n$ represent the labels, with $y_i = \beta^T\x_i + \eta$ for a random vector $\beta \in \R^d$ and $\eta$ drawn independently from $N(0,\delta^2)$.  Fix $\beta$, and consider partitioning the datapoints into two sets, according to whether the label is positive or negative.  In the case where the labels are complete noise ($\delta^2 = 1$), the expected 
value of a positively labeled point is the same as that of a negatively labeled point and is $\overrightarrow{0}$.  In the case where there is little noise, the expected value $\mu_+$ of a positive point will be different than that of a negative point, $\mu_-$, and the distance between these points corresponds to the distance between the mean of the `top' half of a Gaussian and the `bottom' half of a Gaussian.  Furthermore, this distance between the expected means will smoothly vary between $0$ and $2 \sqrt{2/\pi}$ as the variance of the noise, $\delta^2$, varies between $1$ and $0$.

The crux of the intuition for the ability to estimate $\delta^2$ in the regime where $n=O(\sqrt{d})$ is the following observation: while the empirical means of the positive and negative points have high variance in the $n=o(d)$ regime, it is possible to accurately estimate the \emph{distance} between $\mu_+$ and $\mu_-$ from these empirical means! At a high level, this is because the empirical means consists of $d$ coordinates, each of which has a significant amount of noise. However, their squared distance is just a single number which is a sum of $d$ quantities, and we can leverage concentration in the amount of noise contributed by these $d$ summands to save a $\sqrt{d}$ factor. This closely mirrors the folklore result that it requires $O(d)$ samples to accurately estimate the mean of an identity covariance Gaussian with unknown mean, $N(\mu,I_d)$, though the norm of the mean $\|\mu\|$ can be estimated to error $\eps$ using only $n=O(\sqrt{d}/\eps)$.

Our actual estimators, even in the isotropic case,  do not directly correspond to the intuitive argument sketched in this section.  In particular, there is no partitioning of the data according to the sign of the label, and the unbiased estimator that we construct does not rely on any Gaussianity assumption.    

\subsection{The Estimators}\label{sec:est}

The basic idea of our proposed estimator is as follows. Given a joint distribution over $(\x,y)$ where $\x$ has mean $0$ and variance $\Sigma$, the classical least square estimator which minimizes the unexplained variance takes the form $\beta = \E[\x\x^T]^{-1}\E[y\x]=\Sigma^{-1}\E[y\x]$, and the corresponding value of the unexplained variance is $\E[(y-\beta^T\x)^2] = \E[y^2]-\beta^T\Sigma\beta$. Notice that the least square estimator is exactly the model parameter $\beta$ in the linear model setting, and we use the same notation to denote them. The variance of the labels, $y$ can be estimated up to $1/\sqrt{n}$ error with $n$ samples, after which the problem reduces to estimating $\beta^T\Sigma\beta$. While we do not have an unbiased estimator of $\beta^T\Sigma\beta$, as we show, we can construct an unbiased estimator for $\beta^T\Sigma^k\beta$ for any integer $k\ge 2$.


To see the utility of estimating these ``higher moments'', assume for simplicity that $\Sigma$ is a diagonal matrix.  Consider the distribution over $\R$ consisting of $d$ point masses with the $i$th point mass located at $\Sigma_{i,i}$ with probability mass $\beta_i^2/\|\beta\|^2$.  The problem of estimating $\beta^T\Sigma\beta$ is now precisely the problem of approximating the first moment of this distribution, and we are claiming that we can compute unbiased (and low variance) estimates of $\beta^T\Sigma^k\beta$ for $k=2,3,\ldots$, which exactly correspond to the 2nd, 3rd, etc. moments of this distribution of point masses.  Our main theorem follows from the following two components: 1) There is an unbiased estimator that can estimate the $k$th ($k\ge2$) moment of the distribution using only $O(d^{1-1/k})$ samples. 2) Given accurate estimates of the 2nd, 3rd,\ldots,$k$th moments, one can approximate the first moment with error $O(1/k^2)$, and the error can be improved to $O(e^{-k}$ if the covariance matrix $\Sigma$ is well conditioned. The main technical challenge is the first component---constructing and analyzing the unbiased estimators for the higher moments; the second component of our approach amounts to showing that the function $f(x)=x$ can be accurately approximated via the polynomials $f_2(x)=x^2,$ $f_3(x)=x^3, \ldots, f_k(x)=x^k$, and is a straightforward exercise in real analysis.  The final estimator for $\beta^T\Sigma\beta$ in the non-identity covariance setting will be the linear combination of the unbiased estimates of $\beta^T\Sigma^2\beta,$ $\beta^T\Sigma^3\beta,\ldots$, where the coefficients correspond to those of the polynomial approximation of $f(x)=x$ via $f_2,f_3,\ldots.$   The following proposition (proved in the supplementary material) summarizes the quality of this polynomial interpolation:

\begin{proposition}
For any integer $k$ and real value $1\ge b\ge 0$, there is a degree $k$ polynomial $p_k(x)$ with no linear or constant terms, satisfying $|p_k(x)-x|<\min(2/k^2,2e^{-(k-1)\sqrt{b}})$ for all $x\in [b,1]$. 
\end{proposition}
The above proposition follows easily from Theorem 5.5 in~\cite{devore1993constructive} and Theorem 7.8 in~\cite{orecchia2012approximating}, and we include the short proof in the appendix (see Proposition~\ref{prop:poly-approx}).

\medskip

\noindent\textbf{Identity Covariance Setting:} In the setting where the data distribution has identity covariance, $\beta^T\Sigma^2\beta = \beta^T\Sigma \beta$ simply because $1^2 =1,$ and hence we \emph{do} have a simple unbiased estimator, summarized in the following algorithm for the isotropic setting, to which Proposition~\ref{prop:1main} applies:


\begin{algorithm}[H]
\begin{algorithmic}
\STATE \textbf{Input}: $X = 
\begin{bmatrix} 
\x_1\\
\vdots\\
\x_n
\end{bmatrix}, \quad
y = 
\begin{bmatrix} 
y_1\\
\vdots\\
y_n
\end{bmatrix}
$ 
\begin{itemize}
\item Set $A = XX^T$, and let $G = A_{up}$ be the matrix A with the diagonal and lower triangular entries set to zero.
\end{itemize}
\STATE \textbf{Output}: $\frac{y^Ty}{n} - \frac{y^TGy}{\binom{n}{2}}$
\caption{Estimating Linearity, Identity covariance}\label{alg:iso}
\end{algorithmic}
\end{algorithm}

To see why the second term of the output corresponds to an unbiased estimator for $\beta^T\Sigma^2\beta$ (and hence for $\beta^T\Sigma\beta$ in the isotropic case), consider drawing two independent samples $(\x_1,y_1),(\x_2,y_2)$. Indeed, $y_1y_2\x_1^T\x_2$ is an unbiased estimator of $\beta^T\Sigma^2\beta$, because $\E[y_1y_2\x_1^T\x_2] = \E[y_1\x_1^T]\E[y_2\x_2] =\beta^T\Sigma^2\beta$. Given $n$ samples, by linearity of expectation, a natural unbiased estimate is hence to compute this quantity for each pair (of distinct) samples, and take the average of these $\binom{n}{2}$ quantities. This is precisely what Algorithm 1 computes, since $\E[y^T Gy] = \E[\sum_{i< j}y_iy_j\x_i^T\x_j]$.

Given that the estimator in the isotropic case is unbiased, Proposition~\ref{prop:1main} will follow provided we adequately bound its variance:

\begin{proposition}\label{prop:bs2b}
$\Var[y^TGy] = O(C^2n^2(d+n)),$ where $C$ is the bound on the fourth moments.
\end{proposition}
\begin{proof}
The variance can be expressed as the following summation:
$
\Var[y^TGy] = \Var[\sum_{i< j}y_iy_j\x_i^T\x_j] =
\sum_{i< j, i'< j'} (\E[y_iy_jy_{i'}y_{j'}\x_i^T\x_j \x_{i'}^T \x_{j'}] - \E[y_iy_j\x_i^T\x_j]\E[y_{i'}y_{j'}\x_{i'}^T\x_{j'}]).
$ For each term in the summation, we classify it into one of the $3$ different cases according to $i,j,i',j'$:
\begin{enumerate}
\item If $i,j,i',j'$ all take different values, the term is $0$.
\item If $i,j,i',j'$ take $3$ different values, WLOG  assume $i=i'$. The term can then be expressed as:
$\E[y_i^2y_jy_{j'}\x_i^T \x_j \x_i^T \x_{j'}] - (\beta^T\beta)^2\\
= \E[y_i^2(\beta^T \x_i)^2]- (\beta^T\beta)^2$. By our fourth moment assumption, $\E[y_i^2(\beta^T \x_i)^2]\le C \|\beta\|^2$, and we conclude that $C\|\beta\|^2-\|\beta\|^4$ is an upperbound.
\item If $i,j,i',j'$ take $2$ different values, the term is:
$
\E[y_i^2y_j^2(\x_i^T \x_j)^2] - (\beta^T\beta)^2.
$
First taking the expectation over the $j$th sample, we get the following upper bound 
$
\E[Cy_i^2(\x_i^T\x_i)]-(\beta^T\beta)^2.
$
Notice that $\x_i^T\x_i=\sum_{j=1}^d(\e_j^T\x_i)^2$. Taking the expectation over the $i$th sample and applying the fourth moment condition, we get the following bound:
$
dC^2-\|\beta\|^4.
$
\end{enumerate}
The final step is to sum the contributions of these $3$ cases. Case $2$ has $4\binom{n}{3}$ different quadruples $(i,j,i',j')$. Case $3$ has $\binom{n}{2}$ different quadruples $(i,j,i',j')$. Combining the resulting bounds yields:
$
\Var[\sum_{i< j}y_iy_j\x_i^T\x_j]= O(n^2dC^2+n^3C\|\beta\|^2).
$
Since $\E[(y-\beta^T\x)^2]= 1 - \|\beta\|^2\ge 0$, it must be that $\|\beta\|\le 1$. Further by the fact that $C\ge 1$, the proposition statement follows.
\end{proof}

Having shown that the estimator is unbiased and has variance bounded according to the above proposition, Proposition~\ref{prop:1main} now follows immediately from Chebyshev's inequality.

\medskip
\noindent \textbf{Non-Identity Covariance:} Algorithm 2, to which Theorem~\ref{thm:gen} applies, describes our estimator in the general setting where the data has a non-isotropic covariance matrix.

\begin{algorithm}[H]
\begin{algorithmic}
\STATE \textbf{Input}: $X = 
\begin{bmatrix} 
\x_1\\
\vdots\\
\x_n
\end{bmatrix}, \quad 
\y = 
\begin{bmatrix} 
y_1\\
\vdots\\
y_n
\end{bmatrix}
,$ and degree $k$ polynomial $p(x)=\sum_{i=0}^{k-2} a_i x^{i+2}$ that approximates the function $f(x)=x$ for all $x\in [\sigma_{\min},\sigma_{\max}],$ where $\sigma_{\min}$ and $\sigma_{\max}$ are the minimum and maximum singular values of the covariance of the distribution from which the $\x_i$'s are drawn. 
\begin{itemize}
\item Set $A = XX^T$, and let $G = A_{up}$ be the matrix A with the diagonal and lower triangular entries set to zero.
\end{itemize}
\STATE \textbf{Output}: $\frac{y^Ty}{n} - \sum_{i=0}^{k-1} a_i\frac{y^TG^{i+1}y}{\binom{n}{i+2}}$
\caption{Estimating Linearity, General covariance}\label{alg:gen}
\end{algorithmic}
\end{algorithm}

The general form of the unbiased estimators of $\beta^T\Sigma^k\beta$ for $k\ge 2$ closely mirrors the case discussed for $k=2$, and the proof that these are unbiased is analogous to that for the $k=2$ setting explained above.  The analysis of the variance, however, becomes quite complicated, as a significant amount of machinery needs to be developed to deal with the combinatorial number of ``cases'' that are analogous to the 3 cases discussed in the variance bound for the $k=2$ setting of Proposition~\ref{prop:bs2b}.  Fortunately, we are able to borrow some of the approaches of the work~\cite{kong2017spectrum}, which also bounds similar looking moments (with the rather different goal of recovering the covariance spectrum).

The proof of correctness of Algorithm 2, establishing Theorem~\ref{thm:gen} is given in a self-contained form in the appendix.

\section{The Binary Classification Setting}\label{sec:estimators_bin}
In the binary classification setting, we assume that we have $n$ independent labeled samples $(\x_1,y_1),\ldots,(\x_n,y_n)$ where each $\x_i$ is drawn from a Gaussian distribution $\x_i\sim N(0,\Sigma)$. There is an underlying link function $g:\mathcal{R}\to [0,1]$ which is monotonically increasing and satisfies $g(0)=1/2$, and an underlying weight vector $\beta$, such that each label $y_i$ takes value $1$ with probability $g(\beta^T\x_i)$ and $-1$ with probability $1-g(\beta^T\x_i)$. 
Under this assumption, the goal of our algorithm is to predict the classification error of the best linear classifier. In this setting, the best linear classifier is simply the linear threshold function $sgn(\beta^T\x)$ whose classification error is $\frac{1}{2}-\E[|g(\beta^T\x)-\frac{1}{2}|]$. 

The core of the estimators in the binary classification setting is the following observation: given two independent samples $(\x_1,y_1),(\x_2,y_2)$ drawn from the linear classification model described above, $y_1y_2\x_1^T\x_2$ is an unbiased estimator of $4{\E[(g(\beta^T\x)-\frac{1}{2})\frac{\beta^T\x}{\beta^T\Sigma\beta}]}^2\beta^T\Sigma^2\beta$, simply because $\E[y\x]$ is an unbiased estimator of $\E[y\x] = 2\E[(g(\beta^T\x)-\frac{1}{2})\frac{\beta^T\x}{\beta^T\Sigma\beta}]\Sigma\beta$, as we show below in Proposition~\ref{prop:bin-unbiased}. We argue that such an estimator is sufficient for the setting where $\Sigma=I$ and the function $g(x)$ is known. To see that, taking the square root of the estimator yields an estimate of $\E[(g(\beta^T\x)-\frac{1}{2})\frac{\beta^T\x}{\|\beta\|}] = \E_{x\sim N(0,1)}[(g(\|\beta\|x)-\frac{1}{2})x]$, which is monotonically increasing in $\|\beta\|$, and hence can be used to determine $\|\beta\|$. The classification error, $\frac{1}{2}-\E_{x\sim N(0,1)}[|g(\|\beta\|x)-\frac{1}{2}|]$, can then be calculated as a function of the estimate of $\|\beta\|$. The following proposition proves the unbiasedness property of our estimator.
\begin{proposition}\label{prop:bin-unbiased}
$\E[y\x] = 2\E[(g(\beta^T\x)-\frac{1}{2})\frac{\beta^T\x}{\beta^T\Sigma\beta}]\Sigma\beta$.

\end{proposition}
\begin{proof}
First we decompose $\x$ into the sum of two parts: $\frac{\beta^T\x}{\beta^T\Sigma\beta}\Sigma\beta$ and $\x-\frac{\beta^T\x}{\beta^T\Sigma\beta}\Sigma\beta$, where the second part  $\x-\frac{\beta^T\x}{\beta^T\Sigma\beta}\Sigma\beta$ is independent of $\beta^T\x$. Since $\E[(2g(\beta^T\x)-1)(\x-\frac{\beta^T\x}{\beta^T\Sigma\beta}\Sigma\beta)] = \E[2g(\beta^T\x)-1]\E[\x-\frac{\beta^T\x}{\beta^T\Sigma\beta}\Sigma\beta]=0$, we have that $\E[(2g(\beta^T\x)-1)\x] = \E[(2g(\beta^T\x)-1)\frac{\beta^T\x}{\beta^T\Sigma\beta}]\Sigma\beta = 2\E[(g(\beta^T\x)-\frac{1}{2})\frac{\beta^T\x}{\beta^T\Sigma\beta}]\Sigma\beta.$
\end{proof}

For the setting where $\x$ is drawn from a non-isotropic Gaussian with unknown covariance $\Sigma$, we apply a similar approach as in the linear regression case.  First, we obtain a series of unbiased estimators for $\E[(g(\beta^T\x)-\frac{1}{2})\frac{\beta^T\x}{\beta^T\Sigma\beta}]^2\beta\Sigma^k\beta$ with $k=2,3,\ldots$. Then, we find a linear combination of those estimates to yield an estimate of $\E[(g(\beta^T\x)-\frac{1}{2})\frac{\beta^T\x}{\beta^T\Sigma\beta}]^2\beta^T\Sigma\beta = \E[(g(\beta^T\x)-\frac{1}{2})\frac{\beta^T\x}{\|\beta^T\Sigma^{1/2}\|}]^2 = \E_{x\sim N(0,1)}[(g(\|\beta^T\Sigma^{1/2}\|x)-\frac{1}{2})x]$.  This latter expression can then be used to determine $\|\beta^T\Sigma^{1/2}\|$, after which the value of $\frac{1}{2}-\E_{x\sim N(0,1)}[|g(\|\beta^T\Sigma^{1/2}\|x)-\frac{1}{2}|]$ can be determined.

Our general covariance algorithm for estimating the classification error of the best linear predictor, to which Theorem~\ref{thm:bin-gen} applies, is the following:
 
\begin{algorithm}[H]
\begin{algorithmic}
\STATE \textbf{Input}: $X = 
\begin{bmatrix} 
\x_1\\
\vdots\\
\x_n
\end{bmatrix}, \quad 
\y = 
\begin{bmatrix} 
y_1\\
\vdots\\
y_n
\end{bmatrix}
,$ degree $k-1$ polynomial $p(x)=\sum_{i=0}^{k-1} a_i x^i$ that approximates the function $f(x)=x$ for all $x\in [\sigma_{\min},\sigma_{\max}],$ where $\sigma_{\min}$ and $\sigma_{\max}$ are the minimum and maximum singular values of the covariance of the distribution from which the $\x_i$'s are drawn, and function $F_g$ which maps $\E_{x\sim N(0,1)}[(g(\|\beta^T\Sigma^{1/2}\|x)-\frac{1}{2})x]$ to $\frac{1}{2}-\E_{x\sim N(0,1)}[(g(\|\beta^T\Sigma^{1/2}\|x)-\frac{1}{2})]$.  (See Figure~\ref{fig:sig} for a plot of such a function in the case that $g$ is the sigmoid function.)
\begin{itemize}
\item Set $A = XX^T$, and let $G = A_{up}$ be the matrix A with the diagonal and lower triangular entries set to zero.
\item Let $t = \frac{\sqrt{\sum_{i=0}^{k-1} a_i\frac{y^TG^{i+1}y}{\binom{n}{i+2}}}}{2}$, which is an estimate of $\E_{x\sim N(0,1)}[(g(\|\beta^T\Sigma^{1/2}\|x)-\frac{1}{2})x].$
\end{itemize}
\STATE \textbf{Output}: $F_g(t)$
\caption{Estimating Classification Error, General Covariance\label{alg:bin}}
\end{algorithmic}
\end{algorithm}

\begin{figure}[h]
\centering
    \includegraphics[width=0.25\linewidth]{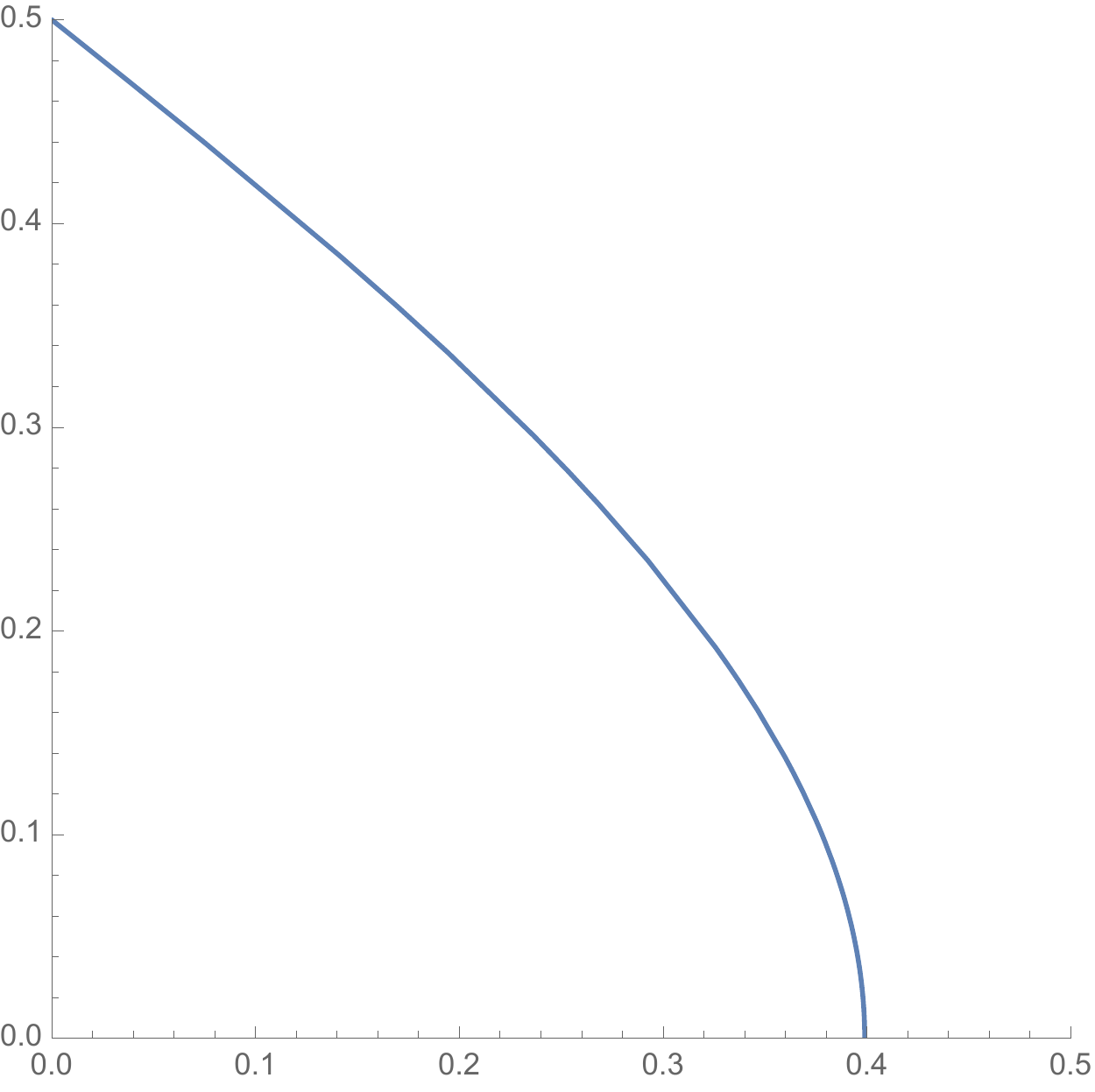}
    \caption{A plot of the function $F_g$ in the case where $g$ is the sigmoid function, showing the one-to-one relationship between the quantity $\E_{x\sim N(0,1)}[(g(\|\beta^T\Sigma^{1/2}\|x)-\frac{1}{2})x]$(x axis) which we estimate directly, and the quantity $\frac{1}{2}-\E_{x\sim N(0,1)}[(g(\|\beta^T\Sigma^{1/2}\|x)-\frac{1}{2})]$(y axis) which is the classification error of the best linear classifier.  As we show in Proposition~\ref{prop:bin-mapping}, an $\eps$-accurate approximation of the former can be mapped to a $\sqrt{\eps}$-accurate approximation of the latter.  As is evident from the figure, the derivative is bounded in magnitude provided the optimal error (y axis) is bounded away from 0, and hence in this regime the dependence improves from $\sqrt{\eps}$ to $O(\eps)$.\label{fig:sig}}    
\end{figure}

For convenience, we restate our main theorem for estimating the classification error of the best linear model:\\

\medskip
\vspace{.2cm}\noindent \textbf{Theorem~\ref{thm:bin-gen}.} \emph{ Suppose we are given $n<d$ labeled examples, $(\x_1,y_1),\ldots,(\x_n,y_n)$, with $\x_i$ drawn independently from a Gaussian distribution with mean $0$ and covariance $\Sigma$ where $\Sigma$ is an unknown arbitrary $d$ by $d$ real matrix.  Assuming that each label $y_i$ takes value $1$ with probability $g(\beta^T\x_i)$ and $-1$ with probability $1-g(\beta^T\x_i)$, where $g(x)=\frac{1}{1+e^{-x}}$ is the sigmoid function. There is an algorithm that takes $n$ labeled samples, parameter $k$, ${\sigma_{max}}$ and ${\sigma_{min}}$ which satisfies $\sigma_{max}I\succeq S^TS\succeq {\sigma_{min}I}$, and with probability $1-\tau$, outputs an estimate $\widehat{err_{opt}}$ with additive error $|\widehat{err_{opt}}-err_{opt}|\le c\Big(\sqrt{\min(\frac{1}{k^2},e^{-(k-1)\sqrt{\frac{\sigma_{min}}{\sigma_{max}}}})\sigma_{max}\|\beta\|^2+\frac{f(k)}{\tau}\sum_{i=2}^k \frac{d^{i/2-1/2}}{n^{i/2}}}\Big),$ where $err_{opt}$ is the classification error of the best linear classifier, $f(k)=k^{O(k)}$ and $c$ is an absolute constant.}
\medskip

We provide the proof of Theorem~\ref{thm:bin-gen} in Appendix~\ref{ap:bin-gen}.   As in the linear regression setting, the main technical challenge is bounding the variance of our estimators for each of the ``higher moments'', in this case for the estimators for the expressions $\E[(g(\beta^T\x)-\frac{1}{2})\frac{\beta^T\x}{\beta^T\Sigma\beta}]^2\beta\Sigma^k\beta$ for $k \ge 2$.  Our proof that these quantities can be accurately estimated in the sublinear data regime does leverage the Gaussianity assumption on $\x$, though does not rely on the assumption that the ``link function'' $g$ is the sigmoid.  The only portion of our algorithm and proof that leverages the assumption that $g$ is  the sigmoid function is in the definition and analysis of the function $F_g$ (of Algorithm~\ref{alg:bin}), which provides the invertible mapping between the quantity we estimate directly, $\E_{x\sim N(0,1)}[(g(\|\beta^T\Sigma^{1/2}\|x)-\frac{1}{2})x]$, and the classification error of the best predictor, $\frac{1}{2}-\E_{x\sim N(0,1)}[(g(\|\beta^T\Sigma^{1/2}\|x)-\frac{1}{2})]$.   Analogous results to Theorem~\ref{thm:bin-gen} can likely be obtained easily for other choices of link function, by characterizing the corresponding mapping $F_g$.

\section{Empirical Results}\label{sec:empirical}

We evaluated the performance of our estimators on several synthetic datasets, and on a natural language processing regression task.  In both cases, we explored the performance across a large range of dimensionalities. In both the synthetic and NLP setting, we compared our estimators with the ``naive'' unbiased estimator, $(\y-X\hat{\beta})^T(\y-X\hat{\beta})/(n-d)$, discussed in Section~\ref{sec:relatedWork}, which is only applicable in the regime where the sample size is at least the dimension.   In general, the results seem quite promising, with the estimators of Algorithms 1 and 2 yielding consistently accurate estimates of the proportion of the variance in the label that cannot be explained via a linear model.  As expected, the performance becomes more impressive as the dimension of the data increases.  All experiments were run in Matlab v2016b running on a MacBook Pro laptop, and the code will be available from our websites.

\subsection{Implementation Details}
Algorithms 1 and 2 were implemented as described in Section~\ref{sec:estimators}.  The only hitherto unspecified portion of the estimators is the choice of the coefficients $a_0, a_1,\ldots,a_k$ in the polynomial interpolation portion of Algorithm 2, which is necessary for our ``moment''-based approach to the non-isotropic setting.  Recall that the algorithm takes, as input, an upper and lower bound, $s_l, s_r$ on the singular values of the data distribution, and then approximates the linear function $f(x)=x$ in the interval $[s_l,s_r]$ via a polynomial of the form $a_0 x^2 + a_1 x^3 +\ldots+a_k x^{k+2}.$  The $\ell_{\infty}$ error of this polynomial approximation corresponds to an upper bound on the bias of the resulting estimator, and the variance of the estimator will increase according to the magnitudes of the coefficients $a_i$.   To compute these coefficients, we proceed via two small linear programs.  The variables of the LPs correspond to the $k$ coefficients, and the objective function of the first LP corresponds to minimizing the $\ell_{\infty}$ error of approximation, estimated based on a fine discretization of the range $[s_l,s_r]$ into $1000$ evenly spaced points.  Specifically, the function $f(x)=x$ is represented as a vector $(x_1,\ldots,x_{1000})$ with $x_1=s_l$ and $x_{1000}=s_r$ as are the basis functions $x^2,x^3,\ldots,x^k$.  The first LP computes the optimal $\ell_{\infty}$ approximation error given $s_l,s_r$ and the number of moments, $k$.  The second LP then computes coefficients that minimize the sum of the magnitudes of the coefficients (with the magnitude of the $i$th coefficient weighted by $2^i$ to account for the higher variance of these moments), subject to incurring an $\ell_{\infty}$ error that is not too much larger (at most a factor of $3/2$ larger) than the optimal one computed via the first LP.  We did not explore alternate weightings, and the results are similar if the factor of $3/2$ is replaced by any value in the range $[1.1, 2].$

\subsection{Synthetic Data Experiments}\label{sec:syn}
\noindent \textbf{Isotropic Covariance:} 
Our first experiments evaluate Algorithm 1 on data drawn from an isotropic Gaussian distribution.  In this experiment, $n$  datapoints $\x_1,\ldots,\x_n \in \R^d$ are drawn from an isotropic Gaussian, $N(0,I_d)$.  The labels $y_1,\ldots,y_n$ are computed by first selecting a uniformly random vector, $\beta$, with norm $\|\beta\| = \sqrt{1-\delta^2},$ and then setting each $y_i = \beta^T \x_i + \eta$ where $\eta$ is drawn independently from $N(0,\delta^2).$  The $y_i$'s are then scaled according to their empirical variance (simulating the setting where we do not know, a priori, that the labels have variance 1), and the magnitude of the fraction of this (unit) variance that is unexplained via a linear model is computed via Algorithm 1.   Figure~\ref{fig:iso_all} depicts the mean and standard deviation (over 50 trials) of the estimated value of unexplainable variance, $\delta^2$,  for three choices of the dimension, $d=$1,000, $d=$10,000, and $d=$50,000, and a range of choices of $n$ for each $d$.  We compare our estimator with the classic ``unbiased'' estimator in the settings when $n>d$.  We also include the test and training performance of the Bayes-optimal linear predictor, which corresponds to solving the $\ell_2$ regularized regression with optimal regularization parameter chosen as a function of the true variance of the noise.  As expected our estimator demonstrates an ability to accurately recover $\delta^2$ even in the sublinear data regime in which it is not possible to learn an accurate model, and the ``unbiased'' estimator has a variance that increases when $n$ is not much larger than $d$.   Figure~\ref{fig:iso_all} portrays the setting where $\delta^2 = 1/3$, and the results for other choices of $\delta^2 \in [0,1]$ are similar.

\noindent \textbf{Non-Isotropic Covariance:} 
We also evaluated Algorithm 2 on synthetic data that does not have identity covariance.   In this experiment, $n$  datapoints $x_1,\ldots,x_n \in \R^d$ are drawn from a uniformly randomly rotated Gaussian $G$ with covariance with singular values $1/d,2/d,3/d,\ldots,1$.  As above, the labels are computed by selecting $\beta$ uniformly and then scaling $\beta$ such $\Var[\beta^T x] = 1-\delta^2.$ The labels are assigned as $y_i=\beta^Tx_i +\eta$, and are then scaled according to their empirical variance.  We then applied Algorithm 2 with $k=1,2,3,4$ moments. Figure~\ref{fig:unif_all} depicts the mean and standard deviation (over 50 trials) of the recovered estimates of $\delta^2$ for the same parameter settings as in the isotropic case ($d=$1,000, $d=$10,000, and $d=$50,000, evaluated for a range of sample sizes, $n$).  For clarity, we only plot the results corresponding to using 2 and 3 moments; as expected, 2-moment estimator is significantly biased, whereas the for 3 (and higher) moments, the bias is negligible compared to the variance.  Again, the results are more impressive for larger $n$, and demonstrate the ability of Algorithm 2 to perform well in the sublinear sample setting where $n<d$.
\begin{figure}[h!]
\centering
    \includegraphics[width=1\linewidth]{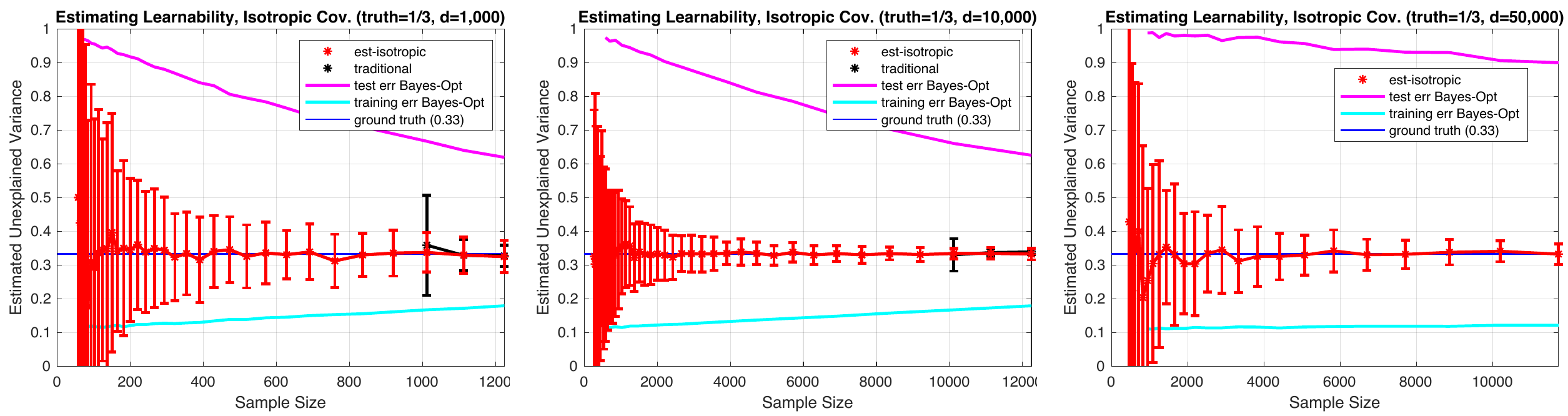}
    \caption{Evaluation of Algorithm 1 (est-isotropic) and the classic ``unbiased'' estimator (tradition) on synthetic identity-covariance data.  Plots depict the mean and standard deviation (based on 50 trials) of the estimate of the fraction of the label variance that cannot be explained via a linear model, in a variety of parameter regimes.   For comparison, we also included the test and training performance of the Bayes-optimal predictor (corresponding to $\ell_2$ regularized regression with optimal regularization parameter chosen as a function of the true variance of the noise).  See Section~\ref{sec:syn} for a complete description of the experimental setting.\label{fig:iso_all}}    
\end{figure}
\begin{figure}[h!]
\centering
    \includegraphics[width=1\linewidth]{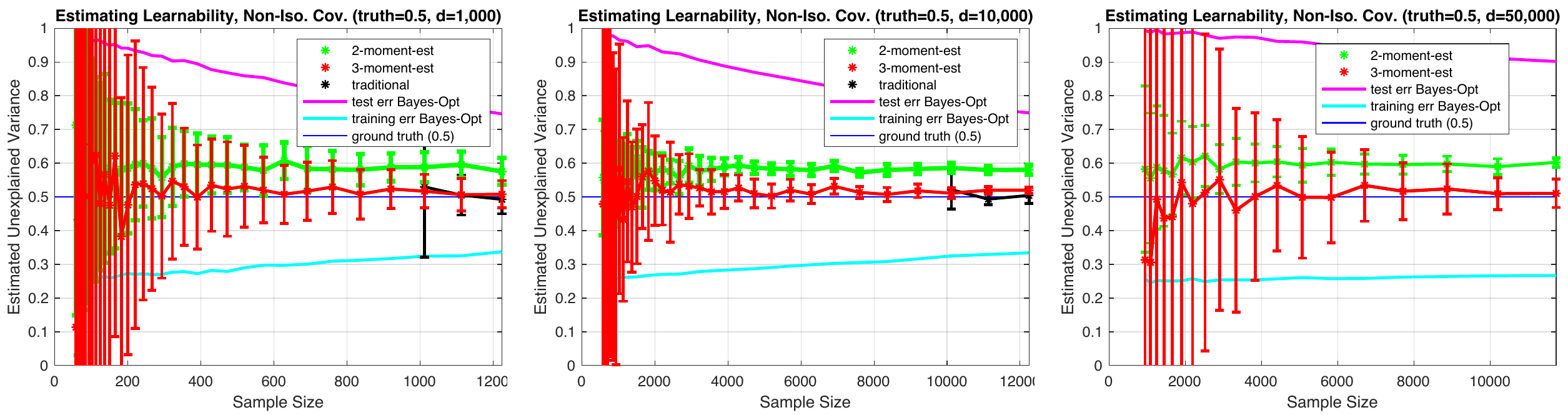}
    \caption{Evaluation of Algorithm 2 (using 2 and 3 moments) and the classic ``unbiased'' estimator (tradition) on synthetic data with covariance spectrum uniformly distributed between 0 and 1.  Plots depict mean and standard deviation (based on 50 trials). As expected, the 2-moment estimator has a significant bias. See Section~\ref{sec:syn} for a complete description of the experimental setting.  The test and training performance of the Bayes-optimal predictor are also shown, for comparison. \label{fig:unif_all}}    
\end{figure}
\subsection{NLP Experiments}\label{sec:nlp}
We also evaluated our approach on an amusing natural language processing dataset: predicting the ``point score'' of a wine (the scale used by \emph{Wine Spectator} to quantify the quality of a wine), based on a description of the tasting notes of the wine.  This data is from Kaggle's  Wine-Reviews dataset, originally scraped from Wine Spectator.   The dataset contained data on 150,000 highly-rated wines, each of which had an integral point score in the range $[80,100]$.  The tasting notes  consisted of several sentences, with each entry having a mean and median length of 40.1 and 39 words---95\% of the tasting notes contained between 20 and 70 words.  The following is a typical tasting note (corresponding to a 96 point wine): \emph{Ripe aromas of fig, blackberry and cassis are softened and sweetened by a slathering of oaky chocolate and vanilla. This is full, layered, intense and cushioned on the palate, with rich flavors of chocolaty black fruits and baking spices\ldots.} 

Our goal was to estimate the ability of a linear model (over various featurizations of the tasting notes) to predict the corresponding point value of the wine.  This dataset was well-suited for our setting because 1) the NLP setting presents a variety of natural high-dimensional featurizations, and 2) the 150k datapoints were sufficient to accurately estimate a ``ground truth'' prediction error, allowing us to approximate the residual variance in the point value that cannot be captured via a linear model over the specified features.  

We considered two featurizations of the tasting notes, both based on the publicly available 100-dimensional GloVe word vectors~\cite{pennington2014glove}.  The first, very naive featurization, consisted of concatenating the vectors corresponding to the first 20 words of each tasting note (this was capable of explaining $\approx 30\%$ of the variance of held-out points---for comparison, using the average of all the word vectors of each note explained $\approx 34\%$ of the variance).  We also considered a much higher-dimensional embedding, yielded by computing the $100^2$-dimensional outerproduct of vectors corresponding to each pair of words appearing in a tasting note, and then averaging these.  This was capable of explaining $\approx 53\%$ of the variance in the point scores.   In both settings, we leveraged the (unlabeled) large dataset to partially ``whiten'' the covariance, by reprojecting the data so as to have covariance with singular values in the range $[1/2,1],$ and removing the $5\%$ or $10\%$ of dimensions with smallest variance, yielding datasets with dimension 1,950 and 9,000, respectively.  The results of applying Algorithm 2 to these datasets are depicted in Figure~\ref{fig:wine_preds}.  The results are promising, and are consistent with the synthetic experiments.  We also note that the classic ``unbiased'' estimator is significantly biased when $n$ is close to $d$---this is likely due to the lack of independence between the ``noise'' in the point score, and the tasting note, and would be explained by the presence of sets of datapoints with similar point values and similar tasting notes.  Perhaps surprisingly, our estimator did not seem to suffer this bias.

\begin{figure}[h!]
\centering
    \includegraphics[width=1\linewidth]{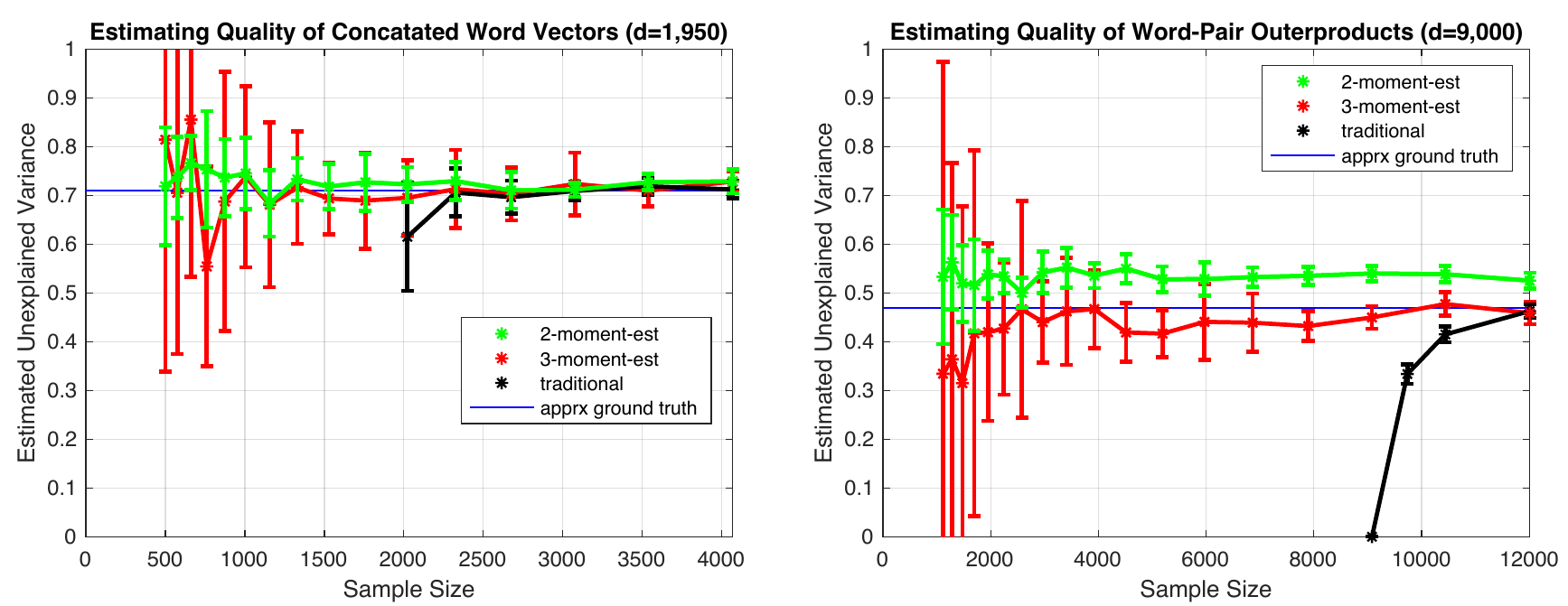}
    \caption{Evaluation of Algorithm 2 (using 2 and 3 moments) and the classic ``unbiased'' estimator (tradition) to predict the ``point value' of a wine, based on a $\approx 40$ word ``tasting note''. Ground truth is estimated based on 150k datapoints. All data is from the Kaggle ``Wine Reviews'' dataset.  The left plot depicts a naive featurization with $d=$1,950, and the right plot depicts a quadratic embedding of pairs of words, with $d=$9,000, which can explain more of the variance in the point scores.  See Section~\ref{sec:nlp} for a further discussion of these results.\label{fig:wine_preds}} 
\end{figure}
\subsection{Binary Classification Experiments}

We evaluated our estimator for the prediction accuracy of the best linear classifier on 1) synthetic data (with non-isotropic covariance) that was drawn according to the specific model to which our theoretical results apply, and 2) the MNIST hand-written digit image classification dataset.  Our algorithm performed well in both settings---perhaps suggesting that the theoretical performance characterization of our algorithm might continue to hold in significantly more general settings beyond those assumed in Theorem~\ref{thm:bin-gen}.

\subsubsection{Synthetic Data Experiments}\label{sec:bin-syn}

We evaluated Algorithm~\ref{alg:bin} on synthetic data with non-isotropic covariance. In this experiment, $n$  datapoints $x_1,\ldots,x_n \in \R^d$ are drawn from a uniformly randomly rotated Gaussian $G$ with covariance with singular values $1/d,2/d,3/d,\ldots,1$. Model parameter $\beta$ is a $d$-dimensional vector with $\|\beta\|=2$ that points in an uniformly random direction. Each label $y_1,\ldots,y_n$ is assigned by setting $y_i$ to be $1$ with probability $g(\beta^Tx_i)$ and $-1$ with probability $1-g(\beta^Tx_i)$, where $g(x)$ is the sigmoid function. We then applied Algorithm 3 with $k=3$ moments. Figure~\ref{fig:bin-syn} depicts the mean and standard deviation (over 50 trials) of the recovered estimates of the classification error of the best linear classifier.  We considered dimension $d=$1,000, and $d=$10,000, and evaluated each setting for a range of sample sizes, $n$.  For context, we also plotted the test and training accuracy of the logistic regression algorithm with $\ell_2$ regularization parameter $1/n$. Again, the performance of our algorithm seems more impressive for larger $d$, and demonstrates the ability of Algorithm~\ref{alg:bin} to perform well in the sublinear sample setting where $n<d$ and the (regularized) logistic regression algorithm can not recover an accurate classifier.
\begin{figure}[h]
\centering
    \includegraphics[width=1\linewidth]{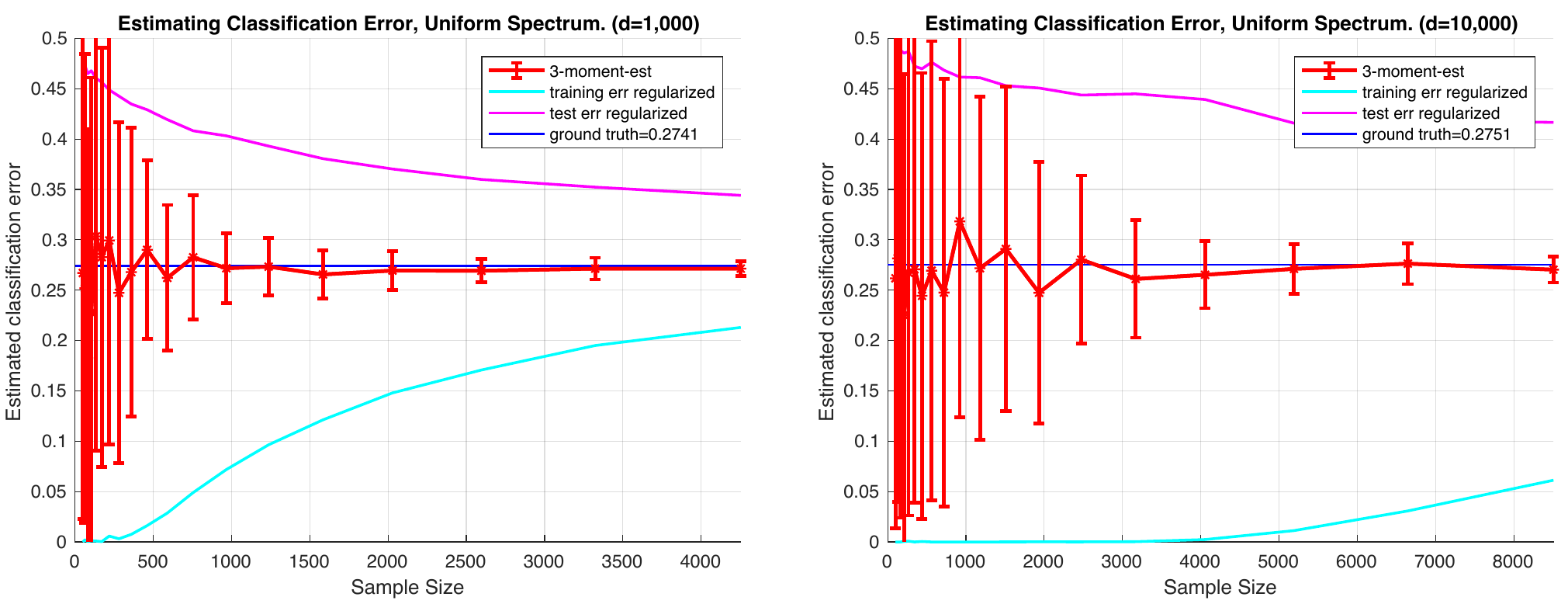}
    \caption{Evaluation of Algorithm \ref{alg:bin} using $3$ moments (3-moment-est) and the $\ell_2$ regularized logistic regression estimator (training err regularized, test err regularized) on synthetic data with covariance spectrum uniformly distributed between 0 and 1.  Plots depict mean and standard deviation (based on 50 trials). See Section~\ref{sec:bin-syn} for a complete description of the experimental setting.\label{fig:bin-syn}}    
\end{figure}

\subsubsection{MNIST Image Classification}\label{sec:bin-MNIST}
We also evaluated our algorithm for predicting the classification error on the MNIST dataset. The MNIST dataset of handwritten digits has a training set with 60,000 grey-scale images. Each image is a handwritten digit with $28\times28$ resolution, and each grey-scale pixel is represented as an integer between 0 and 255. Our goal is to estimate the ability of a linear classifier to predict the label (digit) given the image. Since we are only considering binary linear classifier in this work, we take digits ``0'',``1'',``2'',``3'',``4'' as positive examples and ``5'',``6'',``7'',``8'',``9'' as negative examples, and the task is to determine which group an image belongs to. Since our algorithm requires the dataset to be balanced in terms of positive and negative examples, we subsample from the majority class to obtain a balanced dataset with $58808$ total training examples (20404 each class). Each image is unrolled to a $28^2= 784$ dimensional real vector ($d=784$). All the data are $0$ centered and scaled so the largest singular value of the sample covariance matrix is $1$. For comparison, we implemented logistic regression with no regularization and with  $\ell_2$ regularization with parameter $\lambda=1/n$. We also use the simpler (and more robust) function $F_g(t) = 0.5-t$, which is the linear approximation to the $F_g$ that corresponds to the sigmoid under the Gaussian distribution. Algorithm 3 is applied with $k=3$ and $k=8$ moments. For each sample size $n$, we randomly select $n$ samples from the set of size 58,808. To evaluate the test performance of logistic regression, we use the remaining examples as a ``test'' set.  For each algorith, we repeat 50 times, reselecting the $n$ samples, etc.  Figure~\ref{fig:bin-MNIST} depicts the mean and standard deviation (over 50 trials) of the recovered estimate of the classification error of the best linear classifier. 

As shown in the plot, even with $1,500\approx 2d$ samples, the training error of the unregularized logistic regression is still $0$, meaning the data is perfectly separable, and the learned classifier does not generalize. Although the conditions of Theorem~\ref{thm:bin-gen} obviously do not hold for the MNIST dataset, our algorithm still provides a reasonable estimate even with less than $400\approx d/2$ samples. One interesting phenomenon in the MNIST experiment compared to the synthetic dataset is that the high order moments are smaller both in terms of the value and standard deviation, hence using more moments does not introduce significantly more variance, and still decreases the bias. In the experiments we found that the estimates computed with Algorithm~\ref{alg:bin} are stable even with 12 moments. 

\subsubsection*{Acknowledgments}
This work was supported by NSF awards AF-1813049 and CCF-1704417, ONR Young Investigator Award N00014-18-1-2295, and a Google Faculty Fellowship.

\begin{figure}[h!]
\centering
    \includegraphics[width=.5\linewidth]{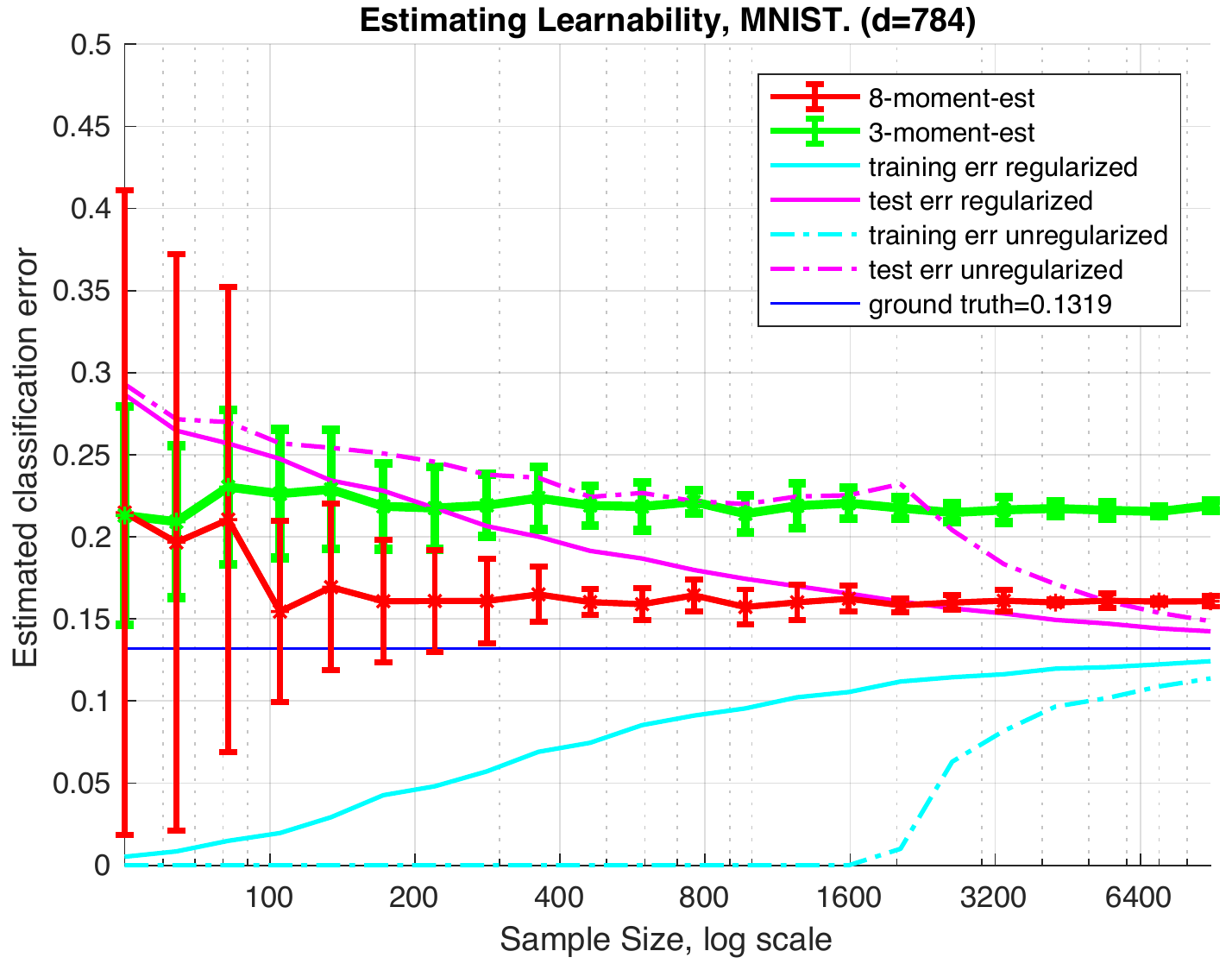}
    \caption{Evaluation of Algorithm~\ref{alg:bin} using $3$ and $8$ moments to predict the classification error of the best linear classifier on the MNIST dataset (to distinguish the class of digits $\{0,1,2,3,4\}$ versus $\{5, 6, 7, 8, 9\}$.  For comparison, we plot the test and training error for unregularized logistic regression and $\ell_2$-regularized logistic regression. The ground truth is the average of the training and testing classification error based on training on 50k datapoints and testing on the remaining datapoints. Plots depict mean and standard deviation (based on 50 trials). See Section~\ref{sec:bin-MNIST} for a complete description of the experimental setting.\label{fig:bin-MNIST}}    
\end{figure}
\newpage
\bibliography{reference}
\bibliographystyle{plain}

\newpage
\newpage

\onecolumn

\appendix
In this appendix, we give self-contained proofs of the theoretical results.

\section{Identity Covariance}

Suppose we are given $n$ labeled examples, $(\x_1,y_1),\ldots,(\x_n,y_n)$, with $(\x_i,y_i)$ drawn independently from a $d+1$-dimension distribution where $\x_i$ has mean zero and identity covariance, and $y_i$ has mean zero and variance $1$, and the fourth moments of the joint distribution are bounded by a constant $C$, namely for all vectors $\u,\v \in \R^d$, $\E[(\x^T\u)^2(\x^T\v)^2]\le C\E[(\x^T\u)^2]\E[(\x^T\v)^2]$ and $\E[(\x^T\u)^2y^2]\le C\E[(\x^T\u)^2]\E[y^2]$. We have the following theorem which guarantees the accuracy of the estimate provided by Algorithm~\ref{alg:iso}.

\vspace{.2cm}\noindent \textbf{Theorem~\ref{thm:agnostic}.} \emph{  
With probability $1-\tau$, Algorithm~\ref{alg:iso} outputs $\hat{\delta}^2$, which is an estimate of $\delta^2 : = \min_{\beta'}\E[({\beta'}^T\x-y)^2]$ that satisfies $|\hat{\delta}^2-\delta^2|\le O(C\frac{\sqrt{d+n}}{\tau n})$ 
}

The proof will follow from combining the fact that $\frac{\y^T G\y}{\binom{n}{2}}$ and  $\frac{\y^T\y}{n}$ are unbiased estimators of $\|\beta\|^2$ and $\E[y^2]$, respectively, and Proposition~\ref{prop:bs2b} which bounds the variances as $\Var[\frac{\y^T G\y}{\binom{n}{2}}] = O(\frac{C^2(d+n)}{n^2})$ and $\Var[\frac{\y^T\y}{n}] = \frac{\E[y^4]}{n} = O(\frac{1}{n})$. Through a simple Chebyshev's inequality argument, we achieve the desired error bound.

\begin{proof}[Proof of Theorem~\ref{thm:agnostic}] 
By Proposition~\ref{prop:bs2b}, $\Var[\frac{\y^TG\y}{\binom{n}{2}}] \le O( \frac{C^2(d+n)}{n^2}).$ Hence by Chebyshev's inequality, with probability $1-\tau/2$, $|\frac{\y^TG\y}{\binom{n}{2}}-\|\beta\|^2|\le O(C\frac{\sqrt{d+n}}{\tau n})$. Again by Chebyshev's inequality, with probability at least $1-\tau/2$, $|\frac{\y^T\y}{n}-\E[y^2]|\le O\big(\frac{1}{\sqrt{n}}\big)$. Thus we have $|\frac{y^Ty}{n} - \frac{y^TGy}{\binom{n}{2}}-\min_{\beta'}\E[({\beta'}^T\x-y)^2]|= |\frac{y^Ty}{n} - \frac{y^TGy}{\binom{n}{2}}-(\E[y^2]-\|\beta\|^2)]|\le O\big(C\frac{\sqrt{d+n}}{\tau n}\big)$ with probability $1-\tau$. 
\end{proof}

In case where the covariance of $\x$ is close to identity, the estimation error of Algorithm~\ref{alg:iso} is formalized in the following corollary. 

\vspace{.2cm}\noindent \textbf{Corollary~\ref{cor:approxid}.} \emph{Suppose we are given $n$ labeled examples, $(\x_1,y_1),\ldots,(\x_n,y_n)$, with $(\x_i,y_i)$ drawn independently from a $d+1$-dimensional distribution where $\x_i$ has mean zero and covariance $\Sigma$ which satisfies $(1-\epsilon)I \preceq \Sigma\preceq (1+\epsilon)I$, and $y_i$ has mean zero and variance $1$, and the fourth moments of the joint distribution $(x,y)$ is bounded by $C$. There is an estimator $\hat{\delta^2}$, that with probability $1-\tau$, approximates $\min_{\beta}\E[({\beta}^T\x-y)^2]$ with additive error $O(C\frac{\sqrt{d+n}}{\tau n}+\epsilon)$.
}
\begin{proof}
For arbitrary $\Sigma$, $\frac{\y^T G\y}{\binom{n}{2}}$ is an unbiased estimator of $\beta\Sigma^2\beta$ with variance $O(\frac{C^2(d+n)}{n^2})$. In the setting where $(1-\eps)I\preceq\Sigma\preceq (1+\eps)I$, we have $\beta^T\Sigma^2\beta-\beta^T\Sigma\beta =O(\eps)$. Following the same argument of Theorem~\ref{thm:agnostic}, we have $|\frac{y^Ty}{n} - \frac{y^TGy}{\binom{n}{2}}-\min_{\beta'}\E[({\beta'}^T\x-y)^2]|= |\frac{y^Ty}{n} - \frac{y^TGy}{\binom{n}{2}}-(\E[y^2]-\beta^T\Sigma\beta)]| = |\frac{y^Ty}{n} - (\frac{y^TGy}{\binom{n}{2}}-\beta^T\Sigma^2\beta)-\big(\E[y^2]-(\beta^T\Sigma\beta-\beta^T\Sigma^2\beta)\big)]|\le O\big(C\frac{\sqrt{d+n}}{\tau n}+\eps\big)$ with probability $1-\tau$.
\end{proof}

\section{General Covariance}
Recall that in the general covariance setting, each sample $(\x,y)$ is drawn independently as follows:
First $x = S\z$ is drawn where $S$ is a $d\times d$ real matrix and $\z$ is a $d$-dimensional vector whose entries are i.i.d distributed with $\E[z_i]=0,\E[z_i^2]=1,\E[z_i^4]\le C$. Then let $y=\beta^T\x+\eta$ where $\eta$ is draw from a distribution $E$ with mean $0$ and variance $\delta^2$. The covariance of $\x$ is $\Sigma = SS^T$ and the variance of $y$ is $\beta^T\Sigma\beta+\delta^2$ which is assumed to be $1$ for the rest of the section. Notation $SS^T$ is always equivalent to $\Sigma$, and both are used throughout this section. 

We restate the main theorem of this section for convenience.

\vspace{.2cm}\noindent \textbf{Theorem~\ref{thm:gen}.} \emph{
Suppose we are given $n<d$ labeled examples, $(\x_1,y_1),\ldots,(\x_n,y_n)$, with $\x_i=S\z_i$ where $S$ is an unknown arbitrary $d \times d$ real matrix and each entry of $\z_i$ is drawn independently from a one dimensional distribution with mean zero, variance $1$, and constant fourth moment.  Assuming that each label $y_i = \x_i \beta + \eta$, where the noise $\eta$ is drawn independently from an unknown distribution $E$ with mean 0 and variance $\delta^2$, and the labels have been normalized to have unit variance. There is an algorithm that takes $n$ labeled samples, parameter $k$, ${\sigma_{max}}$,${\sigma_{min}}$ which satisfies $\sigma_{max}I\succeq S^TS\succeq {\sigma_{min}I}$, and with probability $1-\tau$, outputs an estimate $\hat{\delta}^2$ with additive error $|\hat{\delta}^2-\delta^2|\le \min(\frac{2}{k^2},2e^{-(k-1)\sqrt{\frac{{\sigma_{min}}}{{\sigma_{max}}}}})){\sigma_{max}}\|\beta\|^2+\frac{f(k)}{\tau}\sum_{i=2}^k \frac{d^{i/2-1/2}}{n^{i/2}},$ where $f(k)=k^{O(k)}$.}

Recall that the variance of $y$ is $\beta^T\Sigma\beta+\delta^2$ and can be estimated up to error $\frac{1}{\sqrt{n}}$. Hence the only missing part for estimating $\delta^2$ is an estimator of $\beta^T\Sigma\beta$. The proof of our main theorem relies on Propositions~\ref{prop:gen-unbiased}, ~\ref{prop:gen-var}, and~\ref{prop:poly-approx}.   Proposition~\ref{prop:gen-unbiased} shows that there is a series of unbiased estimators of $\beta^T\Sigma^k\beta$ for each $k\ge 2$. Proposition~\ref{prop:gen-var} gives a variance bound for the series estimators, which yields accuracy guarantees when combined with Chebyshev's inequality.  Finally, Proposition~\ref{prop:poly-approx} provides a series of polynomials that approximates $f(x)=x$. Combining these estimates of $\beta^T\Sigma^k\beta$ for $k=1,2,\ldots$ and the coefficients provided by Proposition~\ref{prop:poly-approx}, we obtain an accurate estimate of $\beta^T\Sigma\beta$.

\begin{proposition}\label{prop:gen-unbiased}
$\E[\y^TG^k\y] = \binom{n}{k+1}\beta^T\Sigma^{k+1}\beta$.
\end{proposition}
\begin{proof}
Expanding $\y^TG^k\y$ we get the following summation:
$$\sum_{i_1,i_2,\ldots,i_{k+1}}y_{i_1}G_{i_1,i_2}\ldots G_{i_k,i_{k+1}}y_{i_{k+1}}.$$
Because $G$ is a strictly upper triangular matrix, this is equivalent to: 
$$\sum_{i_1<i_2<\ldots<i_{k+1}}y_{i_1}G_{i_1,i_2}\ldots G_{i_k,i_{k+1}}y_{i_{k+1}}.$$
By the definition of $G$, the formula is further equal to
\begin{equation*}
\sum_{i_1<i_2<\ldots<i_{k+1}}y_{i_1}\x_{i_1}^T\x_{i_2}\ldots \x_{i_k}^T\x_{i_{k+1}}y_{i_{k+1}}.
\end{equation*}
Finally, taking the expectation we get $\beta^T \Sigma^{k+1}\beta$.
\end{proof}
\begin{proposition}\label{prop:gen-var}
$\Var[\frac{\y^TG^k\y}{\binom{n}{k+1}}]\le f(k)\max(\frac{d^k}{n^{k+1}},\frac{1}{n})$, where $f(k)=2^{12(k+2)}(k+1)^{6(k+1)}C^{k+1}\sigma_1^{2k}$.
\end{proposition}
The proof of this proposition is quite involved, and Section~\ref{app:gen-var} is devoted to this proof.

\begin{proposition}\label{prop:poly-approx}
For any integer $i \ge 2$, there is a degree $i$ polynomial $p_i(x)$ with no linear or constant terms, satisfying $|p_i(x)-x|\le \frac{2}{i(i+1)}$ for all $x\in [0,1]$. For any integer $i \ge 2$, there is a degree $i$ polynomial $p_i(x)$ with no linear or constant terms, satisfying  $|p(x)-x|\le 2e^{-(i-1)\sqrt{b}}$ for all $x\in [b,1]$.
\end{proposition}
\begin{proof}
To prove the first statement of the proposition, we use Muntz polynomials to approximate the monomial $x$. Applying Theorem 5.5 in~\cite{devore1993constructive} with $\lambda_1 = 2,\lambda_2 = 3,\ldots,\lambda_{i-1} = i$ yields that there exists a polynomial $p(x)=\sum_{j=2}^i a_jx^j $ such that $|p(x)-x|\le \frac{2}{i(i+1)}$ for $\x\in [0,1]$. The second statement of the proposition (inverse exponential error) is a restatement of Lemma 7.8 in~\cite{orecchia2012approximating} after multiplying by $x$ on the both sides of the inequality.
\end{proof}
\medskip

We now complete the proof of Theorem~\ref{thm:gen}.
\begin{proof}[Proof of Theorem~\ref{thm:gen}]
We first divide the each sample $\x_i$ by  $\sigma_{max}$ and run Algorithm~\ref{alg:gen} with parameter $k$ and the polynomial constructed from Proposition~\ref{prop:poly-approx}. We denote the weight vector and covariance matrix after scaling as $\beta'$ and $\Sigma'$ which are simply $\sqrt{{\sigma_{max}}}\beta$ and $\frac{\Sigma}{{\sigma_{max}}}$(i.e. $\frac{S^TS}{{\sigma_{max}}}$). Notice that this step does not change the signal ratio. Observe that by using the polynomial coefficient from Proposition~\ref{prop:poly-approx}, we have $|{\beta'}^T\Sigma'\beta' - \sum_{i=0}^{k-2} a_i{\beta'}^T {\Sigma'}^{i+2}\beta|\le \min(\frac{2}{k^2},2e^{-(i-1)\sqrt{\frac{\sigma_{min}}{\sigma_{max}}}})\sigma_{max}\|\beta\|^2$. By Proposition~\ref{prop:gen-unbiased}, we have a series of unbiased estimator of ${\beta'}^T{\Sigma'}^{k}{\beta'}$ for all $k\ge 2$. Further, by Chebyshev's inequality and Proposition~\ref{prop:gen-var}, we have that with probability $1-\tau$, we have an estimate for each ${\beta'}^T{\Sigma'}^k\beta'$ with additive error less than $\frac{f(k)}{\tau}\max(\frac{d^{k/2-1/2}}{n^{k/2}},\frac{1}{\sqrt{n}})$. Note that we use $f(k)$ to denote various functions that only depends on $k$. It is not hard to verify that all the coefficients of the degree $k$ polynomial provided in Proposition~\ref{prop:poly-approx} are less than $k^{O(k)}$. Thus, altogether, we obtain an estimate of $\beta^T\Sigma\beta$ with additive error less than $\frac{f(k)}{\tau}\sum_{i=2}^k \frac{d^{i/2-1/2}}{n^{i/2}}$, which gives the claimed estimation accuracy.
\end{proof}



\subsection{Proof of Proposition~\ref{prop:gen-var}}\label{app:gen-var}
Recall that in the proof of Proposition~\ref{prop:bs2b}, which can be viewed as a special case of Proposition~\ref{prop:gen-var} when $k=1$ and $\Sigma=I$, we expressed the variance as the summation of the product terms where each product is classified into one of the $3$ different cases according to the configuration of $i,j$. As a higher order analogy, it is natural to consider the same strategy. However, naive categorization will result in a combinatorial number of cases for large $k$. Hence we will need to develop a graph theoretical categorization mechanism to simplify the analysis of the cases. 

\vspace{.2cm}\noindent \textbf{Proposition~\ref{prop:gen-var}.}\emph{
$\Var[\frac{\y^TG^k\y}{\binom{n}{k+1}}]\le f(k)\max(\frac{d^k}{n^{k+1}},\frac{1}{n})$, where $f(k)=2^{12(k+2)}(k+1)^{6(k+1)}C^{k+1}\sigma_1^{2k}$.}

The remainder of this section is devoted to the proof.  To begin, observe that the term $\Var[\frac{\y^TG^k\y}{\binom{n}{k+1}}]$ can be expressed as:
$$
\frac{1}{\binom{n}{k+1}^2}\sum_{\pi_1<\pi_2<\ldots<\pi_{k+1},\pi'_1<\pi'_2<\ldots<\pi'_{k+1}} \Big(\E[y_{\pi_1}\x_{\pi_1}^T\x_{\pi_2}\ldots \x_{\pi_k}^T\x_{\pi_{k+1}}y_{\pi_{k+1}}y_{\pi'_1}\x_{\pi'_1}^T\x_{\pi'_2}\ldots \x_{\pi'_k}^T\x_{\pi'_{k+1}}y_{\pi'_{k+1}}]-(\beta^T\Sigma^{k+1}\beta)^2\Big)
$$
For ease of notation, we use $\pi$ to denote the set of indices $\pi_1,\pi_2,\ldots,\pi_{k+1}$ and similarly for $\pi'$. First we bound the expectation of the sum of the products that does not involve $\eta$. Pick a term with index $\pi_1<\pi_2<\ldots<\pi_{k+1},\pi'_1<\pi'_2<\ldots<\pi'_{k+1}$ in the summation and pull out the terms that does not involve $\eta$, we get 
\begin{align*}
\E[\beta^T\x_{\pi_1}\x_{\pi_1}^T\x_{\pi_2}\ldots \x_{\pi_k}^T\x_{\pi_{k+1}}\x_{\pi_{k+1}}^T\beta \beta^T\x_{\pi'_1}\x_{\pi'_1}^T\x_{\pi'_2}\ldots \x_{\pi'_k}^T\x_{\pi'_{k+1}}\x^T_{\pi'_{k+1}}\beta]\\
 = \E[\Big(\z_{\pi_1}^T S^T S\z_{\pi_2}\ldots \z_{\pi_k}^T S^T S\z_{\pi_{k+1}}\z_{\pi_{k+1}}^T S^T\beta\beta^T S \z_{\pi_1}\Big) 
 \Big(\z_{\pi'_1}^T S^T S\z_{\pi'_2}\ldots \z_{\pi'_k}^T S^T S\z_{\pi'_{k+1}}\z_{\pi'_{k+1}}^T S^T\beta\beta^T S \z_{\pi'_1}\Big)]\\
 = \E[\sum_{\gamma,\delta,\gamma',\delta'}\Big(z_{\pi_{k+1},\gamma_{k+1}} {(S^T\beta\beta^T S)}_{\gamma_{k+1},\delta_1} z_{\pi_1,\delta_1} \Big)
 \Big(z_{\pi'_{k+1},\gamma'_{k+1}} {(S^T\beta\beta^T S)}_{\gamma'_{k+1},\delta'_1} z_{\pi_1,\delta_1}\Big)\\
 \prod_{j=1}^k\Big(z_{\pi_j,\gamma_j} {(S^T S)}_{\gamma_j,\delta_{j+1}}z_{\pi_{j+1},\delta_{j+1}}\Big)
 \Big(z_{\pi'_j,\gamma'_j} {(S^T S)}_{\gamma'_j,\delta'_{j+1}}z_{\pi'_{j+1},\delta'_{j+1}}\Big)]\\
\end{align*}
Notice that the only random variables here are the $z$'s, hence a natural idea is to group the terms together according to the expectation of the $z$ variables (i.e. $\E[\prod_{i=1}^{k+1} z_{\pi_i,\delta_i}z_{\pi_i,\gamma_i}z_{\pi'_i,\delta'_i}z_{\pi'_i,\gamma'_i}]$) before carrying out the summation. $\pi_1,\pi_2,\ldots,\pi_{k+1},\pi'_1,\pi'_2,\ldots,\pi'_{k+1}$ naturally defines a partition $P$ which groups the variables that take the same value together. Notice that each set of $P$ has size at most $2$. A partition $P$ of variables $\pi_1,\pi_2,\ldots,\pi_{k+1},\pi'_1,\pi'_2,\ldots,\pi'_{k+1}$  defines a partition $P^*$ as follows: for each set $\{\pi_i,\pi'_j\}\in P$, create a set $\{\delta_i,\gamma_i,\delta'_i,\gamma'_i\}$ in $P^*$ . Given a realization of variables $\delta_1,\gamma_1,\delta'_1,\gamma'_1,\ldots,\delta_{k+1},\gamma_{k+1},\delta'_{k+1},\gamma'_{k+1}$, we can define a refinement of $P^*$, called $Q$, by further partitioning each set in $P^*$ according to the values these variables take. Through this construction procedure, each realization of variables $\pi,\pi',\delta,\gamma,\delta',\gamma'$ uniquely defines a pair of partitions $(P,Q)$. We say that the variables $\pi,\pi',\delta,\delta',\gamma,\gamma'$ respects $(P,Q)$  (we denote this as $\pi,\pi',\delta,\delta',\gamma,\gamma'|P,Q$ for shorthand). With the above definition, we claim that any two variable realizations that respect the same $(P,Q)$ has the same expectation of $z$s:
\begin{fact}
Given a partition of variables $P,Q$, $\E[\prod_{i=1}^{k+1} z_{\pi_i,\delta_i}z_{\pi_i,\gamma_i}z_{\pi'_i,\delta'_i}z_{\pi'_i,\gamma'_i}]$ is the same for all realizations that respect $P,Q$.
\end{fact}
Notice that if a set in $Q$ has cardinality $1$, the expectation $\E[\prod_{i=1}^{k+1} z_{\pi_i,\delta_i}z_{\pi_i,\gamma_i}z_{\pi'_i,\delta'_i}z_{\pi'_i,\gamma'_i}]$ will be $0$, hence we can restrict our attention to assume that the cardinality of all the sets in $Q$ are either $2$ or $4$.

To facilitate the computation of the summation of $${(S^T\beta\beta^T S)}_{\gamma_{k+1},\delta_1}{(S^T\beta\beta^T S)}_{\gamma'_{k+1},\delta'_1} \prod_{j=1}^k {(S^T S)}_{\gamma_j,\delta_{j+1}}  {(S^T S)}_{\gamma'_j,\delta'_{j+1}}$$ over all variable realizations that respects $P,Q$, we define $PQ-Graph$ as follows:
\begin{definition}
Given $P=\{P_1,\ldots,P_m\}$ and $Q = \{Q_1,\ldots,Q_w\}$, the corresponding multigraph $PQ$-Graph is created as follows. We create a $P$-node for each set $P_i\in P$ and create a $Q$-node for each set $Q_i\in Q$. For each $i$, an $Q$-edge is created between the two $Q$-node that contains $\delta_i$ and $\gamma_{i+1}$ respectively. For each $i$, an $PQ$-edge is created between a $P$ node and a $Q$ node that contains $\pi_i,\delta_i$ respectively or contains $\pi_i,\gamma_i$ respectively. We create edges for $\pi', \delta',\gamma'$ analogously. 
\end{definition}
Notice that since each $Q_i$ has cardinality $2$ or $4$, every node has degree $2$ or $4$ in the subgraph induced by the $Q$-nodes which we called \textit{Q-Graph}. On the subgraph induced by the  $Q$-nodes, a \textit{free cycle} is defined to be a cycle that only contains nodes with degree $2$. An \textit{arc} is defined to be a simple path that connects nodes with degree $2$ except that the starting node and ending node have degree $4$. The induced graph can be uniquely decomposed into disjoint sets of free cycles and arcs. We have the following lemma regarding the maximum number of free cycles and arcs an induced subgraph can have. 
\begin{lemma}\label{lem:partitionsize1}
Given $P,Q$ such that $\prod_{i=1}^{k+1} z_{\pi_i,\delta_i}z_{\pi_i,\gamma_i}$ is not independent of $\prod_{i=1}^{k+1}z_{\pi'_i,\delta'_i}z_{\pi'_i,\gamma'_i}$ for $\pi,\delta,\gamma,\pi',\delta',\gamma'$ respecting $P,Q$: the number of arcs $\xi$ and the number of free cycles $\eta$ satisfies $\xi/2+\eta\le 2(k+1)-m$, which is the number of degree $4$ P-nodes in the PQ-graph. 
\end{lemma}
\begin{proof}
We prove the lemma by a counting argument. Observe that the $PQ$-graph consists of two closed walks, one corresponds to $\pi,\delta,\gamma$ and one corresponds to $\pi',\delta',\gamma'$. The following operation will be done in the Q-Graph. For each $P$ node with degree $2$, we remove the $Q$ node that is incident to it and connect the two neighboring $Q$ nodes. We argue that if a node $Q$ has degree $2$, it must not be incident to a self-loop. Suppose the Q-node is only incident to a selp-loop, the two PQ-edges of $Q$ must belong to the same closed walk, since otherwise the walk can not be closed. If that is the case, the two closed walks have no parallel edges, which implies the two products are independent and hence yields a contradiction. Notice that this operation does not change the number of arcs or free cycles. After removing these $Q$ nodes, we are left with two kinds of $Q$ nodes, the first kind has degree $2$ and belongs to a free cycle, the second kind has degree $4$ and belongs to $4$ arcs. Suppose there are $l_1$  $Q$ nodes of the first kind, and $l_2$ of the second kind. There will be at most $l_1/2$ free cycles and $2l_2$ arcs. Notice that we are left with $2(k+1)-m$ P-nodes and each P-node is connected to either $2$ Q-nodes of the first kind or $1$ Q-nodes of the second kinds which implies $l_1/2+l_2=2(k+1)-m$. Finally we have $\xi/2+\eta\le l_2+l_1/2=2(k+1)-m$, as desired.
\end{proof}
\begin{cor}\label{cor:partitionsize1}
If there exists an arc that consists of edges from two different walks, $\xi/2+\eta\le 2(k+1)-m-1/2$.
\end{cor}
\begin{proof}
There must be a Q-node in the arc that has degree 2 and is not removed in the procedure described in the proof of Lemma~\ref{lem:partitionsize1}. This Q-node is a first kind Q-node and does not belong to a free cycle. Hence we need to subtract $1/2$ free cycle from our counting argument, thus the Corollary is proved.
\end{proof}
Recall that we say that a realization of $\delta,\gamma,\delta',\gamma'$ respects $Q$ ($\delta,\gamma,\delta',\gamma'|Q$ for shorthand) if for each set $Q_i\in Q$, the variables in the set take the same value. The following key fact establish an upperbound for the summation of ${(S^T\beta\beta^T S)}_{\gamma_{k+1},\delta_1}{(S^T\beta\beta^T S)}_{\gamma'_{k+1},\delta'_1} \prod_{j=1}^k {(S^T S)}_{\gamma_j,\delta_{j+1}}  {(S^T S)}_{\gamma'_j,\delta'_{j+1}}$
\begin{fact} \label{fact:sumQ}
Remove all the arcs and free cycles that contains edge $(Q(\gamma_{k+1}),Q(\delta_1))$ or $(Q(\gamma'_{k+1}),Q(\delta'_1))$ from the Q-graph and denote the number of removed edges as $l$. Suppose we are left with $\xi$ arcs with lengths $l_1,\ldots,l_\xi$ and $\eta$ free cycles with lengths $p_1,\ldots,p_\eta$. We have  $\sum_{\delta,\delta',\gamma,\gamma'|Q}{(S^T\beta\beta^T S)}_{\gamma_{k+1},\delta_1}{(S^T\beta\beta^T S)}_{\gamma'_{k+1},\delta'_1}\\
\prod_{j=1}^k {(S^T S)}_{\gamma_j,\delta_{j+1}}  {(S^T S)}_{\gamma'_j,\delta'_{j+1}} \le \prod_{i=1}^\xi tr((S^T S)^{2l_i})^{1/2}\prod_{j=1}^\eta tr((S^T S)^{p_j}) (\beta^TSS^T\beta)^2\sigma_1^{l-2} \le \sigma_1^{2k}d^{2(k+1)-m-1}$, where $\sigma_1$ denotes the largest eigenvalue of $S^TS$ or equivalently of $SS^T$.
\end{fact}
\begin{proof}
The proof is very similar to the proof of Lemma 1 of~\cite{kong2017spectrum}, we only need a special treatment of the arcs or free cycles that contains $\gamma_{k+1},\delta_1$ and $\gamma'_{k+1},\delta'_1$. If $\delta_{k+1},\gamma_1$ is the $i$th edge of an arc with length $l$, the corresponding summation becomes $tr(\beta^T (SS^T)^{2(l-i)-1}\beta\beta^T(SS^T)^{2i+1}\beta)^{1/2} \le \beta^TSS^T\beta \sigma_1^{(l-1)}$. If $\delta_{k+1},\gamma_1$ belongs to a free cycle of length $l$, the corresponding trace term $tr(T^{l})$ becomes $tr(\beta^T (SS^T)^{l}\beta) \le \beta^TSS^T\beta \sigma_1^{(l-1)}$. Since for any $t$ we have $tr((S^TS)^t)\le d\sigma_1^t$. By Lemma~\ref{lem:partitionsize1} and Corollary~\ref{cor:partitionsize1} and the assumption that $\beta^TSS^T\beta\le 1$, the upperbound holds.
\end{proof}

Finally we are ready to conclude the proof of this case.
\begin{lemma}\label{lem:app:gen-case1}
\begin{align*}
\frac{1}{\binom{n}{k+1}^2}\sum_{\pi,\pi'} \Big(\E[\beta^T\x_{\pi_1}\x_{\pi_1}^T\x_{\pi_2}\ldots \x_{\pi_k}^T\x_{\pi_{k+1}}\x_{\pi_{k+1}}^T\beta
\beta^T\x_{\pi'_1}\x_{\pi'_1}^T\x_{\pi'_2}\ldots \x_{\pi'_k}^T\x_{\pi'_{k+1}}\x^T_{\pi'_{k+1}}\beta]-(\beta^T\Sigma^{k+1}\beta)^2\Big)\\
\le 2^{12(k+1)}(k+1)^{6(k+1)}C^{k+1}\sigma_1^{2k}\max(\frac{d^{k}}{n^{k+1}},\frac{1}{n}).
\end{align*}
\end{lemma}
\begin{proof}
$\frac{1}{\binom{n}{k+1}^2}\sum_{\pi,\pi',\delta,\delta',\gamma,\gamma'|P,Q}{(S^T\beta\beta^T S)}_{\gamma_{k+1},\delta_1}{(S^T\beta\beta^T S)}_{\gamma'_{k+1},\delta'_1}
\prod_{j=1}^k {(S^T S)}_{\gamma_j,\delta_{j+1}}  {(S^T S)}_{\gamma'_j,\delta'_{j+1}}\\ \le 2^{2(k+1)}\frac{\binom{n}{m}}{\binom{n}{k+1}^2}\sigma_1^{2k}d^{2(k+1)-m-1}$, where we have applied Lemma 5 of~\cite{kong2017spectrum}. This quantity is monotically decreasing for $m$ when $n<d$. Since $m\ge k+1$, we conclude that $2^{2k}\sigma_1^{2k}\frac{d^{k}}{n^{k+1}}$ is an upperbound. In the setting where $n\ge d$, $2^{2k}(\beta^TSS^T\beta)^2\sigma_1^{2k}\frac{1}{n}$ is an upperbound. Now what remains is to bound the expectation of the product of $z$'s and count the number of distinct $(P,Q)$. By the $4$th moment condition, $\E[\prod_{i=1}^{k+1} z_{\pi_i,\delta_i}z_{\pi_i,\gamma_i}z_{\pi'_i,\delta'_i}z_{\pi'_i,\gamma'_i}]\le C^{k+1}$. The number of distinct $P,Q$ is bounded by ${2(k+1)}^{2(k+1)}{4(k+1)}^{4(k+1)}$. Hence we conclude the proof. 
\end{proof}

Next, we classify the products that involve $\eta$.
\begin{enumerate}
\item If $\pi_1,\pi_{k+1},\pi'_1,\pi'_{k+1}$ take $4$ different values. All the terms involving $\eta$ have expectation $0$
\item If $\pi_1,\pi_{k+1},\pi'_1,\pi'_{k+1}$ take $3$ different values. WLOG assume $\pi_1=\pi'_1$. The terms that does not involves $\eta_{\pi_{k+1}}$ and $\eta_{\pi'_{k+1}}$ may have non-zero expectation. Pick $\pi,\pi'$, the contribution to the variance (omitting the $\delta^2$ term) is expressed as:
\begin{align*}
\E[\x_{\pi_1}^T\x_{\pi_2}\ldots \x_{\pi_k}^T\x_{\pi_{k+1}}\x_{\pi_{k+1}}^T\beta \beta^T\x_{\pi'_{k+1}}\x_{\pi'_{k+1}}^T\ldots \x_{\pi'_2}^T\x_{\pi'_{2}}\x^T_{\pi_1}]
\end{align*}
For the convenience of the analysis, we redefine $\pi = \{\pi_1,\pi_2,\ldots,\pi_{k+1},\pi'_{k+1},\pi'_{k},\ldots,\pi'_2\}$. With the new definition, the above formula can be expressed as
$$
\E[\sum_{\delta,\gamma}\prod_{i=1,i\ne k+1}^{2k+1} z_{\pi_i,\delta_i}z_{\pi_i,\gamma_i}(S^TS)_{\gamma_i,\delta_{i+1}} (S^T\beta \beta^T S)_{\gamma_{k+1},\delta_{k+2}}].
$$
Again we can define PQ-graph based on a realization of the variables $\pi,\delta,\gamma$. We have the following Lemma regarding the maximum number of free cycles and arcs of the Q-Graph.
\begin{lemma}\label{lem:partitionsize2}
For $\pi,\delta,\gamma$ respects $P,Q$. The number of arcs $\xi$ and the number of free cycles $\eta$ satisfies: if $\xi=0$, $\eta\le 2(k+1)-m$, otherwise $\xi/2+\eta\le 2k+1-m$.
\end{lemma}
\begin{proof}
Similarly to the proof of Lemma~\ref{lem:partitionsize1}, we prove the lemma by a counting argument. The following operation will be done in the PQ-graph. First, for each P-node with degree $2$, we remove the P-node, it's neighboring Q-node, called $Q_i$ and connect the Q-nodes neighboring $Q_i$. Notice that if we ever encounter a case where $Q_i$ is incident to a self-loop, that means the Q-Graph has only $1$ free cycle and nothing else. In this case, the original PQ-graph has no degree $4$ P-node and hence $m=2k+1$ which satisfies the lemma statement. Otherwise we are left with a PQ-graph whose P-nodes are all degree $4$. For each P-node who is connected to two Q-nodes, $Q_1$,$Q_2$, if one of the Q-nodes, say $Q_1$, is incident to a self-loop, we remove the P-node and $Q_1, Q_2$ and connect the two neighbors of $Q_2$. Every time this procedure is done, the number of free cycles and the number of P-nodes each decreases by $1$. Notice that if $Q_2$ is also incident to a self-loop, that means the PQ-graph consists of two free cycles and nothing else. If this is the case, the original PQ-graph has $\eta-1$ P-nodes with degree $4$ which means $m=2k+2-\xi$ and satisfies the lemma statement. Otherwise we are left with arcs only, whose total number satisfies $\xi$ equals $2$ times the remains number of degree $4$ P nodes. Hence we have that $\xi/2+\eta$ is equal to the total number of degree $4$ P nodes which is equal to $2k+1-m$ which also satisfies the lemma statement.
\end{proof}
\begin{fact} \label{fact:sumQ2}
Remove all the arcs and free cycles that contains edge $(Q(\gamma_{k+1}),Q(\delta_{k+2}))$ from the Q-graph and denote the number of removed edges as $l$. Suppose we are left with $\xi$ arcs with length $l_1,\ldots,l_\xi$ and $\eta$ free cycles with length $p_1,\ldots,p_\eta$. We have  $\sum_{\delta,\gamma|Q}{(S^T\beta\beta^T S)}_{\gamma_{k+1},\delta_{k+2}} \prod_{j=1,j\ne k+1}^{2k+1} {(S^T S)}_{\gamma_j,\delta_{j+1}} \le \prod_{i=1}^\xi tr((S^T S)^{2l_i})^{1/2}\prod_{j=1}^\eta tr((S^T S)^{p_j}) (\beta^TSS^T\beta)\sigma_1^{l-1}\le \sigma_1^{2k} d^{2(k+1)-m-1}$, where $\sigma_1$ denotes the largest eigenvalue of $S^TS$ or equivalently of $SS^T$.
\end{fact}
The proof is analogous to that of Fact~\ref{fact:sumQ}. 
Now we are ready to conclude the proof of this case.
\begin{lemma}\label{lem:app:gen-case2}
\begin{align*}
\frac{1}{\binom{n}{k+1}^2}\sum_{\pi,\pi'} \Big(\E[\x_{\pi_1}^T\x_{\pi_2}\ldots \x_{\pi_k}^T\x_{\pi_{k+1}}\x_{\pi_{k+1}}^T\beta\beta^T\x_{\pi'_{k+1}}\x_{\pi'_{k+1}}^T\x_{\pi'_{k}}\ldots \x_{\pi'_2}^T\x_{\pi_1}\eta^2_{\pi_1}]\Big)\\
\le 2^{12(k+1)}(k+1)^{6(k+1)}C^{k+1}\sigma_1^{2k}\max(\frac{d^{k-1}}{n^{k}},\frac{1}{n})
\end{align*}
\end{lemma}
\begin{proof}
$\frac{1}{\binom{n}{k+1}}\sum_{\pi,\pi',\delta,\delta',\gamma,\gamma'|P,Q}{(S^T\beta\beta^T S)}_{\gamma_{k+1},\delta_{k+2}} \prod_{j=1,j\ne k+1}^{2k+1} {(S^T S)}_{\gamma_j,\delta_{j+1}}\le 2^{2(k+1)}\sigma_1^{2k}\frac{d^{2(k+1)-m-1}}{n^{2(k+1)-m}}$, where we have applied inclusion-exclusion principle(see Lemma 5 of~\cite{kong2017spectrum}). This quantity is monotonically decreasing for $m$ when $n<d$. Since $m\ge k+2$, we conclude that $2^{2k}\sigma_1^{2k}\frac{d^{k-1}}{n^{k}}$ is an upperbound. In the setting where $n\ge d$,  $2^{2k}\sigma_1^{2k}\frac{1}{n}$ is an upperbound. Now what remains is to bound the expectation of the product of $z$s and count the number of distinct $(P,Q)$. By the $4$th moment condition, $\E[\prod_{i=1}^{k+2} z_{\pi_i,\delta_i}z_{\pi_i,\gamma_i}z_{\pi'_i,\delta'_i}z_{\pi'_i,\gamma'_i}]\le C^{k+1}$. The number of distinct $P,Q$ is bounded by ${2(k+1)}^{2(k+1)}{4(k+1)}^{4(k+1)}$. Hence we conclude the proof. 
\end{proof}

\item If $\pi_1,\pi_{k+1},\pi'_1,\pi'_{k+1}$ takes $2$ different values. All the terms may have non-zero expectation.
Pick $\pi,\pi'$, the contribution to the variance (omitting the $\delta^4$ term) is expressed as:
\begin{align*}
\E[\x_{\pi_1}^T\x_{\pi_2}\ldots \x_{\pi_k}^T\x_{\pi_{k+1}}\x_{\pi_{k+1}}^T\ldots \x_{\pi'_2}^T\x_{\pi'_{2}}\x^T_{\pi_1}]
\end{align*}
For the convenience of the analysis, we redefine $\pi = \{\pi_1,\pi_2,\ldots,\pi_{k+1},\pi'_{k},\ldots,\pi'_2\}$. With the new definition, the above formula can be expressed as
$$
\E[\sum_{\delta,\gamma}\prod_{i=1}^{2k} z_{\pi_i,\delta_i}z_{\pi_i,\gamma_i}(S^TS)_{\gamma_i,\delta_{i+1}} ].
$$
Again we can define the PQ-graph based on a realization of the variables $\pi,\delta,\gamma$. We have the following lemma regarding the maximum number of free cycles and arcs of the Q-Graph.
\begin{lemma}\label{lemma:lem:partitionsize3}
For $\pi,\delta,\gamma$ respecting $P,Q$, the number of arcs $\xi$ and the number of free cycles $\eta$ satisfies: if $\xi=0$, $\eta\le 2k+1-m$, otherwise $\xi/2+\eta\le 2k-m$.
\end{lemma}
The proof is analogous to the proof of Lemma~\ref{lem:partitionsize2}.  
The following fact is the analog of  Facts~\ref{fact:sumQ} and~\ref{fact:sumQ2}.
\begin{fact}\label{fact:sumQ3}
Given $P=\{P_1,\ldots,P_m\}$, $Q = \{Q_1,\ldots,Q_w\}$, we have $\sum_{\delta,\gamma|Q}\prod_{j=1}^{2k} {(S^T S)}_{\gamma_j,\delta_{j+1}} \le \sigma_1^{2k}d^{2k-m+1}$.
\end{fact}
Finally we are ready to conclude the proof of this case.
\begin{lemma}\label{lem:bound-partitionsize3}
\begin{align*}
\frac{1}{\binom{n}{k+1}^2}\sum_{\pi,\pi'} \Big(\E[\x_{\pi_1}^T\x_{\pi_2}\ldots \x_{\pi_k}^T\x_{\pi_{k+1}}\x_{\pi_{k+1}}^T\x_{\pi'_k}\ldots \x_{\pi'_2}^T\x_{\pi_1}\eta^2_{\pi_1}\eta^2_{\pi_{k+1}}]\Big)\\
\le 2^{12(k+1)}(k+1)^{6(k+1)}C^{k+1}\sigma_1^{2k}\max(\frac{d^{k}}{n^{k+1}},\frac{1}{n})
\end{align*}
\end{lemma}
\begin{proof}
$\frac{1}{\binom{n}{k+1}}\sum_{\pi,\delta,\gamma|P,Q} \prod_{j=1}^{2k} {(S^T S)}_{\gamma_j,\delta_{j+1}}\le 2^{2(k+1)}\frac{\binom{n}{m}}{\binom{n}{k+1}^2}\sigma_1^{2k}\frac{d^{2k-m+1}}{n^{2(k+1)-m}}$, where we have applied the inclusion-exclusion principle (see Lemma 5 of~\cite{kong2017spectrum}). This quantity is monotonically decreasing for $m$ when $n<d$. Since $m\ge k+1$, we conclude that $2^{2k}\sigma_1^{2k}\frac{d^{k}}{n^{k+1}}$ is an upperbound. In the setting where $n\ge d$, the upperbound becomes $2^{2k}\sigma_1^{2k}\frac{1}{n}$. Now what remains is to bound the expectation of the product of the $z$'s and count the number of distinct $(P,Q)$. By the $4$th moment condition, $\E[\prod_{i=1}^{k+2} z_{\pi_i,\delta_i}z_{\pi_i,\gamma_i}z_{\pi'_i,\delta'_i}z_{\pi'_i,\gamma'_i}]\le C^{k+1}$. The number of distinct $P,Q$ is bounded by ${2(k+1)}^{2(k+1)}{4(k+1)}^{4(k+1)}$. Hence we conclude the proof. 
\end{proof}
Combing these three cases concludes the proof of Proposition~\ref{prop:gen-var}.
\end{enumerate}

\section{Proof of Theorem~\ref{thm:bin-gen}, the Linear Classification Model}\label{ap:bin-gen}
In this section, we prove the main theorem in the linear classification model. We restate the theorem:

\medskip
\vspace{.2cm}\noindent \textbf{Theorem~\ref{thm:bin-gen}.} \emph{ Suppose we are given $n<d$ labeled examples, $(\x_1,y_1),\ldots,(\x_n,y_n)$, with $\x_i$ drawn independently from a Gaussian distribution with mean $0$ and covariance $\Sigma$ where $\Sigma$ is an unknown arbitrary $d$ by $d$ real matrix.  Assuming that each label $y_i$ takes value $1$ with probability $g(\beta^T\x_i)$ and $-1$ with probability $1-g(\beta^T\x_i)$, where $g(x)=\frac{1}{1+e^{-x}}$ is the sigmoid function. There is an algorithm that takes $n$ labeled samples, parameter $k$, ${\sigma_{max}}$ and ${\sigma_{min}}$ which satisfies $\sigma_{max}I\succeq S^TS\succeq {\sigma_{min}I}$, and with probability $1-\tau$, outputs an estimate $\widehat{err_{opt}}$ with additive error $|\widehat{err_{opt}}-err_{opt}|\le c\Big(\sqrt{\min(\frac{2}{k^2},2e^{-(k-1)\sqrt{\frac{\sigma_{min}}{\sigma_{max}}}})\sigma_{max}\|\beta\|^2+\frac{f(k)}{\tau}\sum_{i=2}^k \frac{d^{i/2-1/2}}{n^{i/2}}}\Big),$ where $err_{opt}$ is the classification error of the best linear classifier, $f(k)=k^{O(k)}$ and $c$ is an absolute constant.}
\medskip

We list the ingredients necessary for the proof of Theorem~\ref{thm:bin-gen} here. Recall that the success of our algorithm relies on a series of ``moment'' estimators. Proposition~\ref{prop:bin-unbiased2} and Proposition~\ref{prop:bin-varbound} establish accuracy guarantees for these estimators. Given these estimated ``high order moments'', our algorithm computes a linear combination of them to approximate the quantity $\E_{x\sim N(0,1)}[(g(\|\beta^T\Sigma^{1/2}\|x)-\frac{1}{2})x]$, and such an approximation is accurate via the polynomial approximation bound proved in Proposition~\ref{prop:poly-approx}. Finally our algorithm applies the mapping $F$ to the estimate of $\E_{x\sim N(0,1)}[(g(\|\beta^T\Sigma^{1/2}\|x)-\frac{1}{2})x]$ to obtain the classification error of the best linear classifier. The Lipschitz property of $F$, established in Proposition~\ref{prop:bin-mapping}, determines how the accuracy of estimating $\E_{x\sim N(0,1)}[(g(\|\beta^T\Sigma^{1/2}\|x)-\frac{1}{2})x]$ transfers to the accuracy of the final output (i.e. the classification error of the best linear classifier).
\begin{proof}[Proof of Theorem~\ref{thm:bin-gen}]
We first divide each sample $\x_i$ by  $\sigma_{max}$ and run Algorithm~\ref{alg:bin} with parameter $k$ and the polynomial constructed from Proposition~\ref{prop:poly-approx}. We denote the model parameter vector and covariance matrix after scaling as $\beta'$ and $\Sigma'$ which are simply $\sqrt{{\sigma_{max}}}\beta$ and $\frac{\Sigma}{{\sigma_{max}}}$(i.e. $\frac{S^TS}{{\sigma_{max}}}$). Notice that $\beta^T\Sigma\beta=\beta'^T\Sigma\beta'$, and this step does not change the classification error. 

By Proposition~\ref{prop:bin-unbiased2}, we have a series of unbiased estimators of $4\E_{x\sim N(0,1)}[(g(\|\beta^T\Sigma^{1/2}\|x)-\frac{1}{2})x]^2\frac{{\beta'}^T{\Sigma'}^{k}{\beta'}}{\beta'\Sigma'\beta'}$ for all $k\ge 2$. Further, by Chebyshev's inequality, Proposition~\ref{prop:bin-varbound} and the fact that $\E_{x\sim N(0,1)}[(g(\|\beta^T\Sigma^{1/2}\|x)-\frac{1}{2})x]<1/2, \E_{x\sim N(0,1)}[(g(\|\beta^T\Sigma^{1/2}\|x)-\frac{1}{2})x^3]<1$, we have that with probability $1-\tau$, there is an estimate for each $4\E_{x\sim N(0,1)}[(g(\|\beta^T\Sigma^{1/2}\|x)-\frac{1}{2})x]^2\frac{{\beta'}^T{\Sigma'}^{k}{\beta'}}{\beta'\Sigma'\beta'}$ with additive error less than $\frac{f(k)}{\tau}\max(\frac{d^{k/2-1/2}}{n^{k/2}},\frac{1}{\sqrt{n}})$. Note that we use $f(k)$ to denote different functions that only depends on $k$. Hence 
\begin{align}
|\sum_{i=0}^{k-1} a_i\frac{y^TG^{i+1}y}{\binom{n}{i+2}} - \sum_{i=2}^{k+1} 4a_i\E_{x\sim N(0,1)}[(g(\|\beta^T\Sigma^{1/2}\|x)-\frac{1}{2})x]^2\frac{{\beta'}^T{\Sigma'}^{i}{\beta'}}{\beta'\Sigma'\beta'}| ] \le \frac{f(k)}{\tau}\sum_{i=2}^k \frac{d^{i/2-1/2}}{n^{i/2}}.\label{eqn:approx-mom}
\end{align}

Observe that by using the polynomial coefficients from Proposition~\ref{prop:poly-approx}, we have $|{\beta'}^T\Sigma'\beta' - \sum_{i=0}^{k-2} a_i{\beta'}^T {\Sigma'}^{i+2}\beta|\le \min(\frac{2}{k^2},2e^{-(k-1)\sqrt{\frac{{\sigma_{min}}}{{\sigma_{max}}}}})){\sigma_{max}}\|\beta\|^2$. Hence 
\begin{align}
4\E_{x\sim N(0,1)}[(g(\|\beta^T\Sigma^{1/2}\|x)-\frac{1}{2})x]^2\Big(\sum_{i=0}^{k-2}a_i\frac{{\beta'}^T{\Sigma'}^{k}{\beta'}}{\beta'\Sigma'\beta'} - \frac{{\beta'}^T{\Sigma'}{\beta'}}{{\beta'}^T\Sigma'\beta'}\Big)\nonumber\\
\le \frac{4\E_{x\sim N(0,1)}[(g(\|\beta^T\Sigma^{1/2}\|x)-\frac{1}{2})x]^2}{{\beta'}^T\Sigma\beta'} \min(\frac{2}{k^2},2e^{-(k-1)\sqrt{\frac{{\sigma_{min}}}{{\sigma_{max}}}}})){\sigma_{max}}\|\beta\|^2\\
\le \min(\frac{2}{k^2},2e^{-(k-1)\sqrt{\frac{{\sigma_{min}}}{{\sigma_{max}}}}})){\sigma_{max}}\|\beta\|^2\label{eqn:approx-target}
\end{align}
by Proposition~\ref{prop:bin-qbyb}. Thus we combine Equation~\ref{eqn:approx-mom}, ~\ref{eqn:approx-target} and get
\begin{align*}
|\frac{\sum_{i=0}^{k-1} a_i\frac{y^TG^{i+1}y}{\binom{n}{i+2}}}{4} - \E_{x\sim N(0,1)}[(g(\|\beta^T\Sigma^{1/2}\|x)-\frac{1}{2})x]^2| \\
\le \frac{1}{4}\Big(\min(\frac{2}{k^2},2e^{-(k-1)\sqrt{\frac{{\sigma_{min}}}{{\sigma_{max}}}}})){\sigma_{max}}\|\beta\|^2+\frac{f(k)}{\tau}\sum_{i=2}^k \frac{d^{i/2-1/2}}{n^{i/2}}\Big).
\end{align*}
To simplify notation, similarly to the definition in the proof of Proposition~\ref{prop:bin-mapping}, we define $q_1 = \frac{\sqrt{\sum_{i=0}^{k-1} a_i\frac{y^TG^{i+1}y}{\binom{n}{i+2}}}}{2}$, $q_2 = \E_{x\sim N(0,1)}[(g(\|\beta^T\Sigma^{1/2}\|x)-\frac{1}{2})x]$ and accordingly $p_1 = F_g(q_1)$, $p_2 = F_g(q_2) = (\frac{1}{2}-\E_{x\sim N(0,1)}[|g(\|\beta^T\Sigma^{1/2}\|x)-\frac{1}{2}|])$ and $l = \frac{1}{4}\Big(\min(\frac{2}{k^2},2e^{-(k-1)\sqrt{\frac{{\sigma_{min}}}{{\sigma_{max}}}}})){\sigma_{max}}\|\beta\|^2+\frac{f(k)}{\tau}\sum_{i=2}^k \frac{d^{i/2-1/2}}{n^{i/2}}\Big)$. Under this notation, we can apply Proposition~\ref{prop:bin-mapping} and get $|p_1^2-p_2^2|\le |q_1-q_2|$. Together with inequality above that $|q_1^2-q_2^2|\le l$, the following inequality holds: $|p_1-p_2|\le \min(\frac{l}{(p_1+p_2)(q_1+q_2)},\frac{q_1+q_2}{p_1+p_2},p_1+p_2)$. Notice that if $(p_1+p_2)<1/10$, it is easy to verify that $q_1+q_2>1/3$ (which can be seen from Figure~\ref{fig:sig}), hence $|p_1-p_2|\le \min(\frac{3}{(p_1+p_2)},p_1+p_2) \le O(\sqrt{l})$. If $(q_1+q_2)<1/10$, it's easy to very that $p_1+p_2>1/3$ as well, hence $|p_1-p_2|\le \min(\frac{3l}{(q_1+q_2)},3(q_1+q_2))\le O(\sqrt{l})$. Finally we conclude that the output of Algorithm~\ref{alg:bin} satisfies
$$
|F(t)-(\frac{1}{2}-\E_{x\sim N(0,1)}[|g(\|\beta^T\Sigma^{1/2}\|x)-\frac{1}{2}|])|\le O\Big(\sqrt{\min(\frac{1}{k^2},e^{-(k-1)\sqrt{\frac{{\sigma_{min}}}{{\sigma_{max}}}}})){\sigma_{max}}\|\beta\|^2+\frac{f(k)}{\tau}\sum_{i=2}^k \frac{d^{i/2-1/2}}{n^{i/2}}}\Big).
$$
\end{proof}

\begin{proposition}\label{prop:bin-unbiased2}
$\E[\frac{\y^T G^k\y}{\binom{n}{k+1}}] = 4\E[(g(\beta^T\x)-\frac{1}{2})\frac{\beta^T\x}{\sqrt{\beta^T\Sigma\beta}}]^2\frac{\beta^T\Sigma^{k+1}\beta}{\beta^T\Sigma\beta}$.
\end{proposition}
The proof is analogous to that of Proposition~\ref{prop:gen-unbiased}.
\begin{proposition}\label{prop:bin-varbound}
$\Var[\frac{\y^T G^k\y}{\binom{n}{k+1}}] \le \Big(\E[(g(\beta^T\x)-\frac{1}{2})\frac{(\beta^T\x)^3}{\sqrt{\beta^T\Sigma\beta}^3}]^2\E[(g(\beta^T\x)-\frac{1}{2})\frac{(\beta^T\x)}{\sqrt{\beta^T\Sigma\beta}}]^2+\E[(g(\beta^T\x)-\frac{1}{2})\frac{(\beta^T\x)}{\sqrt{\beta^T\Sigma\beta}}]^4+1\Big)f(k)\sigma_1^{2k}\max(\frac{d^{k}}{n^{k+1}},\frac{1}{n})$.
\end{proposition}

This is the main technical core of the proof, and its involved proof is similar to that of the analogous variance bound in the linear regression setting (Proposition~\ref{prop:gen-var}).  We devote Section~\ref{sec:bin_varbound} to this proof.

The following proposition establishes the Lipschitz property of the mapping $F_g$ used in Algorithm~\ref{alg:bin}, namely $F_g(x)-F_g(y) = O(\sqrt{y-x}).$ 

\begin{proposition}\label{prop:bin-mapping} In the case that $g$ is the sigmoid function $g(x) = \frac{1}{1+e^{-x}}$, for any real numbers $b\ge b'\ge 0$,  
$$
\frac{1}{16}(\E_{x\sim N(0,1)}[|g(bx)-\frac{1}{2}|]^2-\E_{x\sim N(0,1)}[|g(b'x)-\frac{1}{2}|]^2)\le \E_{x\sim N(0,1)}[(g(bx)-\frac{1}{2})x]-\E_{x\sim N(0,1)}[(g(b'x)-\frac{1}{2})x].
$$
\end{proposition}
\begin{proof}
Let $p(b) = \frac{1}{2}-\E_{x\sim N(0,1)}[|g(bx)-\frac{1}{2}|]$ and $q(b) = \E_{x\sim N(0,1)}[(g(bx)-\frac{1}{2})x]$. We will prove that $-\frac{\partial q}{\partial p}\ge \frac{1}{16}p$. First assume the condition holds. Let $p(b') = p_1,p(b) = p_2$ and notice that $p_1>p_2$. The right hand side of the statement ${\E_{x\sim N(0,1)}[(g(bx)-\frac{1}{2})x]-\E_{x\sim N(0,1)}[(g(b'x)-\frac{1}{2})x]} = {q(p_2) - q(p_1)}$, can be bounded as ${q(p_2) - q(p_1)}= {\int_{p_2}^{p_1} -\frac{\partial q}{\partial p} dp}\ge \frac{1}{16}{\int_{p_1}^{p_2}pdp} = \frac{1}{16}({p_1^2-p_2^2})\ge \frac{1}{16}(\E_{x\sim N(0,1)}[|g(bx)-\frac{1}{2}|]^2-\E_{x\sim N(0,1)}[|g(b'x)-\frac{1}{2}|]^2).$

To prove $\frac{\partial q}{\partial p}\le -p$, we write down the formula  $\frac{\partial q}{\partial p} = \frac{\partial q}{\partial b}/\frac{\partial p}{\partial b}$, where by definition, 

\begin{align*}-\frac{\partial p}{\partial b} = \frac{\partial}{\partial b}\E_{x\sim N(0,1)}[|g(bx)-\frac{1}{2}|]=\int_{0}^{\infty}\frac{\sqrt{2}x e^{-b x-\frac{x^2}{2}}}{\sqrt{\pi } \left(e^{-b x}+1\right)^2}dx\\
\le \sqrt{\frac{2}{\pi}}\int_{0}^{\infty} x e^{-b x-\frac{x^2}{2}}dx = \sqrt{\frac{2}{\pi }}-b e^{\frac{b^2}{2}} \text{erfc}\left(\frac{b}{\sqrt{2}}\right),
\end{align*}

\begin{align*}
\frac{\partial q}{\partial b} = 2\int_{0}^{\infty}\frac{x^2 e^{-b x-\frac{x^2}{2}}}{\sqrt{2 \pi } \left(e^{-b x}+1\right)^2}dx \ge \frac{1}{2\sqrt{2 \pi }}\int_{0}^{\infty}x^2 e^{-b x-\frac{x^2}{2}}dx = \frac{1}{4} \left(b^2+1\right) e^{\frac{b^2}{2}}
   \text{erfc}\left(\frac{b}{\sqrt{2}}\right)-\frac{b}{2 \sqrt{2 \pi }},
\end{align*}
\begin{align*}
p = \frac{1}{2}-\E_{x\sim N(0,1)}[|g(bx)-\frac{1}{2}|] = \frac{1}{2}-2\int_{0}^{\infty}\frac{e^{-\frac{x^2}{2}} }{\sqrt{2 \pi }}\left(\frac{1}{e^{-b x}+1}-\frac{1}{2}\right) = 2\int_{0}^{\infty}\frac{e^{-\frac{x^2}{2}} }{\sqrt{2 \pi }}\frac{1}{e^{b x}+1}\\
\le 2\int_{0}^{\infty}\frac{e^{-bx-\frac{x^2}{2}} }{\sqrt{2 \pi }} = e^{\frac{b^2}{2}} \text{erfc}\left(\frac{b}{\sqrt{2}}\right).
\end{align*}
 
By the lower and upper bound on the complementary error function, $ \frac{2 e^{-x^2}}{\sqrt{\pi } \left(\sqrt{x^2+2}+x\right)}<\text{erfc}(x)\le\frac{2 e^{-x^2}}{\sqrt{\pi } \left(\sqrt{x^2+\frac{4}{\pi}}+x\right)}$, 

\begin{align*}
\frac{\partial q}{\partial b}\ge \frac{b^2-\sqrt{b^2+4} b+2}{4 \sqrt{2 \pi } \left(\sqrt{b^2+4}+b\right)},\\
-\frac{\partial p}{\partial b} \le \frac{\sqrt{\frac{2}{\pi }} \left(\sqrt{b^2+4}-b\right)}{\sqrt{b^2+4}+b},\\
p\le \frac{2 \sqrt{2}}{\sqrt{\pi  b^2+8}+\sqrt{\pi } b}.
\end{align*}
The above three bounds together imply that 
$-\frac{\partial q}{\partial p}/p\ge \frac{\left(b^2-\sqrt{b^2+4} b+2\right) \left(\sqrt{\pi  b^2+8}+\sqrt{\pi } b\right)}{16
   \sqrt{2} \left(\sqrt{b^2+4}-b\right)}.$ Because $(\sqrt{\pi  b^2+8}+\sqrt{\pi } b)\ge \sqrt{2}(b+2)$,  $$
   \frac{\left(b^2-\sqrt{b^2+4} b+2\right) \left(\sqrt{\pi  b^2+8}+\sqrt{\pi } b\right)}{16 \sqrt{2}
    \left(\sqrt{b^2+4}-b\right)}\ge \frac{\left(b^2-\sqrt{b^2+4} b+2\right) \left(b+2\right)}{16
   \left(\sqrt{b^2+4}-b\right)},
   $$
   and we will show that 
   $$\frac{\left(b^2-\sqrt{b^2+4} b+2\right) \left(b+2\right)}{16
    \left(\sqrt{b^2+4}-b\right)}\ge \frac{1}{16},$$ and thus complete the proof.
The following sequence of equivalent inequalities implies the above inequality:
\begin{align*}
\Leftrightarrow \frac{\left(b^2-\sqrt{b^2+4} b+2\right) \left(b+2\right)}{\left(\sqrt{b^2+4}-b\right)}\ge 1
\Leftrightarrow \Big(\frac{2}{\left(\sqrt{b^2+4}-b\right)}-b\Big)\left(b+2\right)\ge 1\\
\Leftrightarrow \frac{2}{\left(\sqrt{b^2+4}-b\right)}\ge \frac{(b+1)^2}{b+2} 
\Leftrightarrow \left(\sqrt{b^2+4}-b\right)\le 2\frac{b+2}{(b+1)^2}\\
\Leftrightarrow b^2+4 \le (2\frac{b+2}{(b+1)^2}+b)^2
\Leftrightarrow 1 \le \frac{(b+2)^2}{(b+1)^4}+\frac{b(b+2)}{(b+1)^2}\\
\Leftrightarrow 3+2b \ge 0.
\end{align*}
\end{proof}
\begin{proposition}\label{prop:bin-qbyb}  In the case that $g$ is the sigmoid function $g(x) = \frac{1}{1+e^{-x}}$,
$\E_{x\sim N(0,1)}[(g(bx)-\frac{1}{2})x]\le \frac{1}{4}b$.
\end{proposition}
\begin{proof}
As in the proof of Proposition~\ref{prop:bin-mapping}, we define $q(b) = \E_{x\sim N(0,1)}[(g(bx)-\frac{1}{2})x]$ for convenience. The derivative of $q$ satisfies
$$
\frac{\partial q}{\partial b} = 2\int_{0}^{\infty}\frac{x^2 e^{-b x-\frac{x^2}{2}}}{\sqrt{2 \pi } \left(e^{-b x}+1\right)^2}dx \le \frac{\sqrt{2}}{\sqrt{\pi }}\int_{0}^{\infty}x^2 e^{-b x-\frac{x^2}{2}}dx = \frac{1}{2} \left(b^2+1\right) e^{\frac{b^2}{2}}\text{erfc}\left(\frac{b}{\sqrt{2}}\right)-\frac{b}{\sqrt{2 \pi }}.
$$
Applying the upper bound of the Complementary Error Function, $\text{erfc}(x) \le \frac{2 e^{-x^2}}{\sqrt{\pi } \left(\sqrt{x^2+\frac{4}{\pi}}+x\right)}$, we get
$$
\frac{\partial q}{\partial b}\le \frac{b^2-\sqrt{b^2+\frac{8}{\pi }} b+2}{2 \sqrt{2} \left(\sqrt{\pi  b^2+8}+\sqrt{\pi }b\right)} \le \frac{2}{8+2\sqrt{2}{\pi}b} \le \frac{1}{4}.
$$
Since $q(0)=0$, the derivative bound implies $q\le \frac{1}{4}b.$
\end{proof}

\subsection{Proof of the Variance Bound, Proposition~\ref{prop:bin-varbound}}\label{sec:bin_varbound}

Here we bound the variance of our estimates of the ``higher moments'', which is the main technical core of Theorem~\ref{thm:bin-gen}.  We restate the key proposition:

\medskip

\noindent \textbf{Proposition~\ref{prop:bin-varbound}.} 
$\Var[\frac{\y^T G^k\y}{\binom{n}{k+1}}] \le \Big(\E[(g(\beta^T\x)-\frac{1}{2})\frac{(\beta^T\x)^3}{\sqrt{\beta^T\Sigma\beta}^3}]^2\E[(g(\beta^T\x)-\frac{1}{2})\frac{(\beta^T\x)}{\sqrt{\beta^T\Sigma\beta}}]^2+\E[(g(\beta^T\x)-\frac{1}{2})\frac{(\beta^T\x)}{\sqrt{\beta^T\Sigma\beta}}]^4+1\Big)f(k)\sigma_1^{2k}\max(\frac{d^{k}}{n^{k+1}},\frac{1}{n})$.
\medskip

\begin{proof}
As in the proof of Proposition~\ref{prop:gen-var}, the term $\Var[\frac{\y^TG^k\y}{\binom{n}{k+1}}]$ can be expressed as:
\begin{align*}
\frac{1}{\binom{n}{k+1}^2}\sum_{\pi_1<\pi_2<\ldots<\pi_{k+1},\pi'_1<\pi'_2<\ldots<\pi'_{k+1}} \Big(\E[y_{\pi_1}\x_{\pi_1}^T\x_{\pi_2}\ldots \x_{\pi_k}^T\x_{\pi_{k+1}}y_{\pi_{k+1}}y_{\pi'_1}\x_{\pi'_1}^T\x_{\pi'_2}\ldots \x_{\pi'_k}^T\x_{\pi'_{k+1}}y_{\pi'_{k+1}}]-\\
\E[y_{\pi_1}\x_{\pi_1}^T\x_{\pi_2}\ldots \x_{\pi_k}^T\x_{\pi_{k+1}}y_{\pi_{k+1}}]\E[y_{\pi'_1}\x_{\pi'_1}^T\x_{\pi'_2}\ldots \x_{\pi'_k}^T\x_{\pi'_{k+1}}y_{\pi'_{k+1}}]
\Big).
\end{align*}
To carry out the computation, we classify the terms in the summation into several cases according to the realization of $\pi,\pi'$.

\begin{enumerate}
\item The simplest case is $\{\pi_1,\pi_{k+1}\}\cap \pi'=\emptyset$ and $\{\pi'_1,\pi'_{k+1}\}\cap \pi' = \emptyset$. In other words, $\pi_1,\pi_{k+1}$ do not take the same value as any index in $\pi'$ and $\pi_1',\pi_{k+1}'$ do not take the same value as any index in $\pi$. In such a case, the expression can be rewritten as:
$$
16\E[(g(\beta^T\x)-\frac{1}{2})\frac{\beta^T\x}{\|\beta^T\Sigma\beta\|}]^4  \E[\beta^T\Sigma\x_{\pi_2}\ldots \x_{\pi_k}^T\Sigma\beta \beta^T\Sigma\x_{\pi'_2}\ldots \x_{\pi'_k}^T\Sigma\beta].
$$

The case now basically reduces to Lemma~\ref{lem:app:gen-case1} by replacing $\beta$ with $\Sigma\beta$. However notice that since we can no longer assume $\beta^T \Sigma \beta\le 1$, the statement of Fact~\ref{fact:sumQ} needs to be replaced by $\sigma_1^{2k}d^{2(k+1)-m-1}(\beta^T\Sigma\beta)^2$. For all $\pi,\pi'$ which belongs to case $1$ we have
\begin{align*}
\frac{1}{\binom{n}{k+1}^2}\sum_{\pi,\pi'} \E[\beta^T\Sigma\x_{\pi_2}\ldots \x_{\pi_k}^T\Sigma\beta \beta^T\Sigma\x_{\pi'_2}\ldots \x_{\pi'_k}^T\Sigma\beta] \\
\le 2^{12(k+1)}(k+1)^{6(k+1)}C^{k-1}\sigma_1^{2k}(\beta^T\Sigma\beta)^2\max(\frac{d^{k-2}}{n^{k-1}},\frac{1}{n}).
\end{align*}
Notice that we replaced $k+1$ by $k-1$ in the appropriate places because there are only $2(k-1)$ indices in the product. Hence the contribution of this case is bounded by 
$$
f(k) \E[(g(\beta^T\x)-\frac{1}{2})\frac{\beta^T\x}{\sqrt{\beta^T\Sigma\beta}}]^4\sigma_1^{2k}\max(\frac{d^{k-2}}{n^{k-1}},\frac{1}{n}),
$$
where $f=2^{12(k+1)+4}(k+1)^{6(k+1)}C^{k-1}$.
\item Consider the case where $\pi_1,\pi_{k+1}, \pi'_1,\pi'_{k+1}$ takes $4$ different values and only one of them takes the same value as the other indices, meaning $$|\{\pi_1,\pi_{k+1}, \pi'_1,\pi'_{k+1}\}\cap\{\pi_2,\ldots,\pi_{k+1},\pi'_2,\ldots,\pi'_{k+1}\}|=1.$$ WLOG assume $\pi_1 = \pi'_t$, the expectation can be expressed as
\begin{align}\label{eqn:case2}
8\E[(g(\beta^T\x)-\frac{1}{2})\frac{\beta^T\x}{\beta^T\Sigma\beta}]^3 \frac{1}{\binom{n}{k+1}^2}\sum_{\pi,\pi'} \E[y_{\pi_1}\x_{\pi_1}^T\x_{\pi_2}\ldots \x_{\pi_k}^T\Sigma\beta \beta^T\Sigma\x_{\pi'_2}\ldots \x_{\pi_1}\x_{\pi_1}^T\ldots\x_{\pi'_k}^T\Sigma\beta].
\end{align}
Pick one realization of $\pi$ and $\pi'$, and write $\x_{\pi_1} = \frac{\beta^T\x_{\pi_1}}{\beta^T\Sigma\beta}\Sigma\beta + \x_{\pi_1}-\frac{\beta^T\x_{\pi_1}}{\beta^T\Sigma\beta}\Sigma\beta$. We get
\begin{align}
&\E[y_{\pi_1}\x_{\pi_1}^T\x_{\pi_2}\ldots \x_{\pi_k}^T\Sigma\beta \beta^T\Sigma\x_{\pi'_2}\ldots \x_{\pi_1}\x_{\pi_1}^T\ldots\x_{\pi'_k}^T\Sigma\beta]\label{eqn:case2-2}\\
&= 2\E[(g(\beta^T\x)-\frac{1}{2})\frac{(\beta^T\x)^3}{(\beta^T\Sigma\beta)^3}] \E[\beta^T\Sigma\x_{\pi_2}\ldots\x^T_{\pi_k}\Sigma\beta\beta^T\Sigma\x_{\pi_2'}\ldots\x_{\pi'_{i-1}}^T\Sigma\beta\beta^T\Sigma\x_{\pi_{i+1}'}\ldots\x_{\pi_k'}^T\Sigma\beta]\nonumber\\
&+ 2\E[(g(\beta^T\x)-\frac{1}{2})\frac{(\beta^T\x)}{(\beta^T\Sigma\beta)}] \E[\beta^T\Sigma\x_{\pi_2},\ldots,\x^T_{\pi_k}\Sigma\beta\beta^T\Sigma\x_{\pi_2'}\ldots\x_{\pi'_{i-1}}^T(\Sigma -\frac{\Sigma\beta\beta^T\Sigma}{\beta^T\Sigma\beta})\x_{\pi_{i+1}'}\ldots\x_{\pi_k'}^T\Sigma\beta]\nonumber\\
&+2\E[(g(\beta^T\x)-\frac{1}{2})\frac{(\beta^T\x)}{(\beta^T\Sigma\beta)}] \E[\beta^T\Sigma\x_{\pi_{i-1}'}\ldots\x_{\pi'_{2}}^T\Sigma\beta\beta^T\Sigma\x_{\pi_k}\ldots\x^T_{\pi_2}(\Sigma -\frac{\Sigma\beta\beta^T\Sigma}{\beta^T\Sigma\beta})\x_{\pi_{i+1}'}\ldots\x_{\pi_k'}^T\Sigma\beta]\nonumber\\
&+ 2\E[(g(\beta^T\x)-\frac{1}{2})\frac{(\beta^T\x)}{(\beta^T\Sigma\beta)}] \E[\beta^T\Sigma\x_{\pi_2'}\ldots\x_{\pi'_{i-1}}^T(\Sigma -\frac{\Sigma\beta\beta^T\Sigma}{\beta^T\Sigma\beta})\x_{\pi_2}\ldots\x^T_{\pi_k}\Sigma\beta\beta^T\Sigma \x_{\pi_{i+1}'}\ldots\x_{\pi_k'}^T\Sigma\beta].\nonumber
\end{align}
Simplifing the equation yields:
\begin{align*}
= 2\Big(\E[(g(\beta^T\x)-\frac{1}{2})\frac{(\beta^T\x)^3}{(\beta^T\Sigma\beta)^3}]-3\E[(g(\beta^T\x)-\frac{1}{2})\frac{(\beta^T\x)}{(\beta^T\Sigma\beta)^2}]\Big)\\
\E[\beta^T\Sigma\x_{\pi_2},\ldots,\x^T_{\pi_k}\Sigma\beta\beta^T\Sigma\x_{\pi_2'}\ldots\x_{\pi'_{i-1}}^T\Sigma\beta\beta^T\Sigma\x_{\pi_{i+1}'}\ldots\x_{\pi_k'}^T\Sigma\beta]\\
+ 2\E[(g(\beta^T\x)-\frac{1}{2})\frac{(\beta^T\x)}{(\beta^T\Sigma\beta)}] \E[\beta^T\Sigma\x_{\pi_2},\ldots,\x^T_{\pi_k}\Sigma\beta\beta^T\Sigma\x_{\pi_2'}\ldots\x_{\pi'_{i-1}}^T\Sigma\x_{\pi_{i+1}'}\ldots\x_{\pi_k'}^T\Sigma\beta]\\
+2\E[(g(\beta^T\x)-\frac{1}{2})\frac{(\beta^T\x)}{(\beta^T\Sigma\beta)}] \E[\beta^T\Sigma\x_{\pi_2'}\ldots\x_{\pi'_{i-1}}^T\Sigma\beta  \beta^T\Sigma\x_{\pi_k},\ldots,\x^T_{\pi_2}\Sigma\x_{\pi_{i+1}'}\ldots\x_{\pi_k'}^T\Sigma\beta]\\
+ 2\E[(g(\beta^T\x)-\frac{1}{2})\frac{(\beta^T\x)}{(\beta^T\Sigma\beta)}] \E[\beta^T\Sigma\x_{\pi_2'}\ldots\x_{\pi'_{i-1}}^T\Sigma\x_{\pi_2}\x^T_{\pi_2},\ldots,\x^T_{\pi_k}\Sigma\beta\beta^T\Sigma \x_{\pi_{i+1}'}\ldots\x_{\pi_k'}^T\Sigma\beta].
\end{align*}
We analyze one expectation as an example and the rest will follow similarly. Notice that we can apply Lemma~\ref{lem:partitionsize2} to $\E[\beta^T\Sigma\x_{\pi_2},\ldots,\x^T_{\pi_k}\Sigma\beta\beta^T\Sigma\x_{\pi_2'}\ldots\x_{\pi'_{i-1}}^T\Sigma\beta\beta^T\Sigma\x_{\pi_{i+1}'}\ldots\x_{\pi_k'}^T\Sigma\beta]$ and get that if $\xi=0$, $\eta\le2k-2-m$, otherwise $\xi/2+\eta\le 2k-3-m$, simply because there are $2k-3$ variables in the product. Similarly to Fact~\ref{fact:sumQ2}, it's not hard to show that for each (P,Q) the sum of the non-random part can be bounded by $(\beta^T\Sigma\beta)^3\sigma_1^{2k}d^{2k-3-m}$. Similarly, $\E[\beta^T\Sigma\x_{\pi_2'}\ldots\x_{\pi'_{i-1}}^T\Sigma\x_{\pi_2}\x^T_{\pi_2},\ldots,\x^T_{\pi_k}\Sigma\beta\beta^T\Sigma \x_{\pi_{i+1}'}\ldots\x_{\pi_k'}^T\Sigma\beta]$ is bounded by $(\beta^T\Sigma\beta)^2\sigma_1^{2k}d^{2k-3-m}.$ Taking the summation over all $\pi,\pi'$ and picking the worst case $m$ yields:
\begin{align*}
\frac{1}{\binom{n}{k+1}^2}\sum_{\pi,\pi'}\E[\beta^T\Sigma\x_{\pi_2},\ldots,\x^T_{\pi_k}\Sigma\beta\beta^T\Sigma\x_{\pi_2'}\ldots\x_{\pi'_{i-1}}^T\Sigma\beta\beta^T\Sigma\x_{\pi_{i+1}'}\ldots\x_{\pi_k'}^T\Sigma\beta]\\
\le f(k)(\beta^T\Sigma\beta)^3\sigma_1^{2k}\max(\frac{d^{k-2}}{n^{k-1}},\frac{1}{n}),
\end{align*}
and 
\begin{align*}
\frac{1}{\binom{n}{k+1}^2}\sum_{\pi,\pi'}\E[\beta^T\Sigma\x_{\pi_2'}\ldots\x_{\pi'_{i-1}}^T\Sigma\x_{\pi_2}\x^T_{\pi_2},\ldots,\x^T_{\pi_k}\Sigma\beta\beta^T\Sigma \x_{\pi_{i+1}'}\ldots\x_{\pi_k'}^T\Sigma\beta]\\
\le f(k)(\beta^T\Sigma\beta)^2\sigma_1^{2k}\max(\frac{d^{k-2}}{n^{k-1}},\frac{1}{n}),
\end{align*}

where $f(k) = k^{O(k)}$. Plugging in the bound above to Equation~\ref{eqn:case2} bound the contribution of this case by
\begin{align*}
f(k)\E[(g(\beta^T\x)-\frac{1}{2})\frac{\beta^T\x}{\sqrt{\beta^T\Sigma\beta}}]^3  \Big(\E[(g(\beta^T\x)-\frac{1}{2})\frac{(\beta^T\x)^3}{\sqrt{\beta^T\Sigma\beta}^3}]+\E[(g(\beta^T\x)-\frac{1}{2})\frac{(\beta^T\x)}{\sqrt{\beta^T\Sigma\beta}}]\Big)\\
\sigma_1^{2k}\max(\frac{d^{k-2}}{n^{k-1}},\frac{1}{n}),
\end{align*}
where $f(k)=k^{O(k)}$.

\item Consider the case where $\pi_1,\pi_{k+1}, \pi'_1,\pi'_{k+1}$ takes $4$ different values and two of them take the same value as the other indices, meaning $|\{\pi_1,\pi_{k+1}, \pi'_1,\pi'_{k+1}\}\cap\{\pi_2,\ldots,\pi_{k+1},\pi'_2,\ldots,\pi'_{k+1}\}|=2$. Assume $\pi_1 = \pi'_i$,$\pi_{k+1} = \pi'_{i'}$ the expectation can be expressed as
\begin{align}
4\E[(g(\beta^T\x)-\frac{1}{2})\frac{\beta^T\x}{\|\beta^T\Sigma\beta\|}]^2\frac{1}{\binom{n}{k+1}^2}\nonumber\\
\sum_{\pi,\pi'} \E[y_{\pi_1}\x_{\pi_1}^T\x_{\pi_2}\ldots \x_{\pi_{k+1}}y_{\pi_{k+1}} \beta^T\Sigma\x_{\pi'_2}\ldots \x_{\pi_1}\x_{\pi_1}^T\ldots\x_{\pi_{k+1}}\x_{\pi_{k+1}}^T\ldots\x_{\pi'_k}^T\Sigma\beta].\label{eqn:case3}
\end{align}
Similar to Equation~\ref{eqn:case2-2}, we pick one realization of $\pi,\pi'$ from the summation and expend both $\x_{\pi_1} = \frac{\beta^T\x_{\pi_1}}{\beta^T\Sigma\beta}\Sigma\beta + (\x_{\pi_1}-\frac{\beta^T\x_{\pi_1}}{\beta^T\Sigma\beta}\Sigma\beta)$ and $\x_{\pi_{k+1}} = \frac{\beta^T\x_{\pi_{k+1}}}{\beta^T\Sigma\beta}\Sigma\beta + (\x_{\pi_{k+1}}-\frac{\beta^T\x_{\pi_{k+1}}}{\beta^T\Sigma\beta}\Sigma\beta)$. We list all the binomial expansion as below. The terms with $0$ expectation have been omitted.
\begin{enumerate} 
\item Term $(\frac{\beta^T\x_{\pi_{1}}}{\beta^T\Sigma\beta}\Sigma\beta)^3(\frac{\beta^T\x_{\pi_{1}}}{\beta^T\Sigma\beta}\Sigma\beta)^3$. There is one term in this category,
\begin{align*}
4\E[(g(\beta^T\x)-\frac{1}{2})\frac{(\beta^T\x)^3}{(\beta^T\Sigma\beta)^3}]^2\\
\E[\beta^T\Sigma\x_{\pi_2}\ldots\x^T_{\pi_k}\Sigma\beta\beta^T\Sigma\x_{\pi_2'}\ldots \x_{\pi'_{i-1}}^T\Sigma\beta\beta^T\Sigma\x_{\pi'_{i+1}}\ldots\x_{\pi'_{i'-1}}^T\Sigma\beta\beta^T\Sigma\x_{\pi_{i'+1}'}\ldots\x_{\pi_k'}^T\Sigma\beta].
\end{align*}
Applying Lemma~\ref{lem:partitionsize2} yields that if $\xi=0$, $\eta\le 2k-3-m$, otherwise $\xi/2+\eta\le 2k-4-m$, simply because there are $2k-4$ variables in the product. Similarly to Fact~\ref{fact:sumQ2}, for each (P,Q) the sum of the non-random part can be bounded by $(\beta^T\Sigma\beta)^4\sigma_1^{2k}d^{2k-4-m}$. Taking the summation over all $\pi,\pi'$ and picking the worst case $m$ yields that the contribution of this case is bounded by:
\begin{align*}
f(k)\E[(g(\beta^T\x)-\frac{1}{2})\frac{(\beta^T\x)^3}{(\beta^T\Sigma\beta)^3}]^2(\beta^T\Sigma\beta)^4\sigma_1^{2k}\max(\frac{d^{k-2}}{n^{k}},\frac{1}{n^2}),
\end{align*}
where $f(k)=k^{O(k)}.$
\item  Term $(\frac{\beta^T\x_{\pi_{1}}}{\beta^T\Sigma\beta}\Sigma\beta)^3 (\frac{\beta^T\x_{\pi_{k+1}}}{\beta^T\Sigma\beta}\Sigma\beta)(\x_{\pi_{k+1}}-\frac{\beta^T\x_{\pi_{k+1}}}{\beta^T\Sigma\beta}\Sigma\beta)^2$. There are $3$ terms that fall in this category, we list one below and the rest has the same upper bound.
\begin{align*}
4\E[(g(\beta^T\x)-\frac{1}{2})\frac{(\beta^T\x)^3}{(\beta^T\Sigma\beta)^3}]\E[(g(\beta^T\x)-\frac{1}{2})\frac{(\beta^T\x)}{(\beta^T\Sigma\beta)}]\\
\E[\beta^T\Sigma\x_{\pi_2}\ldots\x^T_{\pi_k}\Sigma\beta\beta^T\Sigma\x_{\pi_2'}\ldots \x_{\pi'_{i-1}}^T\Sigma\beta\beta^T\Sigma\x_{\pi'_{i+1}}\ldots\x_{\pi'_{i'-1}}^T(\Sigma -\frac{\Sigma\beta\beta^T\Sigma}{\beta^T\Sigma\beta})\x_{\pi_{i'+1}'}\ldots\x_{\pi_k'}^T\Sigma\beta]
\end{align*}
The contribution of this case is bounded by:
\begin{align*}
f(k)\E[(g(\beta^T\x)-\frac{1}{2})\frac{(\beta^T\x)^3}{(\beta^T\Sigma\beta)^3}]\E[(g(\beta^T\x)-\frac{1}{2})\frac{(\beta^T\x)}{(\beta^T\Sigma\beta)}](\beta^T\Sigma\beta)^3\sigma_1^{2k}\max(\frac{d^{k-2}}{n^{k}},\frac{1}{n^2}),
\end{align*}
where $f(k)=k^{O(k)}.$
\item Term $(\frac{\beta^T\x_{\pi_{1}}}{\beta^T\Sigma\beta}\Sigma\beta)(\x_{\pi_{1}}-\frac{\beta^T\x_{\pi_{1}}}{\beta^T\Sigma\beta}\Sigma\beta)^2(\frac{\beta^T\x_{\pi_{k+1}}}{\beta^T\Sigma\beta}\Sigma\beta)^3$. This case is identical to case (b).
\item Term $(\frac{\beta^T\x_{\pi_{1}}}{\beta^T\Sigma\beta}\Sigma\beta)(\x_{\pi_{1}}-\frac{\beta^T\x_{\pi_{1}}}{\beta^T\Sigma\beta}\Sigma\beta)^2(\frac{\beta^T\x_{\pi_{k+1}}}{\beta^T\Sigma\beta}\Sigma\beta)(\x_{\pi_{k+1}}-\frac{\beta^T\x_{\pi_{k+1}}}{\beta^T\Sigma\beta}\Sigma\beta)^2$. There are 9 cases in this category, and one example is listed below.

\begin{align*}
4\E[(g(\beta^T\x)-\frac{1}{2})\frac{(\beta^T\x)}{(\beta^T\Sigma\beta)}]\E[(g(\beta^T\x)-\frac{1}{2})\frac{(\beta^T\x)}{(\beta^T\Sigma\beta)}]\\
\E[\beta^T\Sigma\x_{\pi_2}\ldots\x^T_{\pi_k}\Sigma\beta\beta^T\Sigma\x_{\pi_2'}\ldots \x_{\pi'_{i-1}}^T(\Sigma -\frac{\Sigma\beta\beta^T\Sigma}{\beta^T\Sigma\beta})\x_{\pi'_{i+1}}\ldots\x_{\pi'_{i'-1}}^T(\Sigma -\frac{\Sigma\beta\beta^T\Sigma}{\beta^T\Sigma\beta})\x_{\pi_{i'+1}'}\ldots\x_{\pi_k'}^T\Sigma\beta].
\end{align*}
The contribution of this case is bounded by:
\begin{align*}
f(k)\E[(g(\beta^T\x)-\frac{1}{2})\frac{(\beta^T\x)}{(\beta^T\Sigma\beta)}]\E[(g(\beta^T\x)-\frac{1}{2})\frac{(\beta^T\x)}{(\beta^T\Sigma\beta)}](\beta^T\Sigma\beta)^2\sigma_1^{2k}\max(\frac{d^{k-2}}{n^{k}},\frac{1}{n^2}),
\end{align*}
where $f(k)=k^{O(k)}$.
\end{enumerate}
To summarize, the contribution of this case is bounded by
\begin{align*}
\Big(\E[(g(\beta^T\x)-\frac{1}{2})\frac{(\beta^T\x)^3}{\sqrt{\beta^T\Sigma\beta}^3}]^2+\E[(g(\beta^T\x)-\frac{1}{2})\frac{(\beta^T\x)^3}{\sqrt{\beta^T\Sigma\beta}^3}]\E[(g(\beta^T\x)-\frac{1}{2})\frac{(\beta^T\x)}{\sqrt{\beta^T\Sigma\beta}}]\\
+\E[(g(\beta^T\x)-\frac{1}{2})\frac{(\beta^T\x)}{\sqrt{\beta^T\Sigma\beta}}]^2\Big)\\
\E[(g(\beta^T\x)-\frac{1}{2})\frac{\beta^T\x}{\sqrt{\beta^T\Sigma\beta}}]^2f(k)\sigma_1^{2k}\max(\frac{d^{k-2}}{n^{k}},\frac{1}{n^2})\\
\le \Big(\E[(g(\beta^T\x)-\frac{1}{2})\frac{(\beta^T\x)^3}{\sqrt{\beta^T\Sigma\beta}^3}]+\E[(g(\beta^T\x)-\frac{1}{2})\frac{(\beta^T\x)}{\sqrt{\beta^T\Sigma\beta}}]\Big)^2\E[(g(\beta^T\x)-\frac{1}{2})\frac{\beta^T\x}{\sqrt{\beta^T\Sigma\beta}}]^2\\
f(k)\sigma_1^{2k}\max(\frac{d^{k-2}}{n^{k}},\frac{1}{n^2}).
\end{align*}

\item Consider the case where $\pi_1,\pi_{k+1}, \pi'_1,\pi'_{k+1}$ takes $3$ different values and further $\{\pi_1,\pi_{k+1}\}\cap \{\pi'_2,\ldots,\pi'_{k}\}=\emptyset$ and $\{\pi'_1,\pi'_{k+1}\}\cap \{\pi_2,\ldots,\pi_{k}\}=\emptyset$. WLOG assume $\pi_1=\pi'_1$, we have that 
\begin{align*}
\E[y_{\pi_1}\x_{\pi_1}^T\x_{\pi_2}\ldots \x_{\pi_k}^T\x_{\pi_{k+1}}y_{\pi_{k+1}}y_{\pi'_1}\x_{\pi'_1}^T\x_{\pi'_2}\ldots \x_{\pi'_k}^T\x_{\pi'_{k+1}}y_{\pi'_{k+1}}]\\
= 4\E[(g(\beta^T\x)-\frac{1}{2})\frac{\beta^T\x}{\beta^T\Sigma\beta}]^2 \E[\x_{\pi_1}^T\x_{\pi_2}\ldots \x_{\pi_k}^T \Sigma\beta\beta^T\Sigma \x_{\pi'_k} \ldots \x_{\pi'_2}^T\x_{\pi_1}],
\end{align*}
which can be handled by applying Lemma~\ref{lem:app:gen-case2}, which yields that the sum of all the terms in this case satisfies
\begin{align*}
4\E[(g(\beta^T\x)-\frac{1}{2})\frac{\beta^T\x}{\beta^T\Sigma\beta}]^2 \sum_{\pi,\pi'} \E[\x_{\pi_1}^T\x_{\pi_2}\ldots \x_{\pi_k}^T \Sigma\beta \x_{\pi_1}^T\x_{\pi'_2}\ldots \x_{\pi'_k}^T\Sigma\beta]\\
\le f(k)\E[(g(\beta^T\x)-\frac{1}{2})\frac{\beta^T\x}{\sqrt{\beta^T\Sigma\beta}}]^2\sigma_1^{2k}\max(\frac{d^{k-2}}{n^{k-1}},\frac{1}{n}),
\end{align*}
where $f(k) = 2^{12(k+1)}(k+1)^{6(k+1)}C^{k}.$
\item Consider the case where $\pi_1,\pi_{k+1}, \pi'_1,\pi'_{k+1}$ takes $3$ different values and either $\{\pi_1,\pi_{k+1}\}\cap |\{\pi'_2,\ldots,\pi'_{k}\}|=1$ or $|\{\pi'_1,\pi'_{k+1}\}\cap \{\pi_2,\ldots,\pi_{k}\}|=1$. WLOG assume $\pi_{k+1}=\pi'_{k+1}$, $\pi_{1}=\pi'_i$, we have that
\begin{align*}
\E[y_{\pi_1}\x_{\pi_1}^T\x_{\pi_2}\ldots \x_{\pi_k}^T\x_{\pi_{k+1}}y_{\pi_{k+1}}y_{\pi'_1}\x_{\pi'_1}^T\x_{\pi'_2}\ldots \x_{\pi'_k}^T\x_{\pi'_{k+1}}y_{\pi'_{k+1}}]\\
= 2\E[(g(\beta^T\x)-\frac{1}{2})\frac{\beta^T\x}{\beta^T\Sigma\beta}] \E[y_{\pi_1}\x_{\pi_1}^T\x_{\pi_2}\ldots \x_{\pi_k}^T \x_{\pi_{k+1}}\x_{\pi_{k+1}}^T \x_{\pi'_k} \ldots \x_{\pi_{1}}\x_{\pi_{1}}^T\ldots \x_{\pi'_2}^T\Sigma\beta].
\end{align*}
As in Equation~\ref{eqn:case2-2}, we expand $\x_{\pi_1} = \frac{\beta^T\x_{\pi_1}}{\beta^T\Sigma\beta}\Sigma\beta + (\x_{\pi_1}-\frac{\beta^T\x_{\pi_1}}{\beta^T\Sigma\beta}\Sigma\beta)$ of $$\E[y_{\pi_1}\x_{\pi_1}^T\x_{\pi_2}\ldots \x_{\pi_k}^T \x_{\pi_{k+1}}\x_{\pi_{k+1}}^T \x_{\pi'_k} \ldots \x_{\pi_{1}}\x_{\pi_{1}}^T\ldots \x_{\pi'_2}^T\Sigma\beta].$$ The expectation is bounded by the summation of the following two cases:
\begin{itemize}
\item $2\E[(g(\beta^T\x)-\frac{1}{2})\frac{(\beta^T\x)^3}{(\beta^T\Sigma\beta)^3}]      \E[\beta^T\Sigma\x_{\pi_2}\ldots\x^T_{\pi_k}\x_{\pi_{k+1}}\x^T_{\pi_{k+1}}\x_{\pi_k'}\ldots\x_{\pi'_{i-1}}^T\Sigma\beta\beta^T\Sigma\x_{\pi_{i+1}'}\ldots\x_{\pi_2'}^T\Sigma\beta].$
\item $6\E[(g(\beta^T\x)-\frac{1}{2})\frac{(\beta^T\x)}{(\beta^T\Sigma\beta)}]      \E[\beta^T\Sigma\x_{\pi_2}\ldots\x^T_{\pi_k}\x_{\pi_{k+1}}\x^T_{\pi_{k+1}}\x_{\pi_k'}\ldots\x_{\pi'_{i-1}}^T(\Sigma -\frac{\Sigma\beta\beta^T\Sigma}{\beta^T\Sigma\beta})\x_{\pi_{i+1}'}\ldots\x_{\pi_2'}^T\Sigma\beta].$
\end{itemize}
To bound the first case, we apply Lemma~\ref{lem:partitionsize2} which yields that if $\xi=0$, $\eta\le 2k-1-m$, otherwise $\xi/2+\eta\le 2k-2-m$, simply because there are $2k-2$ variables in the product. Similarly to Fact~\ref{fact:sumQ2}, for each (P,Q) the sum of the non-random part can be bounded by $(\beta^T\Sigma\beta)^2\sigma_1^{2k}d^{2k-2-m}$. Taking the summation over all $\pi,\pi'$ and picking the worst case $m$ yields that the contribution of this case is bounded by:
\begin{align*}
f(k)\E[(g(\beta^T\x)-\frac{1}{2})\frac{(\beta^T\x)^3}{(\beta^T\Sigma\beta)^3}](\beta^T\Sigma\beta)^2\sigma_1^{2k}\max(\frac{d^{k-2}}{n^{k}},\frac{1}{n^2}),
\end{align*}
where $f(k)= k^{O(k)}$. The second case can be similarly bounded by 

\begin{align*}
f(k)\E[(g(\beta^T\x)-\frac{1}{2})\frac{(\beta^T\x)}{(\beta^T\Sigma\beta)}](\beta^T\Sigma\beta)\sigma_1^{2k}\max(\frac{d^{k-2}}{n^{k}},\frac{1}{n^2}).
\end{align*}
To summarize, the contribution of this case is bounded by 
\begin{align*}
f(k)\Big(\E[(g(\beta^T\x)-\frac{1}{2})\frac{(\beta^T\x)}{\sqrt{\beta^T\Sigma\beta}}]^2+\E[(g(\beta^T\x)-\frac{1}{2})\frac{(\beta^T\x)^3}{\sqrt{\beta^T\Sigma\beta}^3}]\E[(g(\beta^T\x)-\frac{1}{2})\frac{\beta^T\x}{\sqrt{\beta^T\Sigma\beta}}]\Big)\\
\sigma_1^{2k}\max(\frac{d^{k-2}}{n^{k}},\frac{1}{n^2}).
\end{align*}
\item Consider the case where $\pi_1,\pi_{k+1}, \pi'_1,\pi'_{k+1}$ takes $2$ different values, which is equivalent to $\pi_1=\pi'_1, \pi_{k+1}=\pi'_{k+1}$. We have that, 
\begin{align*}
\E[y_{\pi_1}\x_{\pi_1}^T\x_{\pi_2}\ldots \x_{\pi_k}^T\x_{\pi_{k+1}}y_{\pi_{k+1}}y_{\pi'_1}\x_{\pi'_1}^T\x_{\pi'_2}\ldots \x_{\pi'_k}^T\x_{\pi'_{k+1}}y_{\pi'_{k+1}}]\\
= \E[\x_{\pi_1}^T\x_{\pi_2}\ldots \x_{\pi_k}^T\x_{\pi_{k+1}}\x_{\pi_{k+1}}^T\x_{\pi'_k}\ldots   \x_{\pi'_2}^T\x_{\pi_1} ].
\end{align*}
Lemma~\ref{lem:bound-partitionsize3} can be applied to bound the contribution of this case as
$$
\le f(k)\sigma_1^{2k}\max(\frac{d^{k}}{n^{k+1}},\frac{1}{n}),
$$
where $f(k)=2^{12(k+1)}(k+1)^{6(k+1)}C^{k+1}$.
\end{enumerate}
Adding the bounds obtained from analyzing all $6$ cases, we get the following covariance upperbound
\begin{align*}
\Big(\E[(g(\beta^T\x)-\frac{1}{2})\frac{(\beta^T\x)^3}{\sqrt{\beta^T\Sigma\beta}^3}]^2\E[(g(\beta^T\x)-\frac{1}{2})\frac{(\beta^T\x)}{\sqrt{\beta^T\Sigma\beta}}]^2+\E[(g(\beta^T\x)-\frac{1}{2})\frac{(\beta^T\x)}{\sqrt{\beta^T\Sigma\beta}}]^4+1\Big)\\
f(k)\sigma_1^{2k}\max(\frac{d^{k}}{n^{k+1}},\frac{1}{n}),
\end{align*}
where $f(k) = k^{O(k)}.$
\end{proof}

\section{Lowerbounds on Estimating the Variance of the Noise}
\subsection{Identity Covariance Lowerbound}

\vspace{.2cm}\noindent \textbf{Proposition~\ref{prop:lrlbid}.}\emph{
In the setting of Proposition~\ref{prop:1main}, there is a constant $c$ such that no algorithm can distinguish the case that the signal is pure noise (i.e. $\|\beta\|=0$ and $\delta = 1$) versus almost no noise (i.e. $\delta = 0.01$ and $\beta$ is chosen to be a random vector s.t. $\|\beta\|=\sqrt{0.99}$, using fewer than $c \sqrt{d}$ datapoints with probability of success greater than $2/3$.
}

We show our lowerbound by upperbounding the total variational distance between the following two cases:
\begin{enumerate}
\item Draw $n$ independent samples $(\x_1,y_1),\ldots,(\x_n,y_n)$ where $\x_i\sim N(0,I), y_i\sim N(0,1)$.
\item First pick a uniformly random unit vector $v$, then draw $n$ independent samples $(\x_1,y_1),\ldots,(\x_n,y_n)$ where $\x_i\sim N(0,I), y_i = b v^T\x_i+\eta_i,$ where  $\eta_i\sim N(0, 1-b^2)$.
\end{enumerate}
The claim then is that no algorithm can distinguish the two cases with probability more than $2/3$.  Let $Q_n$ denote the joint distribution of $(\x_1,y_1),\ldots,(\x_n,y_n)$ in case $2$. Our goal is to bound the total variance $D_{TV}(Q_n,N(0,I)^{\otimes n})$ which is smaller than $\frac{\sqrt{\chi^2(Q_n,N(0,I)^{\otimes n})}}{2}$ by the properties of chi-square divergence.  In case 2, for a fixed $v$, the conditional distribution $\x|y\sim N(ybv,I-b^2vv^T)$. Let $P_{y,v}$ denote such a conditional distribution. The chi-square divergence can be expressed as:
\begin{align*}
1+ \chi^2(Q_n,N(0,I)^{\otimes n}) = 
\int_{\x_1,y_1}\ldots \int_{\x_n,y_n}\frac{\Big(\int_{v\in \mathcal{S}^d} \prod_{i=1}^n  P_{y_i,v}(\x_i)G(y_i)dv\Big)^2}{\prod_{i=1}^n G(\x_i)G(y_i)} d\x_1dy_1\ldots d\x_ndy_n\\
 = \int_{\x_1,y_1}\ldots \int_{\x_n,y_n}\int_{v\in \mathcal{S}^d} \int_{v'\in \mathcal{S}^d} \prod_{i=1}^n \frac{   P_{y_i,v}(\x_i)  P_{y_i,v'}(\x_i)G(y_i)}{G(\x_i)} dv dv' d\x_1dy_1\ldots d\x_ndy_n\\
 = \int_{v\in \mathcal{S}^d} \int_{v'\in \mathcal{S}^d} \Big(\int_y \int_{\x}\frac{   P_{y,v}(\x)  P_{y,v'}(\x)G(y)}{G(\x)} d\x dy\Big)^n dv dv'\\
 = \int_{v\in \mathcal{S}^d} \int_{v'\in \mathcal{S}^d} \Big(\int_y \Big(\chi^2_{N(0,1)}(P_{y,v}(\x),P_{y,v'})+1\Big)G(y) dy\Big)^n dv dv',
\end{align*}
where $\chi^2_D(D_1,D_2)$ is the pairwise correlation, defined as $\int \frac{D_1(x)D_2(x)}{D(x)}dx-1$ (see Definition 2.9 of~\cite{diakonikolas2016statistical}). The following proposition reduce the high dimensional pairwise correlation to an one dimeional problem.
\begin{proposition}\label{prop:1d}
$\chi^2_{N(0,1)}(P_{y,v}(\x),P_{y,v'}) =  \chi^2_{N(0,1)}(N(by,1-b^2),N(v^Tv'by,1-(v^Tv')^2b^2))$
\end{proposition}
See the first paragraph of the proof of Lemma 3.4 of~\cite{diakonikolas2016statistical} for the proof. Applying Proposition~\ref{prop:1d}, using Fact~\ref{fact:chi2} and denote $v^Tv'$ as $\cos\theta$  we get
\begin{align*}
\int_{v\in \mathcal{S}^d} \int_{v'\in \mathcal{S}^d} \Big(\int_y \frac{1}{\sqrt{1-b^4\cos^2\theta }}\exp(\frac{b^2\cos\theta}{1+b^2\cos\theta }y^2)G(y) dy\Big)^n dv dv'
= \int_{v\in \mathcal{S}^d} \int_{v'\in \mathcal{S}^d}\frac{1}{(1-b^2\cos\theta)^n}dvdv'\\
\le \int_{0}^{\pi/2} (1+\frac{1}{1-b^2}b^2\cos\theta)^n \sin^{d-2}\theta d\theta/B((d-1)/2,1/2) + \int_{\pi/2}^{\pi}\sin^{d-2}\theta d\theta/B((d-1)/2,1/2)\\
=\int_{0}^{\pi/2} (1+\frac{1}{1-b^2}b^2\cos\theta)^n \sin^{d-2}\theta d\theta/B((d-1)/2,1/2) + 1/2\\
=  \sum_{i=0}^n\binom{n}{i}\int_0^{\pi/2}\frac{(\frac{b^2}{1-b^2})^i \sin^{d-2}\cos^i\theta d\theta}{B((d-1)/2,1/2)}+1/2= \sum_{i=0}^n\binom{n}{i}(\frac{b^2}{1-b^2})^i\frac{ B((d-1)/2,(i+1)/2)}{B((d-1)/2,1/2)}+1/2,
\end{align*}
where we have applied Fact~\ref{fact:betafun} and~\ref{fact:chi2} in the above derivation. Let $b_i = \binom{n}{i}(\frac{b^2}{1-b^2})^i\frac{ B((d-1)/2,(i+1)/2)}{B((d-1)/2,1/2)}$ be the $i$th term in the summation. Notice that $b_{i+2}/b_i = \frac{(n-i)(n-i-1)}{(i+1)(i+2)}(\frac{b^2}{1-b^2})^2\frac{i+1}{(d+i)}\le (\frac{b^2}{1-b^2})^2\frac{n^2}{2d}$, $b_0=1$ and $b_1\le \frac{b^2}{1-b^2}\frac{n}{\sqrt{d}}$. Let $n\le \frac{1}{2}\sqrt{d}\frac{1-b^2}{b^2}$, we have $b_{i+1}/b_i\le 1/8, b_1\le 1/2$ and hence $\sum_{i=0}^nb_i\le (1+1/2)\frac{1}{1-1/8}=\frac{12}{7}$. Thus the chi-square divergence $\chi^2(Q_n,N(0,I)^{\otimes n})\le \frac{17}{14}$, by Fact~\ref{fact:chi2vstv} we conclude that $D_{TV}(Q_n,N(0,I)^{\otimes n})\le 0.55$. Thus there is no algorithm that can distinguish the two cases with probability greater than $0.45/2+0.55>0.77.$ By a standard boosting argument, the statement of Proposition~\ref{prop:lrlbid} holds.

\begin{fact}\label{fact:betafun}
\begin{enumerate} 
\item For $x>-1$,$y>-1$ we have that $2\int_0^{\pi/2} \sin^x\theta \cos^y\theta=B((x+1)/2,(y+1)/2)$
\item For all $x,y\in R$,we have that $B(x,y+1)=\frac{y}{x+y}B(x,y).$
\item If we choose $v$ and $v'$ uniformly at random from $\mathcal{S}_d$, the angle $\theta$ between them is distributed with the probability density function $sin^{d-2}(\theta)/B((d - 1)/2, 1/2)$, where $B(x, y)$ is the Beta function.
\end{enumerate}
\end{fact}
\begin{fact}[Pairwise Correlation]\label{fact:chi2}
$$
\chi^2_{N(0,1)}(N(\mu_1, \sigma_1^2),N(\mu_2, \sigma_2^2))=\frac{\exp \left(-\frac{\mu_1^2 \left(\sigma_2^2-1\right)+2 \mu_1 \mu_2+\mu_2^2 \left(\sigma_1^2-1\right)}{2 \sigma_1^2 \left(\sigma_2^2-1\right)-2 \sigma_2^2}\right)}{ \sqrt{\sigma_1^2+\sigma_2^2-\sigma_1^2\sigma_2^2}}-1
$$
\end{fact}
\begin{fact}\label{fact:chi2vstv}
$P,Q$ are distributions. $D_{TV}(P,Q)\le \frac{\sqrt{\chi^2(P,Q)}}{2}$.
\end{fact}
\begin{proof}
Let $f(x)=\frac{|P(x)-Q(x)|}{\sqrt{P(x)}}, g(x)= \sqrt{P(x)}$. By the Cauchy-Schwarz inequality, we have $D_{TV}(P,Q) = \frac{1}{2}\int_x |P(x)-Q(x)|dx = \frac{1}{2}\int f(x)g(x)dx \le \frac{1}{2}\sqrt{\int_x f^2(x) dx \int_xg^2(x)dx}=\frac{\sqrt{\chi^2(P,Q)}}{2}  $
\end{proof}

\subsection{General Covariance Lowerbound}
In this section, we first prove Theorem~\ref{thm:lb-all-cond} in the setting where $\|\beta\|$ is bounded, showing that our algorithm achieves optimal error. Then we prove Theorem~\ref{thm:genLB} which implies that without a bound on $\|\beta\|$, learnabilty can not be estimated with a sublinear sample size.

\subsubsection{Bounded $\|\beta\|$}
We restate Theorem~\ref{thm:lb-all-cond} as the following two propositions which correspond to different conditioning of the covariance $\Sigma$. 
\begin{proposition}\label{prop:lb-const-cond}
Assume that $I\succeq \Sigma \succeq \frac{1}{2}I$ and $\E[y^2]=\beta^T\Sigma\beta+\delta^2=1$. There exist a function $f$ of $1/\eps$ only such that given $f(1/\eps)d^{1-\frac{1}{\log(1/\eps)}}$ samples no algorithm can estimate $\delta^2$ with error less than $\eps$ with probability better than $3/5$.
\end{proposition}
\begin{proposition}\label{prop:lb-zero-cond}
Assume that $\|\Sigma\|\le 1$ and $\E[y^2]=\beta^T\Sigma\beta+\delta^2=1$, $\|\beta\|\le 1$. There exist a function $f$ of $1/\eps$ only and a constant $c_2$ such that given $f(1/\eps)d^{1-\sqrt{\eps}}$ samples no algorithm can estimate $\delta^2$ with error less than $c_2 \eps$ with probability better than $3/5$.
\end{proposition}
\begin{proof}[Proof of Proposition~\ref{prop:lb-const-cond} and Proposition~\ref{prop:lb-zero-cond}]
The proof of Proposition~\ref{prop:lb-const-cond} and Proposition~\ref{prop:lb-zero-cond} will be a slight adaption of the proof of Theorem 4.5 in~\cite{verzelen2018adaptive}. Let $r =\lceil 1/\eps\rceil$ and define $r$-dimensional positive vectors $\bm{\alpha_0} = (\alpha_{0,1},\ldots,\alpha_{0,r}),$  $\bm{\gamma_0} = (\gamma_{0,1},\ldots,\gamma_{0,r})$ and $\bm{\alpha_1} = (\alpha_{1,1},\ldots,\alpha_{1,r}),$  $\bm{\gamma_1} = (\gamma_{1,1},\ldots,\gamma_{1,r})$, whose exact values will be fixed later. Consider the following two cases:
\begin{enumerate}
\item \begin{enumerate}
\item A set of $r$ $d$-dimensional vectors $\v_1,\ldots,\v_r$ are drawn with probability proportional to $|I_d + \sum_{i=1}^r \alpha_{0,i} \v_i\v_i^T|^{-n/2}e^{-\sum_{i=1}^r d\|\v_i\|_2^2/2}$. 
\item Let $\beta=\sum_{i=1}^r\gamma_{0,i}\v_i$ and $\Sigma = \left(I_d+\sum_{i=1}^r\alpha_{0,i}\v_i\v_i^T\right)^{-1}$.
\item  Draw $n$ independent samples $(\x_1,y_1),\ldots,(\x_n,y_n)$ where $\x_i\sim N(0,\Sigma), y_i = \beta\x_i+\eta_i, \eta_i\sim N(0, \delta_0^2)$.
\end{enumerate}
\item \begin{enumerate}
\item A set of $r$ $d$-dimensional vectors $\v_1,\ldots,\v_r$ are drawn with probability proportional to $|I_d + \sum_{i=1}^r \alpha_{1,i} \v_i\v_i^T|^{-n/2}e^{-\sum_{i=1}^r d\|\v_i\|_2^2/2}$. 
\item Let $\beta=\sum_{i=1}^r\gamma_{1,i}\v_i$ and $\Sigma = \left(I_d+\sum_{i=1}^r\alpha_{1,i}\v_i\v_i^T\right)^{-1}$. 
\item Draw $n$ independent samples $(\x_1,y_1),\ldots,(\x_n,y_n)$ where $\x_i\sim N(0,\Sigma), y_i = \beta\x_i+\eta_i, \eta_i\sim N(0, \delta_1^2)$.
\end{enumerate}

\end{enumerate}
Let $P_0$ be the joint distribution of the $n$ samples $(\x_1,y_1),\ldots,(\x_n.y_n)$ under case 1, $P_1$ be distribution of case 2. The following corollary is a combination of Lemma 7.4 and Lemma 7.5 in~\cite{verzelen2018adaptive}.
\begin{cor}\label{cor:lr-lb-gen}
If $n\le C(r)d$ for a function $C$ of $r$ only, 
\begin{align}
\sum_{i=1}^r \frac{\gamma_{0,i}^2}{1+\alpha_{0,i}}+\delta_0^2 = \sum_{i=1}^r \frac{\gamma_{1,i}^2}{1+\alpha_{1,i}}+\delta_1^2 = 1\label{eqn:lr-lb-gen-2mom}
\end{align}
and for all $k = 2,3,\ldots,2r$, 
\begin{align}
\sum_{i=1}^r \frac{\gamma_{0,i}^2}{(1+\alpha_{0,i})^k} = \sum_{i=1}^r \frac{\gamma_{1,i}^2}{(1+\alpha_{1,i})^k}\label{eqn:lr-lb-gen-allmom}
\end{align} then, 
$$
D_{TV}(P_0,P_1) \le C'(r) (\frac{n^{1+1/2r}}{d})^r,
$$
where $C'(r)$ is a function of $r$ only.
\end{cor}
Corollary~\ref{cor:lr-lb-gen} serves as a powerful tool to prove the indistinguishability, and the only missing part is to construct $\bm{\alpha_0}, \bm{\gamma_0},\delta_0, \bm{\alpha_1}, \bm{\gamma_1},\delta_1$ such that Equation~\ref{eqn:lr-lb-gen-2mom} and Equation~\ref{eqn:lr-lb-gen-allmom} are satisfied, and $|\delta_0^2-\delta_1^2|$ is maximized. Viewing $\bm{\alpha_0}, \bm{\gamma_0},\delta_0, \bm{\alpha_1}$ as two measures $\sum_{i=1}^r\gamma_{0,i}^2\delta_{1/(1+\alpha_{0,i})}$ and $\sum_{i=1}^r\gamma_{1,i}^2\delta_{1/(1+\alpha_{1,i})}$ where $\delta_x$ is the delta function at $x$, our goal is to maximize the first moment discrepancy subject to the condition that the moments indexed $2,\ldots,2r$ match. The following proposition constructs two measures with the desired property, however with the caveat that these are not guaranteed to be the summations of  delta measures. The proof is essentially identical to Lemma 1 in~\cite{cai2011testing}.
\begin{proposition}\label{prop:lb-mom-ext}
For any $a,b\in R$, there exists two positive measure $\mu_0,\mu_1$ supported on $[a,b]$ such that
\begin{align*}
\int x^k\mu_0(dx) = \int x^k\mu_1(dx) \text{ for all }k=2,3,\ldots,2r;\\
\int x\mu_0(dx) - \int x\mu_1(dx) = E_L[x;[a,b]];\\
\int \mu_0(dx)\le 1; \int \mu_1(dx)\le 1,
\end{align*}
where $E_L[x;[a,b]]$ is the distance in the uniform norm on $[a,b]$ from the function $f(x)=x$ to the space spanned by the monomials with degree $2,3,\ldots,2r$.
\end{proposition}
Although $\mu_0,\mu_1$ from Proposition~\ref{prop:lb-mom-ext} are not guaranteed to be a summation of $r$ delta measures, the following proposition which is shown in Theorem 4.3 in~\cite{curto1991recursiveness} guarantees the existence of such a pair of measures.
\begin{proposition}\label{prop:lb-mom-spike}
Let $\m = (m_0,m_1,\ldots, m_{2k})$, $m_0>0$ and let $r = rank(\m)$. The following are equivalent:
\begin{enumerate}
\item There exists a positive Borel measure with supp $\mu\subset  [a,b]$, such that $\int x^j\mu(dx) = m_j$ for $j=0,\ldots,2r$.
\item There exists a positive measure $\mu$ which is a summation of $r$ delta measures with supp $\mu\subset[a,b]$ and $\int x^j\mu(dx) = m_j$ for $j=0,\ldots,2r$.
\end{enumerate}
$r=rank(\m)$ is defined to be the rank of the Hankel matrix of the moment vector $\m$.
\end{proposition}
Combining Proposition~\ref{prop:lb-mom-ext} and Propostion~\ref{prop:lb-mom-spike}, we obtain parameters $\bm{\alpha_0},\bm{\alpha_0},\bm{\gamma_0},\bm{\alpha_1}, \bm{\gamma_1}$
such that Equation~\ref{eqn:lr-lb-gen-allmom} is satisfied, and $|\delta_0^2-\delta_1^2| = E_L[x;[a,b]]$. As $E_L[x;[0,1]]=\Omega(1/r^2)$ by Theorem 2.1 of~\cite{hasson1980comparison}, we set $r = \frac{1}{2\sqrt{\eps}}$ and together with Corollary~\ref{cor:lr-lb-gen} we have that $D_{TV}(P_0,P_1) \le C'(1/\sqrt{\eps})(\frac{n^{1+\sqrt{\eps}}}{d})^{1/2\sqrt{\eps}}$. Hence, by setting $n = f(1/\eps) d^{1-\sqrt{\eps}}$ for certain function $f$, we get $D_{TV}(P_0,P_1)<\frac{1}{5}$ and $|\delta_0^2-\delta_1^2| = \Omega(\eps)$.
this conclude the proof of Proposition~\ref{prop:lb-zero-cond}.
In order to prove Proposition~\ref{prop:lb-const-cond}, we need the following slight variant of Proposition~\ref{prop:lb-mom-ext}.
\begin{proposition}
For any $a,b\in R$, there exists two positive measure $\mu_0,\mu_1$ supported on $[a,b]$ such that
\begin{align*}
\int x^k\mu_0(dx) = \int x^k\mu_1(dx) \text{ for all }k=1,2,\ldots,2r;\\
\int \mu_0(dx) - \int \mu_1(dx) = E_L[1;[a,b]];\\
\int \mu_0(dx)\le 1; \int \mu_1(dx)\le 1,
\end{align*}
where $E_L[1;[a,b]]$ is the distance in the uniform norm on $[a,b]$ from the function $f(x)=1$ to the space spanned by the monomials with degree $1,2,\ldots,2r$.
\end{proposition}
It is shown in Lemma 4 of~\cite{carmon2018analysis} that $E_L[1;[1/2,1]] \ge e^{-2r}$. Hence there is a pair of positive measures $\mu_0, \mu_1$ whose first $2r$ moments match and the total mass differ by at least $e^{-2r}$. Let us define measure $\mu_0'(x) = \mu_0(x)/x, \mu_1'(x) = \mu_1(x)/x$. The measures $\mu_0',\mu_1'$ have the same $2,3,\ldots,2r+1$'s order moments, and their first moment differ by at least $e^{-2r}$. Applying Proposition~\ref{prop:lb-mom-spike} we obtain the parameters $\bm{\alpha_0},\bm{\alpha_0},\bm{\gamma_0},\bm{\alpha_1}, \bm{\gamma_1}$
such that Equation~\ref{eqn:lr-lb-gen-allmom} is satisfied, and $|\delta_0^2-\delta_1^2| \ge e^{-2r}$. Setting $r =  \frac{\log(1/\eps)}{2}$, Corollary~\ref{cor:lr-lb-gen} implies that $D_{TV}(P_0,P_1) \le C'(\log(1/\eps))(\frac{n^{1+\frac{1}{\log(1/\eps)}}}{d})^{\log(1/\eps)/2}$. By setting $n = f(1/\eps) d^{1-\frac{1}{\log(1/\eps)}}$ for some function $f$, we get $D_{TV}(P_0,P_1)<\frac{1}{5}$ and $|\delta_0^2-\delta_1^2|\ge \eps$
\end{proof}



\subsubsection{Unbounded $\|\beta\|$}
We prove Theorem~\ref{thm:genLB} in this subsection, showing that for unbounded $\|\beta\|$, learnability can not be accurately estimated with a sublinear sample size.

\begin{proof}[Proof of Theorem~\ref{thm:genLB}]
The proof of Theorem~\ref{thm:genLB} follows immediately from the standard fact that, for some constant $c$, it is impossible to distinguish $n=cd$ samples drawn from $N(0,I_d)$, versus $n$ samples from a randomly rotated rank $d-1$ Gaussian distribution that has $d-1$ singular values equal to 1, and one singular value equal to zero.   (See, e.g. Proposition 7.1 of~\cite{diakonikolas2016statistical}.) 

To see why that fact implies the claimed lowerbound, note that in the former case, the first coefficient of each sample is ``pure noise'', whereas in the second case, with probability 1 over the random rotation, the first coordinate is a linear function of the remaining $d-1$ coordinates, as the distribution only spans a $d-1$ dimensional subspace of $\R^d$.
\end{proof}
\section{Lowerbounds on Estimating the Classification Error\label{sec:bin-lowerbound}}
\subsection{Identity Covariance Lowerbound}
The next theorem establishes the lowerbound for the binary classification setting where the covariance of the data generating distribution is the identity.
\begin{theorem}
Given $\frac{\sqrt{d}}{2}$ samples, no algorithm can distinguish the case that the label is pure noise, meaning the label of each data point is uniformly randomly drawn from $\{+1,-1\}$, independent from the data, versus no noise, where there is an underlying hyperplane represented as vector $\beta$ such that the label is $sgn(\beta^Tx)$ with probability greater than $0.77$.
\end{theorem}

We show our lowerbound by upperbounding the total variational distance between the following two cases
\begin{enumerate}
\item Draw $n$ independent samples $(\x_1,y_1),\ldots,(\x_n,y_n)$ where $\x_i\sim N(0,I),$ and $y_i\sim \{-1,+1\}.$
\item First uniformly random pick a unit vector $v$, the draw $n$ independent samples $(\x_1,y_1),\ldots,(\x_n,y_n)$ where $\x_i\sim N(0,I), y_i = sgn(v^T\x_i)$.
\end{enumerate}
The claim then is that no algorithm can distinguish the two cases with probability more than $0.77$.  Let $Q_n$ be the joint distribution of $(\x_1,y_1),\ldots,(\x_n,y_n)$ in case $2$. Our goal is to bound the total variance $D_{TV}(Q_n,N(0,I)^{\otimes n})$ which is smaller than $\frac{\sqrt{\chi^2(Q_n,N(0,I)^{\otimes n})}}{2}.$
In case 2, for a fixed $v$, the conditional distribution $P(\x|y)= I(sgn(v^T\x)=y)G(\x)$ where $I$ is the indicator function. Let $P_{y,v}$ denote such a conditional distribution. The chi-square divergence can be expressed as:
\begin{align*}
1+ \chi^2(Q_n,N(0,I)^{\otimes n}) = 
\int_{\x_1,y_1}\ldots \int_{\x_n,y_n}\frac{\Big(\int_{v\in \mathcal{S}^d} \prod_{i=1}^n  P_{y_i,v}(\x_i)P(y_i)dv\Big)^2}{\prod_{i=1}^n G(\x_i)P(y_i)} d\x_1dy_1\ldots d\x_ndy_n\\
 = \int_{\x_1,y_1}\ldots \int_{\x_n,y_n}\int_{v\in \mathcal{S}^d} \int_{v'\in \mathcal{S}^d} \prod_{i=1}^n \frac{   P_{y_i,v}(\x_i)  P_{y_i,v'}(\x_i)P(y_i)}{G(\x_i)} dv dv' d\x_1dy_1\ldots d\x_ndy_n\\
 = \int_{v\in \mathcal{S}^d} \int_{v'\in \mathcal{S}^d} \Big(\int_y \int_{\x}\frac{   P_{y,v}(\x)  P_{y,v'}(\x)P(y)}{G(\x)} d\x dy\Big)^n dv dv'.
 \end{align*}
Notice that the $\frac{P_{y,v}(\x)  P_{y,v'}(\x)}{G(\x)}=4G(\x)$ only when $sgn(v^T\x)=sgn({v'}^T\x)=y$ and otherwise $0$. Hence $\int_y \int_{\x}\frac{   P_{y,v}(\x)  P_{y,v'}(\x)P(y)}{G(\x)} d\x dy = 2-\frac{2\theta}{\pi}$.
\begin{align*}
\int_{v\in \mathcal{S}^d} \int_{v'\in \mathcal{S}^d}(1+(1-\frac{2\theta}{\pi}))^n dvdv'\\
\le \int_{0}^{\pi/2} (1+(1-\frac{2\theta}{\pi}))^n \sin^{d-2}\theta d\theta/B((d-1)/2,1/2) + \int_{\pi/2}^{\pi}\sin^{d-2}\theta d\theta/B((d-1)/2,1/2)\\
\le \int_{0}^{\pi/2} (1+\cos\theta)^n \sin^{d-2}\theta d\theta/B((d-1)/2,1/2) + 1/2
\end{align*}
where we have applied Fact~\ref{fact:betafun},~\ref{fact:chi2} and the fact that $\frac{2\theta}{\pi}\le \cos\theta$ for $\theta\in [0,\frac{\pi}{2}]$ in the above derivation. This corresponds to the case where $\frac{b^2}{1-b^2}=1$ in the proof of Proposition~\ref{prop:lrlbid}. Let $n\le \frac{1}{2}\sqrt{d}$, we have that there is no algorithm that can distinguish the two cases with probability greater than $0.77$.
\subsection{General Covariance Lowerbound}
\begin{proposition}\label{thm:bin-lb-gen}
Without any assumptions on the covariance of the data distribution, it is impossible to distinguish the case that the label is pure noise, meaning the label of each data point is uniformly randomly drawn from $\{+1,-1\}$, independent from the data, versus no noise, where there is an underlying hyperplane represented as vector $\beta$ such that the label is $sgn(\beta^Tx)$, with probability better than $2/3$ using $c\cdot d$ samples, for some constant $c$.
\end{proposition}
\begin{proof}
We prove the proposition by reducing the problem of distinguishing pure noise versus no noise in the linear regression setting (Theorem~\ref{thm:genLB}) to the problem in the binary classification setting. Recall that in the proof of Theorem~\ref{thm:genLB}, each label $y_i$ is drawn from standard Gaussian distribution $N(0,1)$ in the pure noise case, and each label $y_i=\beta^T\x_i$ in the pure signal case. Given samples from the linear regression setting, we can create binary labels for each sample $i$ with $y'_i=sgn(y_i)$ where $sgn(x)$ is the sign function which takes value $1$ for $x>0$ and $-1$ for $x\le 0$. The distribution of $y'_i$ constructed from the pure noise case of linear regression will exactly be the distribution of the label in the pure noise case of the binary classification setting stated in the proposition, and this holds analogously for the pure signal case. Hence if there is an algorithm that can distinguish the case that the label is pure noise versus pure signal in the binary classification setting, that would also yield an algorithm for the linear regression setting, which is prohibited by Theorem~\ref{thm:genLB}. Thus the proof is complete.
\end{proof}